\documentclass[12pt]{statsoc}
\usepackage[a4paper]{geometry}
\usepackage{amssymb,amsmath}

\usepackage{amsthm}

\usepackage{graphicx}
\usepackage{subfigure}
\usepackage{paralist}

\usepackage{algorithm}
\usepackage{algorithmic}

\usepackage{url}
\usepackage{mathrsfs}
\usepackage{booktabs}
\usepackage{natbib}

\usepackage[colorlinks,citecolor=blue,urlcolor=blue]{hyperref}
\usepackage[toc,page,title,titletoc,header]{appendix}
\usepackage{arydshln}

\newtheorem{theorem}{Theorem}
\newtheorem{corollary}{Corollary}
\newtheorem{lemma}{Lemma}[section]
\newtheorem{remark}{Remark}

\usepackage{etoolbox}

\makeatletter
\patchcmd{\@makecaption}
{\parbox}
{\advance\@tempdima-\fontdimen2} 
{}{}
\makeatother

\newcommand{\bm}{\boldsymbol}

\newcommand{\bsb}{\boldsymbol}
\newcommand{\bsbY}{\boldsymbol{Y}}
\newcommand{\bsbB}{\boldsymbol{B}}
\newcommand{\bsbA}{\boldsymbol{A}}
\newcommand{\bsbC}{\boldsymbol{C}}
\newcommand{\bsbV}{\boldsymbol{V}}
\newcommand{\bsbU}{\boldsymbol{U}}
\newcommand{\bsbD}{\boldsymbol{D}}
\newcommand{\bsbE}{\boldsymbol{E}}
\newcommand{\bsbK}{\boldsymbol{K}}
\newcommand{\bsbL}{\boldsymbol{L}}
\newcommand{\bsbX}{\boldsymbol{X}}
\newcommand{\bsbZ}{\boldsymbol{Z}}
\newcommand{\bsbI}{\boldsymbol{I}}
\newcommand{\bsbS}{\boldsymbol{S}}
\newcommand{\bsbF}{\boldsymbol{F}}
\newcommand{\bsbP}{\boldsymbol{P}}

\newcommand{\bsbW}{\boldsymbol{W}}
\newcommand{\bsbSig}{\boldsymbol{\Sigma}}
\newcommand{\bsbDelta}{\boldsymbol{\Delta}}
\newcommand{\bsbGamma}{\boldsymbol{\Gamma}}

\newcommand{\bsbc}{\boldsymbol{c}}
\newcommand{\bsbx}{\boldsymbol{x}}
\newcommand{\bsby}{\boldsymbol{y}}
\newcommand{\bsbs}{\boldsymbol{s}}
\newcommand{\bsbv}{\boldsymbol{v}}
\newcommand{\bsbb}{\boldsymbol{b}}
\newcommand{\bsba}{\boldsymbol{a}}
\newcommand{\bsbf}{\boldsymbol{f}}

\newcommand{\bsbalpha}{\boldsymbol{\alpha}}

\newcommand{\bsbe}{{\boldsymbol{\eta}}}
\newcommand{\bsbmu}{{\boldsymbol{\mu}}}
\newcommand{\bsbdelta}{\boldsymbol{\delta}}

\newcommand{\bsbtau}{\boldsymbol{\tau}}

\DeclareMathOperator*{\argmin}{argmin}
\newcommand{\Proj}{{{\mathcal P}}}
\newcommand{\EE}{\,\mathbb{E}}
\newcommand{\EP}{\,\mathbb{P}}
\DeclareMathOperator{\vect}{\mbox{vec}\,}
\newcommand{\rd}{\,\mathrm{d}}
\newcommand{\rank}{\mathrm{rank}}

\newcommand{\breg}{{\mathbf{\Delta}}}
\newcommand{\Breg}{{\mathbf{D}}}

\title[]{Supervised Multivariate Learning with Simultaneous \\ Feature Auto-grouping and Dimension Reduction}
\author[Yiyuan She {\it et al.}]{Yiyuan She}
\address{Department of Statistics, Florida State University,
        Tallahassee,
        USA.}
\email{yshe@stat.fsu.edu}
\author[]{Jiahui Shen}
\address{Department of Statistics, Florida State University,
        Tallahassee,
        USA.}
\author[]{Chao Zhang}
\address{Center for Information Science, Peking University,
        Beijing,
        China.}

\usepackage[normalem]{ulem}

\begin{document}

\begin{abstract}
        Modern  high-dimensional methods often  adopt the ``bet on sparsity'' principle, while in supervised multivariate learning statisticians  may face ``dense'' problems with a large number of nonzero coefficients. This paper proposes a novel clustered reduced-rank learning (CRL) framework  that imposes two joint matrix regularizations to automatically group the features in constructing   predictive factors. CRL    is more interpretable than  low-rank modeling and relaxes  the stringent sparsity  assumption in variable selection.  In this paper, new  information-theoretical limits are presented to     reveal the intrinsic cost of seeking for clusters, as well as   the blessing from dimensionality in multivariate learning. Moreover, an efficient optimization  algorithm is developed, which  performs subspace learning and clustering  with guaranteed convergence. The obtained  fixed-point estimators,  though not necessarily globally optimal, enjoy the desired statistical accuracy beyond the standard likelihood setup under some regularity conditions.  Moreover, a new kind of information criterion, as well as its scale-free form, is proposed   for  cluster and rank   selection,  and has a rigorous theoretical support  without assuming an infinite sample size.         Extensive simulations and real-data experiments demonstrate the statistical accuracy and interpretability  of the proposed method.
\end{abstract}


\section{Introduction}
\label{introduction}
Modern statistical  applications create an urgent need for analyzing and interpreting  high-dimensional data with low-dimensional structures.
This paper works in a supervised multivariate setting with $n$ samples for $m$ responses and $p$ features (or predictors): $\bsbY = [\bsby_1, \ldots, \bsby_m] \in \mathbb R^{n\times m}$ and $\bsbX=[\bsbx_1, \ldots, \bsbx_p]\in \mathbb R^{n\times p}$. Given  a   loss    $l_0$, not necessarily a negative log-likelihood function, one can solve the following optimization problem to model the set of responses of interest\begin{align}
\min_{\bsbB\in \mathbb R^{p\times m}} l_0 (\bsbX \bsbB; \bsbY).
\end{align}
Here,      the unknown coefficient matrix   $\bsbB  = [\bsbb_1, \ldots, \bsbb_p]^T $ has $pm$  unknowns, with     $\bsbb_j$ summarizing  the contributions of the $j$-th predictor to all the responses.

The modern-day challenge comes from large $p$ and/or $m$. Statisticians often prefer selecting  a small subset of features---for example, a group-$\ell_1$ penalty $\lambda \sum \| \bsbb_j\|_2$ \citep{yuan2006model} can be added in the criterion to promote row-wise sparsity in  $\bsbB$, which results in  a more interpretable model  than using an  $\ell_2$-type penalty $\lambda \| \bsbB\|_F^2$. However, with a large $m$, there may exist  few features that are completely irrelevant to the whole set of responses.   One may perform   variable selection in a transformed space rather than the original   space  \citep{johnstone2009consistency},  but  how to find a proper transformation to reveal sparsity is problem-specific.

Perhaps a natural alternative  is to make the coefficients   form a relatively small number of groups,  within each of which all coefficients
are forced to be equal. 
 This is  referred to as ``{equisparsity}'' in \cite{she2010sparse}. In the general multivariate  setup,  instead of requiring a large number of zero rows in the true signal $\bsbB^*$, we  assume that it has  relatively few distinct row patterns   $\bsbb_{(1)}^{* T}, \bsbb_{(2)}^{* T}, \ldots,   \bsbb_{(q)}^{* T}$. Then, from
 \begin{align}\bsbX \bsbB^* = \bsbx_1 \bsbb_1 ^{*T}+\cdots + \bsbx_p \bsbb_p^{*T}= \big(\sum_{j\in \mathcal G_1} \bsbx_j \big)  \bsbb_{(1)} ^{*T}+   \cdots +  \big(\sum_{j\in\mathcal  G_q} \bsbx_j \big)  \bsbb_{(q)}^{*T},\end{align}    the features sharing the same  $\bsbb_{(k)}^{* }$ ($1\le k \le q$) are automatically    grouped,  based on their contributions to $\bsbY$.  The feature grouping is as interpretable as feature selection and  can offer  further parsimony, since the latter only targets  the set of irrelevant features with $\bsbb_j^{*}=\bsb0$.

The problem of how to cluster the unknown coefficient matrix to achieve the best predictability  falls into  supervised learning, where  both $\bsbY$ and $\bsbX$  are available. This is  in contrast to conventional clustering tasks for  unsupervised learning that   operate  on a single data matrix.  But it shares the same computational challenge in large dimensions.  A nice but sometimes unnoticed   fact is that  if $p$   points form $q$ clusters in a large  $m$-dimensional vector space, then the clusters can  be revealed in just a   $q$-dimensional subspace, such as the one spanned by the cluster centroids. In real  data analysis, it is not rare that the dimension of the cluster centroid space is much less than $q$ (even as low as 2 or 3). This motivates us to  perform   simultaneous dimension reduction to  ease the job of clustering.

Specifically, we  propose to  including an additional  low-rank constraint, and the resulting  jointly regularized form  provides an extension of the   celebrated reduced rank regression (RRR, \cite{izenman1975reduced}). RRR assumes that the rank of the true $\bsbB^*$ is no more than a small number $r$, or equivalently, $\bsbB^* =\bsbB_1  \bsbB_2^{ T}$ with each $\bsbB_i $ having   $r$ columns. Once locating a proper loading matrix $\bsbB_1 $,   the final model amounts to fitting $\bsbY$ on $r$ factors formed by $\bsbX \bsbB_1$. Unfortunately, it is well known that  the factor construction from a large number of features lacks  interpretability.  Our  proposal of clustered rank reduction  enforces row-wise equisparsity in $\bsbB_1$ (or the overall coefficient matrix)  so that  in extracting  $r$  predictive factors,   the original features can be automatically  consolidated into  $q$ groups at the same time.

\begin{figure}[h!]
        \centering
        \includegraphics[width=.75\textwidth, height=3.45in]{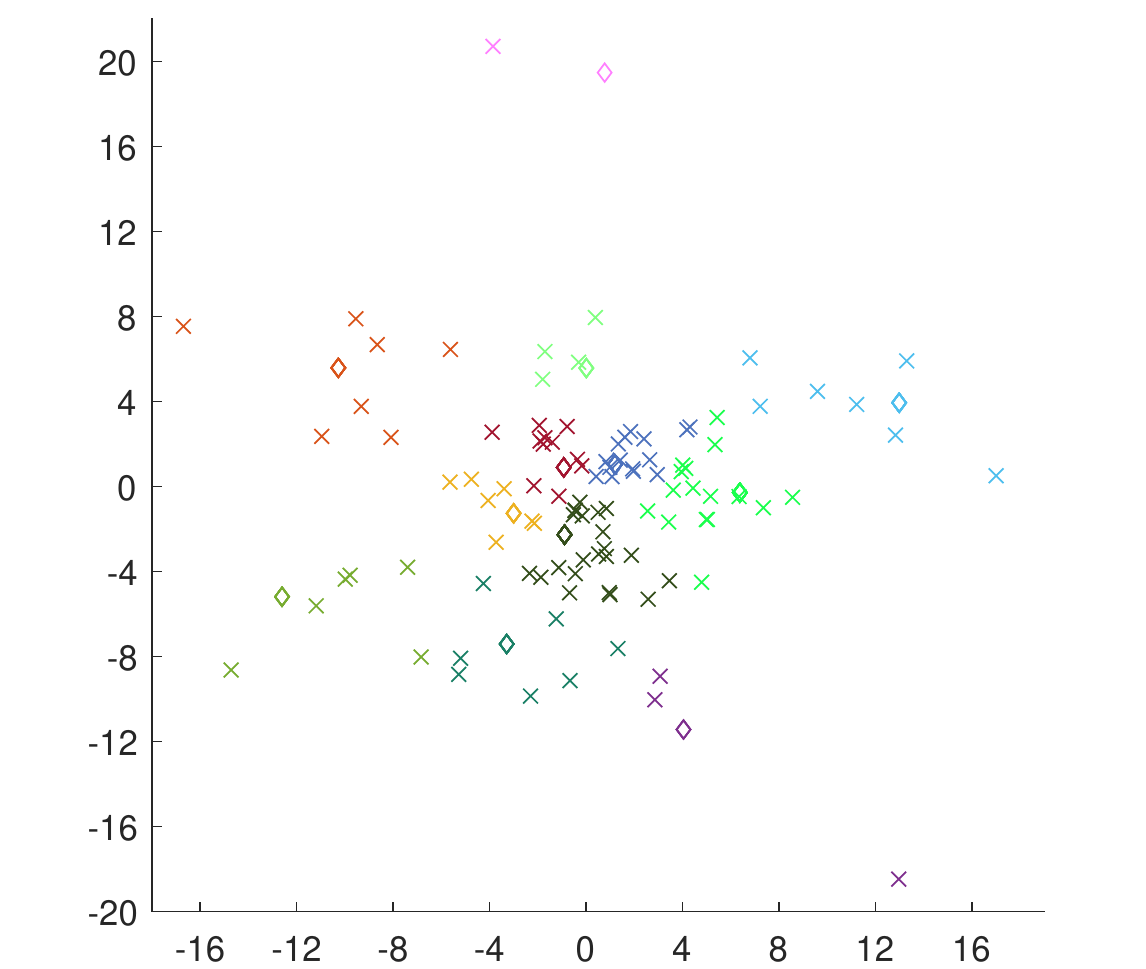}
        \caption{\small Yeast cell cycle data: the loading weights, denoted by `$\times$', by fitting  an RRR with $r=2$.  The  many noticeable nonzeros imply that a number of  TFs    have effects on  the gene expressions during the cell cycle process. Here, the 12 diamonds       represent the consolidated  loadings  obtained by the proposed    CRL method.\label{TFex}}
\end{figure}

It is perhaps best to illustrate the idea on    a real-world example.  The yeast cell cycle data used in   \cite{chun2010sparse}
studies transcription factors (TFs) related to gene expression over time. In addition to the predictor matrix $\bsbX\in \mathbb R^{542\times 106}$ with $106$ TFs    collected on $542$ genes,  a response  matrix $\bsbY\in \mathbb R^{542\times 18}$ containing RNA levels measured on the same  genes  is available at $18$ time points. A   naive multivariate regression    would have about  2k  unknowns,     and so we fit an RRR with  $r=2$ and plot      the  loadings of the 106  TFs in   Figure \ref{TFex}. The fact that  most of the loadings are apparently  nonzero makes variable selection less effective in reducing the  complexity of the model. Indeed, because of the {high-quality} experimental design by biologists, quite a few TFs seem to have effects on  the gene expressions during the cell cycle process.

 Intuitively,    clustering the  TFs' loadings  would offer a  significant   reduction of the number of free parameters. Note that to enhance interpretability and avoid ad-hoc tuning,  we force the
estimates within a group to be equal.
A  stagewise procedure  performing    estimation and  clustering in two distinct steps would be      suboptimal; we aim to solve the problem \textit{as a whole}. The diamonds in Figure \ref{TFex}      show the loadings obtained by the proposed method that  simultaneously groups the features in performing dimension reduction.

Compared with the low-rank modeling, the new parsimonious model not only boosted the prediction accuracy   by  23\% (over 200 repeated training-test splits with 50\% for training and 50\% for test), but offered some  meaningful TF groups. For example, that  {\tt ACE2}, {\tt SWI5} and {\tt SOK2} fall  into the same group and share the same set of large coefficients  provides  useful biological insights, as it is well known that {\tt ACE2} and {\tt SWI5} are paralogs, meaning that they are related to each other through a gene duplication event and  are highly conserved in yeast cell cycle gene progression, and   according to \cite{Pan2000}, with regards  to nitrogen limitation, {\tt SOK2},  along with {\tt ACE2} and {\tt SWI5}, is essential in the pseudohyphal  growth of yeast cells.  Our analysis also provides a cluster with three TFs, namely {\tt HIR1}, {\tt STP2} and {\tt SWI4},  all of which are chromatin-associated transcription factors   involved in regulating the expression of multiple genes at distinct phases of yeast cells \citep{Lambert2010}.

This paper studies simultaneous feature auto-grouping and dimension reduction, and attempts to tackle some related challenges   in methodology, theory, and computation.  Our main contributions are as follows.

\begin{compactitem}
\item
A novel clustered reduced-rank learning (CRL) framework is proposed, which imposes joint matrix regularizations through a convenient SV formulation. It    relaxes  the    assumption of sparsity and  offers improved   interpretability compared with vanilla low-rank modeling. The concurrent dimension reduction substantially eases  the task of clustering in high dimensions.

\item  Universal information-theoretical limits       reveal    the intrinsic cost of seeking for clusters, as well as the benefit of accumulating a large number of responses in multivariate learning, which seems to be largely  unknown   in the literature before. 

\item Tight error bounds   are shown for CRL  beyond the standard likelihood setup and      justify its  minimax   optimality in some common scenarios. These nonasymptotic results are strikingly different from those  for sparse learning  and are the first of their kind. Our theoretical studies  favor CRL   over variable selection   when the numbers of relevant features and irrelevant features are of the same order,   or when   the number of responses is greater than or equal to the number of  features up to a multiplicative constant.

\item
An efficient optimization-based algorithm is developed, which performs   simultaneous subspace pursuit and clustering with   guaranteed convergence.  The resulting  fixed-point estimators, though  not necessarily globally optimal, achieve the desired statistical accuracy under some regularity conditions.

\item  A  predictive information criterion is proposed  for  joint cluster and/or rank   selection. Its  brand new model complexity notion differs  from  existing information criteria, but   has a rigorous theoretical support  in finite samples. A  scale-free form is further proposed to  bypass the noise scale estimation.\\
\end{compactitem}

 The rest of the paper is organized as follows. Section \ref{sec_framework} describes in detail the clustered reduced-rank learning framework to automatically group  the predictors in building a predictive low-rank  model. Section \ref{sec_theory} shows some   universal minimax lower bounds and  tight upper bounds of CRL, from which one can conclude that CRL enjoys minimax  optimality if the number of clusters is at most polynomially large in the rank. The obtained rates
differ {substantially}   from the standard results assuming  sparsity, and interestingly, having a large number of responses seems to be a {blessing}.
   Section \ref{sec:comp}  develops an   iterative  and easy-to-implement algorithm by linearization and block coordinate descent, where Procrustes rotations and clusterings are performed repeatedly with guaranteed convergence. A new predictive information criterion, together with its scale-free form, is proposed for model selection in the context of clustered rank reduction.   Section \ref{sec_dataexp}   shows some real data analysis.     We conclude in Section \ref{sec:summ}.
The appendices provide all technical details  and more computer experiments.

\emph{Notation and symbols}. The following notation and symbols will be used. 
Given a differentiable $f$, we use $\nabla f$ to denote its gradient, and $f$ is called $\mu$-strongly convex  if
$
f(\bsbe')\ge f (\bsbe) + \langle \nabla f(\bsbe), \bsbe' - \bsbe\rangle + \mu \| \bsbe' - \bsbe\|_2^2/2,  \forall \bsbe, \bsbe',
$ 
and $L$-strongly smooth if
$
f(\bsbe')\le f (\bsbe) + \langle \nabla f(\bsbe), \bsbe' - \bsbe\rangle + L \| \bsbe' - \bsbe\|_2^2/2, \ \forall \bsbe, \bsbe'. $
   In particular,   $f(\bsbe) = \| \bsbe - \bsby\|_2^2/2$ is 1-strongly convex.
For any $\bsbA, \bsbB \in \mathbb{R}^{n\times m}$, we denote by $\langle \bsbA, \bsbB  \rangle$ the inner product of $\bsbA, \bsbB$. Given $\bsbA \in \mathbb{R}^{n \times m}$,   $\bsbA^+$ denotes its Moore-Penrose inverse,  $\rank(\bsbA)$ denotes its rank, and when $n=m$, $\sigma_{\max}(\bsbA)$ denotes its maximal eigenvalue. We use $\bsbA[i,j]$ to represent the $(i,j)$-th element in $\bsbA$ and $\bsbA[i,:]$ (or $\bsbA[:,i]$) to represent the $i$-th row (or column) of $\bsbA$.  Some conventional matrix norms of $\bsbA$ are  as follows: $\|\bsbA \|_F$ denotes the Frobenius norm, $\|\bsbA \|_2$ the spectral norm, $\|\bsbA\|_*$  the nuclear norm, and   $\| \bsbA\|_{2, \infty} = \max_{1\le j \le p} \| \bsba_j\|_2$ for $\bsbA = [\bsba_1, \ldots, \bsba_p]^T$.
The constants denoted by $C$, $c$  are not necessarily the same at each occurrence.
Finally,    $a\lesssim b$ means $a\le c b$  up to a multiplicative positive  constant $c$, and $a\asymp b$ means $a\lesssim b$ and $b\lesssim a$.


\section{Clustered Reduced Rank Regression}
\label{sec_framework}


This section   focuses on the   {quadratic loss} commonly used in multivariate regression,$$
l_0(\bsbX \bsbB; \bsbY) = \| \bsbY - \bsbX \bsbB\|_F^2/2.
$$
The discussions in this important case will lay   out a foundation for     computation and theoretical analysis  in later sections regarding a general  loss.

Motivated by Section \ref{introduction}, rather than  assuming   that most features are irrelevant to  the responses,     we propose to enforce   row-wise equisparsity  in $\bsbB$ so that we can  group the features in modeling  $\bsbY$.
Sparsity is just a special case of equisparsity, and clustering   the nonzero values can   gain further parsimony. 
Meanwhile, we would like to regularize the multivariate model with low rank,  making it possible to project  the data into a much smaller subspace to reveal the row patterns of $\bsbB$.

To mathematically formulate the problem, we use   $\| \bsbb \|_{\mathcal C}$ to denote the number of distinct elements in vector $\bsbb$, and $\| \bsbB \|_{2, \mathcal C}$   the number of distinct rows   of $\bsbB$. 
 Then,  our clustered reduced-rank learning  ({CRL}) involves the minimization of the loss criterion with two constraints
$ \rank(\bsbB)\le r,   \|\bsbB\|_{2,\mathcal C}\leq q$. The joint regularization formulation poses significant  challenges  in  both computation and theory.

A trick to  {decouple} the two intertwined constraints   is to write   $\bsbB = \bsbS \bsbV^T$ with $\bsbV$ an $m\times r$ column-orthogonal matrix. The ``SV'' formulation will be used in optimization as well. Since   $\bsbS = \bsbB \bsbV$, we get   $\| \bsbS\|_{2,\mathcal C} = \| \bsbB\|_{2,\mathcal C} $  and thus   an equivalent CRL  problem with separate  constraints on $\bsbS$ and $\bsbV$:
\begin{equation}
\min_{\bsbS \in \mathbb{R}^{p\times r},\bsbV \in \mathbb{R}^{m\times r}} \|\bsbY-\bsbX \bsbS \bsbV^T\|_F^2  \quad \mbox{s.t.}\, \ \ \bsbV^T\bsbV=\bsbI,\  \|\bsbS\|_{2,\mathcal C}\leq q. \label{supervised_CRL}
\end{equation}
In \eqref{supervised_CRL},  $q$ ($1\le q\le p$) controls the number of feature groups, and $r$ ($ r \leq m$) is  the dimension of the subspace to pursue  coefficient clustering. The joint tuning of the regularization parameters $q$ and $ r$ is an important task, too, and a data-adaptive solution with   theoretical support will be given in Section \ref{subsec:tuning}. Let $\bsbV=[\bsbv_1, \ldots, \bsbv_r]$ and $ \bsbS=[\bsbs_1,\ldots,\bsbs_r]$.  We call $\bsbs_i \ (1 \leq i \leq r)$ the \textit{clustering vectors}.
 An intercept term  $\boldsymbol{1}\bsbalpha^T$   can be added to the loss to help control  the scale of   clustering vectors. Equivalently, for regression, one can simply center $\bsbY, \bsbX$ columnwise in advance. We also suggest standardizing the predictors beforehand, as done in other regularization methods like LASSO and ridge regression, unless all  the predictors are on the same scale.

We discuss   a special case of \eqref{supervised_CRL} to provide a more intuitive understanding on the mathematical problem, which   can also be used to develop a sequential estimation procedure. Letting $r=1$,  \eqref{supervised_CRL} becomes $\min_{\bsbs\in \mathbb{R}^{p},\bsbv\in \mathbb{R}^{m}} \|\bsbY- \bsbX\bsbs \bsbv^T\|_F^2      \mbox{ s.t. }   \| \bsbv\|_2^2 =1, \| \bsbs\|_{\mathcal C}\le q $.
Reparematrize  $\bsbs = d \bsbs^{\circ}$ with $d \in \mathbb R$ and $\bsbs^\circ$ satisfying $\| \bsbX\bsbs^{\circ} \|_2=1$. 
Simple algebra shows that   the problem is minimized at     $d=\langle\bsbX \bsbs^{\circ}\bsbv^T,\bsbY \rangle$ and    $\bsbv=\bsbY^T\bsbX\bsbs^{\circ}/\| \bsbY^T\bsbX\bsbs^{\circ}\|_2$. Therefore,  \eqref{supervised_CRL} reduces to \eqref{rank1Prob2} when $r=1$:
\begin{align}
\max_{ \bsbs^{\circ}\in {\mathbb R}^p}  {\bsbs^{\circ}}^T (\bsbX^T{\bsbY \bsbY^T}\bsbX)  \bsbs^{\circ}  \ \ \mbox{ s.t. }  \    \bsbs^{\circ T}  \bsbX^T \bsbX \bsbs^{\circ} =1, \| \bsbs^\circ\|_{\mathcal C}\le q .
\label{rank1Prob2}
\end{align}
Without the last constraint,  \eqref{rank1Prob2} is   a  \textit{generalized  eigenvalue    decomposition}  problem. The regularization enforces equisparsity in     estimating the   generalized eigenvector. Though \eqref{rank1Prob2}  is  intuitive,  \eqref{supervised_CRL}  is  much more amenable to optimization.

The regularization   admits other   variants via  the SV-formulation. For example, with
    $\bsbB=\bsbS\bsbV^T = \bsbs_1\bsbv_1^T + \cdots + \bsbs_r\bsbv_r^T$,   one can    pursue equisparsity in each  component $\bsbs_i \bsbv_i^T$ or  $\bsbs_i$:
\begin{equation}
\min_{\bsbS\in \mathbb{R}^{p\times r},\bsbV\in \mathbb{R}^{m\times r}}  \|\bsbY- \bsbX\bsbS \bsbV^T\|_F^2  \quad \mbox{s.t.} \  \ \bsbV^T\bsbV=\bsbI,\  \|\bsbs_i\|_{\mathcal C}\leq q^e,\  1 \leq i \leq r . \label{supervised_CRL_rankwise}
\end{equation}
This rankwise CRL allows each feature to belong to more than one cluster as  $r>1$. In comparison, the constraint in   \eqref{supervised_CRL} offers a uniform control.
 Unless otherwise mentioned, we will  focus on the row-wise problem \eqref{supervised_CRL},      but our algorithm  applies to both.

\begin{remark} Alternative formulations  of CRL. \upshape
Given a positive definite matrix $\bsbGamma$ of size $m \times m$, consider a weighted criterion:
        $
        \min_{(\bsbS,\bsbV) \in  \mathbb{R}^{p\times r} \times \mathbb{R}^{m\times r}} \allowbreak  \mbox{Tr}\{ ( \bsbY-\bsbX\bsbS\bsbV^T)\bsbGamma( \bsbY-\bsbX\bsbS\bsbV^T)^T \}   \mbox{ s.t. }     \bsbV^T\bsbV=\bsbI, \| \bsbS\|_{{2,C}}\le q.
        $ 
        Then, for  $\bsbB =\bsbS \bsbV^T \bsbGamma^{ {1}/{2}}$, we have  $\rank(\bsbB)\le r$ and    $ \| \bsbB\|_{2,C} =  \|\bsbS\|_{2,C}$. Applying the SV representation to $\bsbB$ gives an equivalent  problem (with $\bsbS, \bsbV$ redefined)
        \begin{equation} \label{weighted_SIC3}
        \min_{(\bsbS,\bsbV) \, \in \, \mathbb{R}^{p\times r} \times \mathbb{R}^{m\times r}} \frac{1}{2}\| \bsbY\bsbGamma^{{1}/{2}}  -\bsbX\bsbS\bsbV^T \|_F^2    \ \mbox{ s.t. } \  \bsbV^T\bsbV=\bsbI, \| \bsbS\|_{2,C}\le q.  \end{equation}
        \eqref{weighted_SIC3} is of the same form of    \eqref{supervised_CRL} with an adjusted response matrix.

        Another  related \textit{projected} form    directly measures  discrepancy in the projected space:
         \begin{equation}
        \min_{(\bsbS,\bsbA ) \, \in \, \mathbb{R}^{p\times r} \times \mathbb{R}^{m\times r}} \frac{1}{2}\| \bsbY \bsbA  -\bsbX\bsbS \|_F^2 \mbox{   s.t. }   (\bsbY\bsbA)^T \bsbY \bsbA=n\bsbI, \| \bsbS\|_{2, \mathcal C}\le q,
         \end{equation}
where we assume $r\le \rank(\bsbY)$.
        Let  $\bsbSig_{ Y}= \bsbY^T\bsbY/n=\bsbU \bsbD \bsbU^T$ with the diagonal matrix $\bsbD$ containing   $\rank(\bsbY)$ nonzero eigenvalues, $\bsbW = \bsbD^{1/2}\bsbU^T \bsbA$, and $\bsbB = \bsbS \bsbW^T \bsbD^{1/2} \bsbU^T$. Because    $ \|\bsbY \bsbA - \bsbX \bsbS\|_F^2 = \mbox{Tr} \{(  \bsbY -\bsbX\bsbB) \bsbSig_Y^{+}  ( \bsbY -\bsbX\bsbB)^T\} +nr-n  \rank(\bsbY)$ (cf. Lemma \ref{lem:ccatoweighted}), $\rank(\bsbB)\le r$,  and $\bsbB $, $\bsbS $ have the same row patterns,
we see that         the projected form simply amounts to  taking   $\bsbGamma= \bsbSig_{Y}^{+}$. 
        Popular choices of $\bsbGamma$ are   based on the  covariance matrix of $\bsbY$. 
See \cite{SheiGGL2020} for a proposal to account for dependence in the case of a general loss. A further  topic is   to  estimate  the high-dimensional   covariance matrix and mean matrix jointly,   but it  is beyond the scope of the current paper and we regard   $\bsbGamma$ as known. Then, based on \eqref{weighted_SIC3},
        one simply needs to ``whiten'' $\bsbY$ by  $
        \bsbY\bsbGamma^{\frac{1}{2}}
        $  
        beforehand.
\end{remark}

\begin{remark} \label{rem:pairwisepen} Pairwise-difference  penalization.         \upshape
        An alternative   idea, following   \cite{she2010sparse} and \cite{chi2015splitting}, is  to penalize the pairwise row-differences   of $\bsbB = [  \bsbb_1, \ldots,  \bsbb_p]^T$:
\begin{align*}
\sum_{1\le j< j'\le p} P(\|  \bsbb_j - \bsbb_{j'}\|_2; \lambda), 
\end{align*} where $P$ is a sparsity-inducing  function.  This type of regularization is however not of our primary interest,  due to its computational burden, suboptimal error rate, and difficulties in parameter tuning.
See Appendix \ref{app:pairwise} and Theorem \ref{th:pairwisel0} for more details.

\end{remark}

\begin{remark}\label{rmk:unsup} Unsupervised learning.  \upshape  Supervised learning is the focus of  our work,  but when there is a single data matrix $\bsbY\in \mathbb R^{n\times m}$, we can set $\bsbX=\bsbI$ in \eqref{supervised_CRL} to cluster its rows for unsupervised learning. (Substituting $\bsbY^T$ for $\bsbY$  offers clustered PCA as an alternative to sparse PCA.) Similar to the derivation in the rank-1 case, we can evaluate the optimal $\bsbV$ to get   $\min_{\bsbS \in \mathbb R^{n \times r}} \|  \bsbS  \|_F^2/2 - \|\bsbY^T \bsbS\|_*$ s.t. $\|\bsbS\|_{2, \mathcal C}\le q $, or       $\max_{ \| \bsbS^\circ\|_F=1}  \|\bsbY^T \bsbS^\circ\|_*^2 /2$  s.t. $\|\bsbS^\circ\|_{2, \mathcal C}\le q $ via $\bsbS = d\bsbS^\circ$. (In the more general supervised setup, we can show that  $\bsbS$ solves $\min_{\bsbS\in \mathbb R^{p\times r} } \allowbreak \|\bsbX \bsbS   \|_F^2/2 - \| \bsbY^T \bsbX \bsbS\|_* $ s.t. $\|\bsbS\|_{2, \mathcal C}\le q $.) Because  $\|\bsbY^T \bsbS\|_* = \mbox{Tr}\{ (\bsbS^T \bsbY\bsbY^T \bsbS)^{1/2}\} $,     CRL's clustering vectors  depend on $\bsbY$ through its sample inner products only. One can then introduce  a \textit{kernel}  CRL by substituting a positive semi-definite   $\bsbK$   for  $\bsbY \bsbY^T$. The desired clusters can still be obtained by solving the SV-form problem, with a   suitable pseudo-response constructed from  the kernel matrix.    
 Let's  consider two special cases    to contrast the unsupervised CRL with some related methods. (a) No data projection, i.e.,   $r=m$.  Then we can show that K-means is an algorithm to solve the  problem  (cf. Section \ref{subsec:algdesign}). Modern implementations of K-means    make good use of seeding and   can   obtain a decent   solution in {low} dimensions, which  will   assist the optimization of      CRL,   owing to its low-rank nature. Of course, as  K-means   operates in the input space, it   can be     ineffective    for large $m$.   
        (b)     No   equisparsity regularization. In this case, CRL   reduces to spectral clustering (cf. Appendix \ref{appsub:impldetails}). 
        Giving up the equisparsity regularization    simplifies the computation significantly, but spectral clustering, as well as other similarity-motivated procedures,  performs dimension reduction and clustering in two distinct steps. It would be less greedy to perform both steps {simultaneously}. In Section \ref{sec:comp}, we will see  that      the CRL algorithm can integrate   clustering   and subspace learning  to   solve the problem  as a whole.
\end{remark}

\section{Nonasymptotic  Statistical Analysis of Clustered Reduced-rank Learning}
\label{sec_theory}
Rigorous theoretical guarantees must be provided to justify  the proposed clustered rank reduction method. There is a   big literature gap  to fill in this regard. For example, how many samples are needed for signal recovery by adopting  the new notion of structural parsimony? In which situations will pursuing equisparsity be advantageous over performing  variable selection? Is it always  necessary to obtain a globally optimal solution in the nonconvex setup? The answers to these questions  seem   to be largely  unknown.

In this section, we go beyond the regression setup and consider  a  loss $l_0(\bsbX \bsbB; \bsbY) $ that is defined on the systematic component $\bsbX \bsbB$ with $\bsbY$ as parameters. Here,   $\bsbB$ is unknown and $\bsbX$ and $ \bsbY$ are observed, and so we occasionally omit the dependence on  $\bsbY$. Assume that $l_0$ is differentiable with respect to  $\bsbX \bsbB$.   Our tool for  tackling    a general loss  is the \textit{generalized   Bregman function} \citep{SheBregman}: given a      differentiable function $\psi$,
\begin{equation}
\breg_\psi( \bm\alpha,\bm\beta) := \psi(\bm\alpha) - \psi(\bm\beta) - \langle\nabla\psi(\bm\beta), \bm\alpha-\bm\beta\rangle. \label{genbregdef}
\end{equation}
If   $\psi$ is also strictly convex, $\breg_\psi(\bm\alpha,\bm\beta)$ becomes the standard Bregman divergence denoted by $\mathbf D_\psi(\bm\alpha,\bm\beta)$  \citep{Bregman1967}.     A simple  example is $\Breg_2(\bsb{\alpha},\bsb{\beta}) :=  \|\bsb{\alpha}-\bsb{\beta}\|_2^2/2$,  associated with $\psi=\|\,\cdot\, \|_2^2/2$, and its matrix version is $\Breg_2 (\bsbA, \bsbB) = \| \vect (\bsbA) - \vect(\bsbB)\|_2^2 /2 = \| \bsbA - \bsbB\|_F^2/2$. In general, $\breg_\psi(\bm\alpha,\bm\beta)$ may not be symmetric, and we define   its symmetrized version   by 
$\bar{\bm\Delta}_\psi(\bm\alpha,\bm\beta):=(\bm\Delta_\psi(\bm\alpha,\bm\beta)+\bm\Delta_\psi(\bm\beta,\bm\alpha))/2.$ 

Introducing the notion of noise in  the non-likelihood setup is  another essential component, since   $  l_0 $    may not correspond to    a distribution function.  We define the \textit{effective noise} associated with the statistical truth $ \bsbB^* $   by
\begin{equation} \label{noise-def}
\bsbE = -\nabla l_0(\bm X\bsbB^*; \bsbY).
\end{equation}
So having a zero-mean noise means that the risk vanishes at the statistical truth, assuming we can  exchange the  gradient and expectation. In a  canonical  generalized linear model (GLM) with   cumulant function  $b(\cdot)$ and $g=(\nabla b )^{-1}$  as the canonical link (cf. Appendix \ref{app:proofs}), the (unscaled) loss can be represented by $  -\langle\bm Y, \bsbX \bsbB\rangle +   b(\bsbX \bsbB)$, and by matrix differentiation, $$\bsbE  = \bsbY - g^{-1}(\bsbX\bsbB^*) = \bsbY-\mathbb E(\bsbY),$$ or     $\bsbE = \bsbY-\bsbX\bsbB^*  $ in   regression.
Unless otherwise specified, we assume that $\vect (\bsbE)$ is a \textit{sub-Gaussian} random vector with mean zero and scale bounded by $\sigma$ (namely, all   marginals $\langle \vect(\bsbE), \bm\alpha \rangle$   satisfy $\|\langle \vect(\bsbE), \bm\alpha \rangle\|_{\psi_2}\leq \sigma \|\bm\alpha \|_2$, $\forall \bm\alpha\in \mathbb R^{p}$, where $\| \cdot\|_{\psi_2} = \inf  \big\{ t>0: \EE \exp[( \cdot\,  /{t})^2]\le 2 \big\},
$ cf. \cite{van1996weak}). Note that  the components of  $\bsbE$ need  not be independent. Sub-Gaussian noises   are   typical  in regression and classification problems, since Gaussian and bounded random variables are  sub-Gaussian.  Yet   sub-Gaussianity   is not   critical   for    our analysis.

    This section focuses on  row-wise equisparsity   $\| \bsbB\|_{2, \mathcal C}$. It turns out that the two measures   $\|\cdot\|_{2, \mathcal C}$ and $\rank(\cdot)$ can effectively bound the  stochastic terms arising from CRL. Because it is be difficult in  present-day applications to tell whether   the sample size, relative to the problem dimensions, is large enough  to apply asymptotics,    all of our investigations will be  nonasymptotic.
\subsection{Universal minimax lower bounds}
\label{subsec:minimax}
The first question one must answer is how   small the  error could be under equisparsity with possibly low rank.   We derive new minimax lower bounds     to address the question.  Let $I(\cdot)$ be an arbitrary nondecreasing   function with $I(0)=0, I\not\equiv 0$; some particular examples are $I(t) = t$ and $I(t) = 1_{t\ge c}$. 

\begin{theorem} \label{th:minimax}
    Assume $\bsbY | \bsbX \bsbB^*$ follows  a distribution in the regular  {exponential   family} with dispersion $\sigma^2$ with      $l_0$ the associated  negative log-likelihood  function (cf. Appendix \ref{app:proofs} for details). Define a signal class by
        \begin{align}
        \bsbB^*\in \mathcal S(q, r)= \{\bsbB \in \mathbb R^{p\times m}:   \|\bsbB \|_{2, \mathcal C}\leq q, \,\rank(\bsbB)\leq r \},\label{modelass}
        \end{align} where      $p   \ge q\ge   r \ge 2$,    $r(q\wedge  \rank(\bsbX)+m - r )\ge 4$. Let $b, \zeta$ be any integers satisfying
        \begin{align}
        \sum_{i=0}^{\zeta} {r\choose i} (b-1)^i\ge q,
        \end{align} with $
        b\ge 2 ,  1\le \zeta \le r$, and define  a complexity function\begin{align}
        P (q, r) =   (q +m  ) r  + p  \,  \{\log  (er) -\log \log q\}. \label{specailminimaxrate}
        \end{align}

        Assume   for some $\kappa>0$
        \begin{align}
        \breg_{l_0} (\bsb0, \bsbX   \bsbB)\sigma^2\allowbreak\le    \kappa \|  \bsbB\|_F^2/2, \ \  \forall \bsbB  \in   \mathcal S(q, r). \label{minimaxestregcond}
        \end{align}   Then  there exist positive constants $c, c'$, depending on $I(\cdot)$ only,  such that
        \begin{align}
        \inf_{\hat \bsbB} \sup_{\bsbB^* \in \mathcal S(q, r)} \EE\left\{I\Big(\|  \bsbB^* -   \hat\bsbB\|_F^2{\big/}\Big[c \sigma^2 \big\{(  q +m  ) r  + \frac{p\log q}{b^2 \zeta}\big\}/\kappa\Big]\Big)\right\} \geq c' >0, \label{estgeneralminimaxrate}
        \end{align}
        where $\hat \bsbB$ denotes an arbitrary estimator.
        In particular, under   $8\le q \le \exp(r)$,
        \begin{equation}
        \inf_{ \hat \bsbB}\,\sup_{ \bsbB^*\in \mathcal S(q,r)} \mathbb E\big[I\big( \|  \bsbB^* - \hat\bsbB\|_F^2/\{ c \sigma^2  P(q, r) / \kappa\}\big)\big] \ge c' >0. \label{estgeneralminimaxrate-specific}
        \end{equation}

\end{theorem} 

A  more complete theorem     including  minimax lower   bounds for both the   estimation error $\| \hat \bsbB - \bsbB^*\|_F$ and    prediction error $\| \bsbX \hat \bsbB - \bsbX\bsbB^*\|_F$ is presented in Appendix \ref{subsec:proofminimax},  from which one can see that the size of  $\phi := q^{1/r}$ plays a vital role     in determining the final error rate.
These results are the first of their kind and provide useful guidance  in the  context of equisparsity. 
Our proof is nontrivial and makes use of       $q$-ary codes in information theory \citep{pellikaan_wu_bulygin_jurrius_2017}, as well as some useful facts   of the generalized Bregman functions for GLMs.

The regularity condition \eqref{minimaxestregcond}  is not restrictive. For regression and logistic regression, the condition is implied by  $\breg_{l_0} ( \bsb0, \bsbX   \bsbB   )\sigma^2\allowbreak\le \| \bsbX \bsbB\|_F^2/2 $ and $\breg_{l_0} (  \bsb0,\bsbX   \bsbB   ) \le \| \bsbX \bsbB\|_F^2/8 $, respectively.  The bound in 
\eqref{estgeneralminimaxrate} is  general, while   \eqref{estgeneralminimaxrate-specific} is     perhaps more illustrative: when $q\le  \exp(r)$  and $\kappa \le c   n$, $$\EE\big[  \|  \bsbB^* - \hat\bsbB\|_F^2 \big] \ge c   \sigma^2  P(q, r) / n, \mbox{ and } \EP\big[  \|  \bsbB^* - \hat\bsbB\|_F^2 \ge   c   \sigma^2  P(q, r) / n\big]> c_0>0,$$  by setting    $I(t) = t$ and $I(t) = 1_{t\ge c}$, respectively.
Therefore, in the  scenario of $q$    {being polynomially} large in  $ r$, i.e., $q\le r^c$ for some constant $c$, no estimator can beat the error rate   $P (q, r)  \asymp  (q +m  ) r  + p     \log  q$  in a minimax sense. Interestingly, when $m$ is a constant, the rate reduction  compared to $pm$ is  not  significant, whereas  having a large number of response variables is (perhaps surprisingly) a blessing for pursuing equisparsity.
\subsection{Upper error bounds of CRL}
\label{subsec:upperrCRL}
Can we    approach   the   optimal error rate using  a particular estimator? This part  shows that CRL is  a legitimate method,  and  more importantly, pursuing its globally optimal solutions  is unnecessary in many cases. Rather, finding a \textit{fixed point} of CRL,  defined by  \eqref{surrodefTH} and \eqref{BiterDefTH} below, would suffice for regular problems.

Let $r^* = \rank(\bsbB^*)$ and $ q^* = \|\bsbB^*\|_{2, \mathcal C}$. Given a differentiable $l_0$, define
\begin{align}
G_\rho(\bsbB; \bsbB^{-})=  l_0(\bsbX \bsbB; \bsbY) - \breg_{l_0}  (\bsbX\bsbB, \bsbX \bsbB^{-})+  {\rho} \Breg_{2}  (\bsbB, \bsbB^{-}), \label{surrodefTH}
\end{align}
where $\rho$, representing the inverse stepsize,  is an algorithm parameter to be chosen. Then,  for all    fixed points defined by
\begin{align}
\hat \bsbB \in \argmin_{\bsbB: \|\bsbB\|_{2, \mathcal C}\le q, \rank(\bsbB)\le r} G_\rho(\bsbB; \bsbB^-)|_{\bsbB^- = \hat \bsbB}, \label{BiterDefTH}
\end{align}
a nonasymptotic error bound can be derived by   calculating the metric entropy
of
the associated manifolds and using    the  Stirling numbers of the second kind.

\begin{theorem}\label{th_local}
        Let $r\ge r^*$, $q\ge q^*$, and $\hat \bsbB$ be any fixed point satisfying \eqref{BiterDefTH} for some $\rho>0$.
        Define
         \begin{equation}
         P_o (q, r) =  \{q \wedge \rank(\bsbX)  +m
\} r +(p -q )\log q. \end{equation}
  Assume   $\rho>0$ is chosen so that
\begin{align}
        \rho \Breg_2(\bsbB_1,  \bsbB_2)   \le  (2  \bar \breg_{l_0}-\delta \Breg_2) (\bsbX \bsbB_1, \bsbX  \bsbB_2)   + K \sigma^2  P_o(q,r), \ \forall \bsbB_i: \rank(\bsbB_i) \le r, \|\bsbB_i\|_{2, \mathcal C} \le q
        \label{regcondlocal-genloss0}
        \end{align}
        for  some   $\delta>0$ and  sufficiently large   $K\ge0$. Then, $\hat \bsbB$   satisfies
        \begin{align}
        \EE [\| \bsbX \hat \bsbB - \bsbX   \bsbB^*\|_F^2]  \lesssim \frac{ K\delta\vee 1}{\delta^2} \big\{ \sigma^2(q\wedge \rank(\bsbX) + m)r  +   \sigma^2(p- q) \log q + \sigma^2\big\}.
        \label{genlosserrrate-fixedpoint0}
        \end{align}
\end{theorem}

It is not difficult to see that when     $l_0$   is $\mu$-strongly convex, the following \textit{matrix restricted eigenvalue} condition implies \eqref{regcondlocal-genloss0}  with $K=0$:
\begin{align}\rho \|\bsbB_1 -  \bsbB_2\|_F^2   \le   (2\mu  -\delta) \|\bsbX (\bsbB_1 -   \bsbB_2)\|_F^2, \quad \forall  \bsbB_i : \rank(\bsbB_i) \le r, \|\bsbB_i\|_{2, \mathcal C} \le q. \label{matrixREcond0}
\end{align}  When $q$ is small, \eqref{matrixREcond0} is  applicable to large-$p$ designs. Similar  regularity conditions  are widely used in compressed sensing,   variable selection and low rank estimation \citep{candes2007dantzig,Bickel09,CandPlan}.

 Let $q = \vartheta q^*$, $r = \vartheta r^*$ with $\vartheta\ge  1$. When $\vartheta$,   $\delta$ and $K$  are treated as constants and $\sigma=1$, from  \eqref{genlosserrrate-fixedpoint0},   the prediction error bound is   of the order
\begin{align}
\{q^*\wedge \rank(\bsbX)+ m\} r^* + (p - q^*) \log q^* ,
\label{errboundnaive}
\end{align}
ignoring    all trivial multiplicative/additive terms.
The   rate distinguishes CRL   from  various sparse learning methods in the literature.

\begin{remark} Computational feasibility. \upshape
The fact that the fixed-point solutions, though not necessarily
globally or even locally  optimal, can have provable guarantees   offers a   feasible computation of the nonconvex CRL optimization problem in regular cases.

Specifically, regardless of the choice of the loss,   $G_\rho$ always has a simple quadratic form in terms of    $\bsbB$, which gives rise to an iterative update of the coefficient matrix. Similar results can be shown for $( \bsbS , \bsbV)$  obtained by alternative optimization; see  Theorem \ref{th_local-SV} in Remark \ref{rem:svbounds}.    Section \ref{subsec:algdesign} designs an efficient algorithm  on the basis of   linearization and block coordinate descent.
\end{remark}

\begin{remark} Error rate comparison. \upshape
To clarify the theoretical meaning of \eqref{errboundnaive}, we   make an error-rate comparison between CRL and  some commonly used   estimators in  regression  with $\sigma=1$   (assuming all regularity conditions are met).
First,    assuming that    $\bsbX$ has full column rank, the ordinary least squares   has   $\EE\|\bsbX \hat \bsbB - \bsbX \bsbB^*\|_F^2=  mp$.
With no  rank reduction ($r^* = q^*$), \eqref{errboundnaive} gives $ (q^* + m) q^* + (p - q^*) \log q^*$,     which is   $\lesssim mp$ when  the number of responses is larger than the number of feature groups.
Of course, if  $r^* < q^*$, CRL   can achieve   a much  lower error rate. Comparing \eqref{errboundnaive} with    $(p+m)r^* $ by  low-rank matrix estimation \citep{BSW11}, we see that  CRL   does a substantially better job  if the number of clusters does not grow exponentially with the rank, namely,   $q^*\ll \exp (r^*)  $. 

Variable selection gives another important means of regularization. If  $\bsbB^*= [\bsbb_1^*, \ldots, \bsbb_p^*]^T$ is  row-wise sparse with  $s^* = |\{ j: \bsbb_j^* \ne \bsb0\}|$,         the prediction error  by means of variable selection is of the order \citep{lounici11}
\begin{align}
s^* m + s^* \log p.
\label{glassorate}
\end{align}
The comparison between \eqref{glassorate}   and \eqref{errboundnaive}  shows no clear winner: for the degrees-of-freedom terms,           $(\rank(\bsbX) \wedge q^* +m)r ^* \le q^* r^* + m s^* \lesssim   s^* m $, while for the `inflation' terms, $(p - q^*) \log q^*$  is typically larger than $s^* \log p$.   But two scenarios draw our particular attention:
\begin{align*}
\mbox{ (i) ``many  responses'':  } m\ge c p \qquad \mbox{  (ii) ``linear sparsity'': }   s^* = c p
\end{align*}  where $c$ is a positive constant. In {either} situation,   CRL is advantageous over  variable selection.

Concretely, in case (i), $s^*m + s^* \log p \asymp s^*m$, while  from $s^*\ge q^*\ge r^*$, we have $s^*m\ge r^* m \ge r^* q^*$ and   $s^* m \gg (\log q^*) (p - q^*)$, and so  $s^*m + s^* \log p\gtrsim \{q^*\wedge \rank(\bsbX)+ m\} r^* + (p - q^*) \log q^*  $. In case (ii), $  s^* \log p \asymp p \log p \ge (p - q^*) \log q^*$ and  the same conclusion holds. In other words, when the number of responses is greater than  the number of features up to a multiplicative constant, or   when the number of relevant features and the number of irrelevant features are of the same order,  CRL has a lower error rate          with  rigorous theoretical support.
\end{remark}

\begin{remark} Minimax optimality. \upshape
One may be curious if the upper error bound of CRL could match    the universal minimax lower bound.
Notably, under the mild conditions      $q^* \le \rank(\bsbX)  $ and $  q^* \ll \exp(r^*)$ (and so $q^*\log q^* \ll
q ^*r^*$),  \eqref{errboundnaive} becomes       $$ (q ^*+m  ) r^* + p \log q^*,$$    which is  exactly the rate shown at the end of  in Section \ref{subsec:minimax}.
So at least for canonical GLMs with   $q^*$   polynomially large in $
r ^* $, CRL does enjoy
minimax rate optimality.

Of course,  the  previous discussions
 assume that   $q$ and $r$  are specified so that they are not too large relative to   $q^*$ and $ r^*$,   respectively. In general,        the data-adaptive tuning to be introduced in Section \ref{subsec:tuning}   still ensures  \eqref{errboundnaive}.\\\end{remark}


        To the best of our knowledge, Theorem \ref{th_genlosserr} is the  first   nonasymptotic statistical analysis of   the set of CRL's fixed points in  nonconvex optimization. One might ask whether the error rate can be further improved  by  pursuing a  {global}   CRL estimator. The following theorem shows that this is not the case, but the regularity condition   \eqref{regcondlocal-genloss0} gets relaxed to some extent.

        \begin{theorem}\label{th_genlosserr}
                Let $\hat\bsbB $ be an optimal CRL solution with     $r \ge  r^*$ and $ q \ge q^*$.   Assume that there exists some $\delta>0$ such that
                \begin{align}
                \breg_{l_0}(\bsbX \bsbB_1, \bsbX \bsbB_2) \ge \delta \Breg_2 (\bsbX \bsbB_1, \bsbX \bsbB_2), \ \forall \bsbB_i: \|\bsbB_i\|_{2, \mathcal C}\le q, \rank(\bsbB_i)\le r. \label{regcond-genloss}
                \end{align}
                Then
                $
                \EE [\| \bsbX \hat \bsbB - \bsbX   \bsbB^*\|_F^2]  \lesssim \frac{1}{\delta^2} \{ \sigma^2(q \wedge \rank(\bsbX) + m)r   + \sigma^2 (p- q ) \log q +\sigma^2\}$. 
        \end{theorem}
        Unlike    \eqref{regcondlocal-genloss0} in Theorem   \ref{th_local},    \eqref{regcond-genloss}  uses         ${\breg}_{l_0}$ (in place of
        twice of its {symmetrized} version) and does not involve  $\rho$. The conclusion of Theorem  \ref{th_genlosserr} can be extended to  an oracle inequality \citep{donoho1994}, and these $\ell_2$-recovery
results can be used to give some     estimation error bounds under   proper regularity conditions.  Corollary \ref{cor:estbnd} gives an illustration.
        \begin{corollary} \label{cor:estbnd}
             Let $l_0$ be $\mu$-strongly convex.   Then for any $\bsbB: \rank(\bsbB) \leq r, \|\bsbB\|_{2, \mathcal C}\le q$,
\begin{align}
                \EE   \Breg_{l_0}(  \bsbX   \hat \bsbB, \bsbX \bsbB^*)\|_F^2 \lesssim
   \EE \Breg_{l_0}( \bsbX    \bsbB,  \bsbX \bsbB^*)  +
\frac{\sigma^2}{\mu} \{ (q\wedge \rank(\bsbX) +m  ) r+(p -q)\log q\} +\frac{ \sigma^2}{\mu}. \label{eqoracleineq}\end{align}
                Furthermore, assume   $\|\bsbX \bsbB\|_F^2/ n \ge \delta \| \bsbB\|_{2, \infty}^2, \forall \bsbB: \rank(\bsbB)\le (1+\vartheta) r^* , \|\bsbB\|_{2, \mathcal C}\le  \vartheta  q^{* 2}$ for some $\delta>0$, and     $r  =\vartheta r^*, q  =  \vartheta   q^*$, $\vartheta\ge 1$,   $q^*> 1$. Then with probability at least $1 - C\exp\{-c(m+\rank(\bsbX))\}$,     \begin{align}\| \hat \bsbB -   \bsbB^*\|_{2, \infty}^2
                \le   \frac{c_0\vartheta^2  \sigma^2}{  n \delta \mu^2 } \{(q^{*} \wedge \rank(\bsbX)+m  ) r^{*} + (p  -q^*)\log q^{*} \}  \label{eq2infnormbound}\end{align} for some constants $c_0, c, C>0$.
        \end{corollary}


        The RHS of \eqref{eqoracleineq}  offers a bias-variance tradeoff and as a result,  CRL applies to   $\bsbB^*$ with just   \textit{approximate} equisparsity and/or low rank.
 The $(2, \infty)$-norm error bound \eqref{eq2infnormbound} implies  {faithful}  cluster recovery   with  high probability,  if
the {signal-to-noise ratio}  is properly large:        $\min_{b_{j\cdot}^* \ne b_{k\cdot}^*}\allowbreak|  \mbox{avg}_{l} {b_{jl}^{* 2}} - \mbox{avg}_{l} b_{kl}^{*  2}|/\sigma^2 >    2\zeta$. Here,   $ \mbox{avg}_{l} {b_{jl}^{* 2}}$ is the average of $ b_{j1}^{* 2}, \ldots, b_{jm}^{* 2}$ and $\zeta=   \frac{c_0 \vartheta^2}{\delta \mu^2}   \{\frac{(p-q^*)\log q^{*}}{n m}+\frac{  r^* (q^{*}\wedge \rank(\bsbX))}{nm}+\frac{r^*}{n}\}.$
        Then, from  the bound on $\| \hat \bsbB -   \bsbB^*\|_{2, \infty}^2$, we know that for  $q= q^*$, $\hat \bsbB$ exhibits the same row clusters as does $\bsbB^*$, and for $q>q^*$,  $\hat \bsbB$     refines the clustering structure of $\bsbB^*$.


\section{Computation and Tuning}
\label{sec:comp}
In this section, we develop an efficient optimization-based  CRL
algorithm  with implementation ease and guaranteed convergence. We also propose  a novel model comparison  criterion and provide its theoretical justification   from a predictive learning perspective,
without assuming any infinite sample-size or large signal-to-noise ratio conditions.

\subsection{Algorithm design}
\label{subsec:algdesign}

In this part, we discuss how to solve the CRL  problem with   $q$ and $ r$  fixed. CRL  poses some intriguing  challenges in optimization: the problem is  highly nonconvex, $l_0$   is not restricted to the quadratic loss or a negative log-likelihood function, and  the equisparsity and low-rank constraints are of discrete nature. These obstacles render standard algorithms inapplicable.  In addition, as $p$ and $m$ may be large in real applications, the coefficient matrix $\bsbB\in \mathbb R^{p\times m}$ can easily contain an overwhelming number of unknowns. Then, how to make  use of its low-rank nature to reduce the computational cost at each iteration,  while maintaining the convergence of the overall procedure, is  crucial in large-scale computation.

Before describing the algorithm design in full detail, we provide below a simplified version of the   algorithm.

{\small
\begin{algorithm}[h]
        \caption{Clustered Reduced-rank Learning Algorithm
        }
        \label{alg_kmean}
        \begin{algorithmic}[1]
                \REQUIRE
                $(\bsbY,\bsbX) \in \mathbb R^{n \times m} \times \mathbb R^{n \times p}$,
                $l_0$:  a  loss function with $\nabla l_{0}$   $L$-Lipschitz continuous;
                $q$: desired number of clusters;
                $\epsilon$: error tolerance;
                $M$: maximum  number of iterations;
                a feasible starting point: $\bsbF^0, \bsbmu^0,   \bsbV^{[0]}$,  $\bsb{\alpha}^{[0]}$;   $ \rho$: inverse stepsize (e.g., $L\|\bsbX\|_2^2$)
                \STATE$k \leftarrow 0$, $\bsbS^{[0]} = \bsbF^0 \bsbmu^0$; 
                \WHILE{$k<M$ \textbf{and} $\epsilon <\|\bsbB^{[k]}-\bsbB^{[k-1]}\| \mbox{ (if existing)}  $ }
                \STATE $k \leftarrow k+1$;
                \STATE $\bsbB^{[k-1]}=\bsbS^{[k-1]} (\bsbV^{[k-1]})^T$,\,\,$\widetilde{\bsbY} \leftarrow \bsbB^{[k-1]}-\bsbX^T\nabla l_0(\bsb1 \bsb{\alpha}^{[k-1]}+\bsbX \bsbB^{[k-1]} )/\rho$;
                \STATE $\bsb{\alpha}^{[k]} \leftarrow  \bsb{\alpha}^{[k-1]}- [\nabla l_0(\bsb1 \bsb{\alpha}^{[k-1]}+\bsbX \bsbB^{[k-1]} )]^{T}\bsb1/\rho  $;
                \STATE $t \leftarrow 0, \bsbS^{(0)} \leftarrow \bsbS^{[k-1]}, \bsbV^{(0)} \leftarrow \bsbV^{[k-1]}, \bsbB^{(0)} \leftarrow \bsbS^{(0)}  \bsbV^{(0)T}$;
                \WHILE{not converged}
                \STATE $t \leftarrow t+1$;
                \STATE $\bsbW   \leftarrow \widetilde{\bsbY}^{T} \bsbS^{(t-1)}$;
                \STATE $\bsbV^{(t)} \leftarrow \bsbU_w\bsbV^T_w$ with $\bsbU_w$, $\bsbV_w$ from the SVD $\bsbW=\bsbU_w\bsbD_w\bsbV_w^T$;
                \STATE $\bsbL \leftarrow \widetilde{\bsbY} \bsbV^{(t)}$;
                \STATE With $\bsbmu^0$ as the initial cluster centroid matrix, call K-means on $\bsbL$ to update $\bsbF^l$ and $\bsbmu^{l}\, (l \geq 1)$ alternatively till convergence;
                \STATE $\bsbS^{(t)}  \leftarrow  \bsbF^{l} \bsbmu^{l}$, $\bsbB^{(t)}= \bsbS^{(t)}(\bsbV^{(t)})^T$, $\bsbmu^0 \leftarrow \bsbmu^{l}$;
                \ENDWHILE
                \STATE $ \bsbS^{[k]} \leftarrow \bsbS^{(t)}, \bsbV^{[k]} \leftarrow \bsbV^{(t)}, \bsbB^{[k]} \leftarrow \bsbS^{[k]} (\bsbV^{[k]})^T$;
                \ENDWHILE
                \RETURN $\bsbS=\bsbS^{[k]}, \bsbV=\bsbV^{[k]}, \bsbF=\bsbF^l$.
        \end{algorithmic}
\end{algorithm}
}
Define $\iota_{\bsbV}(\bsbV) =0$ if $\bsbV^T\bsbV=\bsbI$ and $+\infty$ otherwise. Similarly,  $\iota_{2,\mathcal C}(\bsbS)=0$ if $\| \bsbS \| _{2,\mathcal C} \leq q$ and $+\infty$ otherwise, and    $\iota_{\mathcal C}(\bsbS)=0$ if $\| \bsbs_i \| _{\mathcal C} \leq q^e$ for all $1 \leq i \leq r$ and $+\infty$ otherwise. We use    $\iota(\bsbS)$ to denote   $\iota_{2,\mathcal C}(\bsbS)$  in the row-equisparsity case and    $\iota_{\mathcal C}(\bsbS)$ in the rankwise case. The loss  $l_0(\bsbX \bsbB; \bsbY)$ is also written as $l_0(\bsbX \bsbS \bsbV^T; \bsbY)$
  or   $l(\bsbS,\bsbV; \bsbX, \bsbY)$,  often abbreviated as $l_0(\bsbX\bsbS \bsbV^T)$ or $l(\bsbS,\bsbV)$ for convenience.
The general CRL optimization problem  can be stated as
\begin{align}\min_{\bsbS \in \mathbb R^{p \times r},\bsbV \in \mathbb R^{m \times r}} f(\bsbS,\bsbV):= l(\bsbS,\bsbV; \bsbX, \bsbY)+ \iota(\bsbS)+\iota_{\bsbV}(\bsbV).\label{geneprob}
\end{align}
For simplicity, assume   the gradient of $l_0$ is $L$-Lipschitz continuous for some $L>0$: \begin{align} \| \nabla l_0  (\bsb{\Theta}_1) -   \nabla l_0  (\bsb{\Theta}_2)\|_F \leq L  \|\bsb{\Theta}_1 -\bsb{\Theta}_2 \|_F, \forall \bsb{\Theta}_1, \bsb{\Theta}_2 .\label{lipcond}
\end{align} 

 One idea might be to  apply alternating optimization directly, but    it would encounter  difficulties when $l_0$ is  non-quadratic.
Motivated by Theorem \ref{th_local}, we use a surrogate function to design an iterative algorithm.  Given $(\bsbS^{-}$, $\bsbV^{-})$, define
 \begin{align}
 \begin{split}
G_\rho(\bsbS{,}\bsbV; \bsbS^{-}{,}\bsbV^{-}) = \ &  l(\bsbS^{-}{,}\bsbV^{-}) + \langle \nabla l_0(\bsbX\bsbB^{-}), \bsbX(\bsbB-\bsbB^{-}) \rangle  \\ &+ \frac{\rho}{2} \|\bsbB-\bsbB^{-}\|_F^2   + \iota(\bsbS) + \iota_{\bsbV}(\bsbV),
\end{split}
\label{sol2}
\end{align}
where $\bsbB^{-}=\bsbS^{-}  (\bsbV^{-})^T$,  $\bsbB=\bsbS\bsbV^T$.  The dependence of $G$ on $\rho$ is often dropped for notational simplicity.
  \eqref{sol2} applies   linearization on $\bsbS \bsbV^T$ as a whole to construct the surrogate, but not on $\bsbS$ or $\bsbV$ individually.

Given any $(\bsbS^{[0]}, \bsbV^{[0]})$, let  $(\bsbS^{[k]}, \bsbV^{[k]})$  ($k\ge 1$) satisfy
\begin{align}
(\bsbS^{[k]}, \bsbV^{[k]})\in\argmin_{(\bsbS, \bsbV)}G(\bsbS, \bsbV; \bsbS^{[k-1]},\bsbV^{[k-1]}), \label{seqdef1}
\end{align}
or just
\begin{align}
G(\bsbS^{[k]}, \bsbV^{[k]}; \bsbS^{[k-1]},\bsbV^{[k-1]})\le G(\bsbS^{[k-1]}, \bsbV^{[k-1]}; \bsbS^{[k-1]},\bsbV^{[k-1]}).\label{seqdef2}
\end{align}
We can show that   whenever $\rho$ is chosen large enough,   \begin{align}
f(\bsbS^{[k]}, \bsbV^{[k]})&\leq G(\bsbS^{[k]},\bsbV^{[k]}; \bsbS^{[k-1]},\bsbV^{[k-1]}), \label{surrogate_ineqs0}
\end{align}
from which it follows that $f(\bsbS^{[k]}, \bsbV^{[k]})\leq G(\bsbS^{[k-1]}, \bsbV^{[k-1]}; \bsbS^{[k-1]},\bsbV^{[k-1]}) =f(\bsbS^{[k-1]}, \bsbV^{[k-1]})$.
 A conservative choice is   $\rho  = L\|\bsbX \|_2^2$ (cf. Appendix \ref{appsub:impldetails}), but the structural parsimony  in $\bsbB^{[k]}$ makes it possible to pick a    much smaller $\rho$, which is beneficial  from  Theorem \ref{th_local}.
  Let $\overline\kappa_2(q,r)$  satisfy  $\Breg_{2} ( \bsbX   \bsbB_1, \bsbX   \bsbB_2  ) \le    \overline\kappa_2(q, r) \Breg_2(   \bsbB_1, \bsbB_2)$, for  $ \bsbB_i  :   \|\bsbB_i \|_{2, \mathcal C}\le q, \rank(\bsbB_i) \le r$. Summarizing the above derivations gives the following computational convergence. 

\begin{theorem} \label{th:compconv}Given  any feasible initial point $(\bsbS^{[0]},\bsbV^{[0]})$, the sequence of iterates $(\bsbS^{[k]},\bsbV^{[k]})$ generated from \eqref{seqdef1} or \eqref{seqdef2} satisfies $$ f(\bsbS^{[k-1]},\bsbV^{[k-1]}) -f(\bsbS^{[k]},\bsbV^{[k]}) \ge \frac{\rho -L\overline{\kappa}_2(q, r) }{2} \|\bsbB^{[k]} - \bsbB^{[k-1]} \|_F^2   $$
 for any  $  k \geq 1$, where  $\bsbB^{[k]} = \bsbS^{[k]} (\bsbV^{[k]})^T$.
Therefore, if $\rho >L    \overline{\kappa}_2(q, r)$, $ f(\bsbS^{[k]},\bsbV^{[k]})$ is monotonically decreasing, $ \bsbB^{[k]} - \bsbB^{[k-1]} \rightarrow 0$ as $k\rightarrow +\infty $ and $ \min_{1\le k \le K}\|\bsbB^{[k]} - \bsbB^{[k-1]} \|_F^2\le \frac{1}{K}\cdot\frac{2 f(\bsbS^{[0]},\bsbV^{[0]}) }{ \rho -L\overline{\kappa}_2(q, r)}$, $\forall K\ge 1$.
\end{theorem}
When the value of      $L$ is unknown,         \eqref{surrogate_ineqs0} can  be used for line search to get a proper  $\rho$  in implementation.
After some simple algebra,  the optimization problem in \eqref{seqdef1} is
\begin{equation}\label{sol2_3}
\min_{\bsbS \in \mathbb R ^{ p \times r},\bsbV \in \mathbb R ^{ m \times r}} \frac{1}{2} \big\|  \bsbB^{[k-1]}-\frac{\bsbX^T\nabla l_0(\bsbX\bsbB^{[k-1]})}{\rho} -  \bsbS\bsbV^T\big\|_F^2 + \iota(\bsbS) \quad \mbox{s.t.}\ \   \bsbV^T \bsbV =\bsbI.
\end{equation}
 \eqref{sol2_3}  is   the   unsupervised CRL problem. There is no need to solve the problem in depth though;  we   use
block coordinate descent (BCD) to get some  $(\bsbS^{[k]}, \bsbV^{[k]})$ that satisfies
 \eqref{seqdef2}.
Let
\begin{align}
\widetilde{\bsbY}=\bsbB^{[k-1]}-\bsbX^T{\nabla l_0(\bsbX\bsbB^{[k-1]})}/{\rho},\  \  \bsbW = \widetilde{\bsbY}^T \bsbS , \ \ \bsbL= \widetilde{\bsbY}\bsbV.
\end{align}
First,  with $\bsbS$ held fixed, a globally optimal $\bsbV$ can be  obtained by Procrustes rotation: $\bsbV = \bsbU_w \bsbV_w^T$, where $\bsbU_w$ and $\bsbV_w$ are from the SVD of
$
\bsbW. 
$
Equivalently, $\bsbV = \{(\bsbW \bsbW^T)^{+}\}^{1/2} \bsbW.$
The Procrustes rotation     simplifies to a normalization operation $ \bsbW  /\| \bsbW\|_2$  when $r=1$.

Next, we solve for $\bsbS$ given $\bsbV$. Using the orthogonal decomposition $\|\widetilde{\bsbY}-\bsbS\bsbV^T \|_F^2=\allowbreak \|\bsbL-\bsbS \|_F^2 + \| \widetilde{\bsbY}\bsbV_\perp \|_F^2 , $ the problem  reduces to a low-dimensional one:  $$
\min_{\bsbS \in \mathbb R^{p \times r}}     \|  \bsbL-\bsbS  \|_F^2 +\iota(\bsbS).
$$
Consider $\|\bsbS\|_{2, C}\le q$ first. Let  $\bsbS  =\bsbF\bsbmu$, where $\bsbmu \in \mathbb R^{q \times r}$ stores the   $q$ cluster centroids, and $\bsbF\in \mathbb R^{p\times q}  $ is the associated binary  membership matrix, with $\bsbF[j,k]=1$ indicating that the $j$-row of $\bsbL$ falls into the $k$-th cluster, i.e.,  $\bsbF \in \mathcal F^{p\times q}:= \{\bsbF \in\mathbb   R^{p\times q}: \bsbF\ge 0, \bsbF \bsb{1} = \bsb{1}, \bsbF^T \bsbF \mbox{ is diagonal} \}$.  BCD can be used to update $\bsbF$ and $\bsbmu$ alternatively: given $\bsbF$, the optimal
$\bsbmu$ is $ (\bsbF^T \bsbF)^{-1}\bsbF^T \bsbL$, while  given $\bsbmu=[\bsbmu_1,\cdots,\bsbmu_q]^T$, it suffices to solve        $\min_{ \bsbf\in \{0, 1\}^q, \bsb{1}^T\bsbf=1  } \|\bsbL[j,:] -\bsbf^T\bsbmu \|_2^2 $,   $1\le j\le  p$, from which it follows that $\bsbF[j,c_0]=1 $ for $c_0=\argmin_{1 \leq c \leq q} \| \bsbL[j,:] - \bsbmu_c \|_2^2$,   and $\bsbF[j,c]=0$ for any $c \neq c_0$. The  algorithm   turns out to be    K-means.  Similarly, for the rank-wise   constraint $\|\bsbs_i\|_{\mathcal C} \leq q^e$ ($1 \leq i \leq r$), we just need to run K-means on each  column of $\bsbL$ to get   $\bsbS$.  State-of-the-art implementations of K-means    skillfully  use  initialization strategies
and    usually give a     high-quality or even globally optimal  solution  in low dimensions \citep{zhang2009k,bachem2016fast}. Of course, other unsupervised clustering criteria and algorithms can be seamlessly integrated into the framework.

To sum up, the CRL    algorithm   performs simultaneous dimension reduction and clustering and  has guaranteed convergence. Regarding  the per-iteration complexity, apart from some fundamental matrix operations, the algorithm involves the SVD of $\bsbW$ and the clustering on $\bsbL$. Neither is costly in computation, since $\bsbW$ and $ \bsbL$  have only $r$ columns.


\subsection{Parameter tuning}
\label{subsec:tuning}
CRL has two  regularization parameters   $q $ and $r$; once they are given,  CRL can  determine the  model structure that fits  best to the data, including the cluster sizes and the projection subspace. In many applications, we find it     possible  to  directly specify  these bounds  based on domain knowledge, and    they are not very sensitive parameters. But for the sake of cluster and/or rank selection,  one must carefully tune the regularization parameters in a data-adaptive manner.
The goal of this subsection is to design a proper model comparison  criterion assuming a series of candidate models have  been obtained (rather than  developing a numerical optimization algorithm).

It is well known that  parameter tuning is quite  challenging  in the  context of clustering.  AIC, BIC, and many other known information criteria do not seem to work well, and what makes a sound complexity penalty term   is a notable open problem.    Fortunately, the statistical studies    in Section \ref{sec_theory}    shed   some  light on the topic. We advocate  a new  model penalty  $P_{o}(\cdot)$ as follows
\begin{align}
P_{o}(\bsbB) =  \{\|\bsbB\|_{2, \mathcal C}\wedge \rank(\bsbX) +m  \} \rank(\bsbB) +\{p -\|\bsbB\|_{2, \mathcal C}\}\log \|\bsbB\|_{2, \mathcal C}. \label{picpenpo}
\end{align}   

\begin{theorem} \label{th:pic}
        Given any  differentiable loss $l_0$, assume that  $\vect (\bsbE)$ (cf. \eqref{noise-def}) is sub-Gaussian  with mean zero and scale bounded by $\sigma$, $\bsbB^*\ne \bsb0$ and        there exist constants $\delta>0$ and $C\ge 0$ such that $\bm\Delta_{l_0}(\bsbX\bsbB_1,\bsbX\bsbB_2) +C \sigma^2 P_o( \bsbB_1) + C \sigma^2 P_o(\bsbB_2)\ge \delta\Breg_2(\bsbX\bsbB_1,\bsbX\bsbB_2)$ for all $\bsbB_i$.   Then,   any $\hat \bsbB$ that minimizes the following criterion
        \begin{align}
        l_0 (\bsbX \bsbB; \bsbY)  + A \sigma^2 \big[\{\|\bsbB\|_{2, \mathcal C}\wedge \rank(\bsbX) +m  \} \rank(\bsbB) +\{p -\|\bsbB\|_{2, \mathcal C}\}\log \|\bsbB\|_{2, \mathcal C}\big],  \label{pic-gen}
        \end{align} where $A$
        is a sufficiently large constant, must satisfy    $$ \EE\big[ \|  \bsbX \hat \bsbB  - \bsbX \bsbB^* \|^2_F  \vee P_o(\hat \bsbB)\big]
        \lesssim    \sigma^2\big \{(q^*\wedge \rank(\bsbX) +m  ) r^* +(p -q^*)\log q^*\big\}. $$
\end{theorem}

Compared with the results   in Section \ref{subsec:upperrCRL},   Theorem  \ref{th:pic} offers the same desired order of statistical accuracy, but   involves {no} regularization parameters. We refer to the new information criterion defined by \eqref{pic-gen} as the     predictive information criterion (\textbf{{PIC}}). Unlike   most information criteria, PIC has a nonasymptotic justification, and  does not need  $n\rightarrow +\infty$ or   any growth conditions on $p $ or $m$. The new criterion aims to achieve the best prediction accuracy, and applies regardless of  the signal-to-noise ratio.

If the effective noise has a constant scale parameter like in classification,  \eqref{pic-gen} can be directly used. But some problems have an unknown $\sigma$. For example,      $\bsbY | \bsbX \bsbB^*$ may belong to  the  exponential dispersion family with a density   $$\exp\big[   \{\langle \;\cdot\;,  \bsbX\bsbB^* \rangle -  {b}(\bsbX\bsbB^* )\}/\phi\big]$$ with respect to some base measure.  For such models with dispersion, the standard  practice is to substitute a preliminary estimate $\hat\sigma^2$ for   $\sigma^2$ in \eqref{pic-gen}, but a fascinating  fact is that in some scenarios like regression ($b = \|\cdot \|_F^2/2$),   the estimation of $\sigma^{2}$ can be totally  bypassed   with a scale-free form of PIC.

Recall the Orlicz $\psi_\alpha$-norm \citep{van1996weak} defined for a random variable $Y$: $\| Y\|_{\psi_\alpha}     = \inf  \big\{ t>0: \EE \exp[( Y/{t})^{\alpha}]\le 2 \big\}$. Sub-Gaussian  random variables have finite    $\psi_2$-norm, and sub-exponential random variables (like Poisson and $\chi^2$) have finite $\psi_1$-norm. As $\alpha<1$,  random variables with even heavier  tails are included \citep{gotze2019concentration}.

\begin{theorem}\label{cor:sf-pic}
Let  the loss  be
\begin{align} \label{glmloss-unscaled}
l_0(\bsbX \bsbB; \bsbY) = -\langle \bsbY , \bsbX \bsbB \rangle + b(\bsbX \bsbB),
\end{align}
where   $b$ is differentiable,    $\mu$-strongly convex and $\mu'$-strongly smooth with $\kappa =\mu'/\mu$, the domain $\Omega = \{\bsb{\eta} \in \mathbb R^{n\times m}: b(\bsb{\eta})< \infty\}$ is open, and $\bsbY$ takes values in the closure of $\{\nabla b(\bsbe): \bsbe\in \Omega\}$. Assume the effective noise $\bsbE = \bsbY - \nabla b(\bsbX \bsbB^*)$    has independent, zero-mean entries $e_{ik}$ that satisfy  $\|e_{ik}\|_{\psi_\alpha}\le \sigma$ for some $\alpha\in (0, 2]$  and are nondegenerate in the sense that $var(e_{ik})\asymp    \sigma^2$, where $\sigma$ is  an unknown parameter. Suppose that the true model is not over-complex in the sense that $\kappa P_o(\bsbB^*) \allowbreak <  mn / A_0$ for some  constant $A_0>0$.
        Let $\delta(\bsbB)=A \{  P_o(\bsbB)/( mn/\kappa  )\}$ for some constant $A: A<A_0$, and so $\delta(\bsbB^*)<1$. Consider the following criterion 
        \begin{align}
       \frac{ {l_0(\bsbX \bsbB; \bsbY) +  b^*( \bsbY)} }{  1 - \delta(\bsbB) },\label{eq:sf-pic}
        \end{align}
        where $b^*(\cdot ) = \sup_{\bsb{\eta}}  \langle \cdot, \bsb{\eta}\rangle  - b(\bsb{\eta})$ is the \emph{Fenchel conjugate} of $b(\cdot)$.

Then, for sufficiently large  values of $A_0,A$,   any $\hat\bsbB$ that minimizes \eqref{eq:sf-pic} 
        subject to $\delta(\bsbB)<1$    satisfies $$  \| \bsbX \hat \bsbB - \bsbX \bsbB^*\|_F^2 \vee \frac{P_o(\hat \bsbB)}{\mu^2}\, \lesssim \, \frac{\kappa\sigma^2}{\mu^2} \{(q^*\wedge \rank(\bsbX) +m  ) r^*   +(p -q^*  )\log q^*\}$$  with probability at least $1 -C \exp\{ -c  (m+  P_o(\bsbB^*))\}-C \exp \{ -c   (mn)^{\alpha/2} \}$,  or          $1 - \allowbreak  C \exp\{- c\allowbreak(m+\rank(\bsbX))\}  -C \exp \{ -c   (mn)^{\alpha/2}  \}$ as $q^* > 1$,  for some
        constants $C,  c >0$.
\end{theorem} 

The theorem needs no restricted eigenvalue or signal strength   assumptions. Some other scale-free forms can be used based on the techniques in  \cite{SCV}; see Remark    \ref{rmk:pic}.    



\section{Experiments}
\label{sec_dataexp}
We performed extensive simulation studies which, due to limited space, are   presented in the appendices.  The results show the benefits of CRL:   the desired structural parsimony  can be successfully captured by  simultaneous clustering and dimension reduction, and the  removal of  nuisance dimensions leads to improved statistical performance and reduced computational cost. We also tested kernel  CRL on a variety of benchmark datasets (cf.  Figure \ref{kernel_res} and Tables \ref{sim_uic}--\ref{sim_mbar2}) and performed  experiments in network  community detection 
(cf. Figure \ref{lfr} and Table \ref{tab_GN}). In addition, simulation studies were conducted to  test the performance of CRL when model misspecification occurs, compared with  LASSO, group LASSO, reduced rank regression and fused LASSO \citep{tibshirani2005sparsity}  (cf. Table \ref{sim_table_misspec}). Interested readers may refer to    Appendix \ref{app:exp} for details.
Here, we use   two real datasets to demonstrate the performance  of  CRL in supervised learning. Our code can be found in the supplementary material\footnote{Also available at \url{https://ani.stat.fsu.edu/~yshe/code/CRL.zip}}.
\subsection{Horseshoe crab data}\label{subsec:horseshoe}
 This part  performs  ``model segmentation''  on a horseshoe crab  dataset  in   \cite{agresti2012categorical}, to showcase an application of CRL. The response variable is the number of male crabs   residing near a female crab's nest, denoted by {\small\textsf{satellites}}, ranging from 0 to 15.  The dataset  records the number of  satellites  for 173 female horseshoe crabs  in   vector $\bsby$, as well as some covariates in $\bsbX$,  such as  width, color,  weight and the intercept. Here, width  refers to the carapace width of a female crab, measured in centimeters;     color has several categories from light to dark, and  darker female crabs  tend to be {{older}} than  lighter-colored ones. Following Agresti, we removed  some redundant and irrelevant predictors and used {\small\textsf{width}} and a dummy variable {\small\textsf{dark}} to model {\small\textsf{satellites}}.   Fitting a   simple regression model  to the overall data  gives $-10 + 0.5 \cdot {\small\textsf{width}} -0.4 \cdot {\small\textsf{dark}}$.

 An interesting question in  statistical modeling is to study the possible existence of latent ``sub-populations'', across which   predictors have different coefficients. To this end,    we  re-characterize the problem using a trace regression \citep{Kolt11}:
\begin{align}
y_i\sim \langle \bsbX_i, \bsbB \rangle, \ i=1,\ldots, n \label{tracemodel}
\end{align} where $\bsbB\in \mathbb R^{n\times p}$ is a  matrix of unknowns, and $\bsbX_i\in \mathbb R^{n\times p}$ has all rows  zero except the $i$-th row, which is equal to $\bsbX[i,:]$. For the horseshoe crab data,    the 173 rows of   matrix  $\bsbB$ give sample-specific  coefficient  vectors, and the model is clearly overparameterized.  CRL helps to estimate the coefficient matrix  and  identify  a small number of sub-models. After running the optimization  algorithm and parameter tuning, the whole sample is split to two sub-groups ($q=2$). The model on  the first subset (117 observations)  is    \begin{align}
\mbox{model 1: } -9.6 + 0.4 \cdot {\small\textsf{width}} -0.5 \cdot {\small\textsf{dark}}, \label{submod1}\end{align} while   on the second subset (56 observations), we get
\begin{align}\mbox{model 2: } -12 + 0.7 \cdot {\small\textsf{width}} + 2.1 \cdot {\small\textsf{dark}}. \label{submod2}\end{align}

The two resulting models are quite different. For example, for every 1-cm increase in width, \eqref{submod1}   predicts an increase  of 0.4 in the number of satellites, while \eqref{submod2} predicts an increase of 0.7, and the p-values associated with the two slopes are both low   (\textless 3e-4). Also, notice the positive sign of the coefficient estimate for {\small\textsf{dark}} in  \eqref{submod2}. 

 To get more intuition of the two detected sub-populations, we  built a decision tree using CART \citep{breiman1984classification}, which has a pretty simple structure: the prediction outcome is the second sub-population    if $${\small\textsf{(a) satellites}}     \ge 4 \quad \mbox{ and }\quad   {\small\textsf{(b) width}}  <28.7,$$ and the first   if either condition is violated. Therefore,  for the group of female crabs that have at least 4 satellites but do not yet have an extremely large carapace in width,     \eqref{submod2} states that being dark  is actually a beneficial factor in attracting more  satellites.
\subsection{Newsgroup data}

The 20 newsgroup dataset, available at  \url{ftp.ics.uci.edu}, contains about 18k documents   falling into  20  binary categories which we treat  as responses.
The  feature matrix records the  occurrence information of a large dictionary of words. We chose    $p = 200$ words   at random    and used     $n= 2\mbox{,}000$ documents  for training and the remaining  for test.
On this dataset, CRL  produced $q=50$ word groups and constructed      $r=16$  factors.
A prediction error comparison can be made using the test data. The classification accuracy  of          an SVM trained on   the original 200 words is       40.8\%. Using  only the 16 CRL factors improves the rate to   45.6\%, while a  LASSO model
with 16 selected words      only reaches  an accuracy rate of 31.30\%.
%
%

\begin{figure}[h!]
        \centering
        \includegraphics[width=0.9\linewidth]{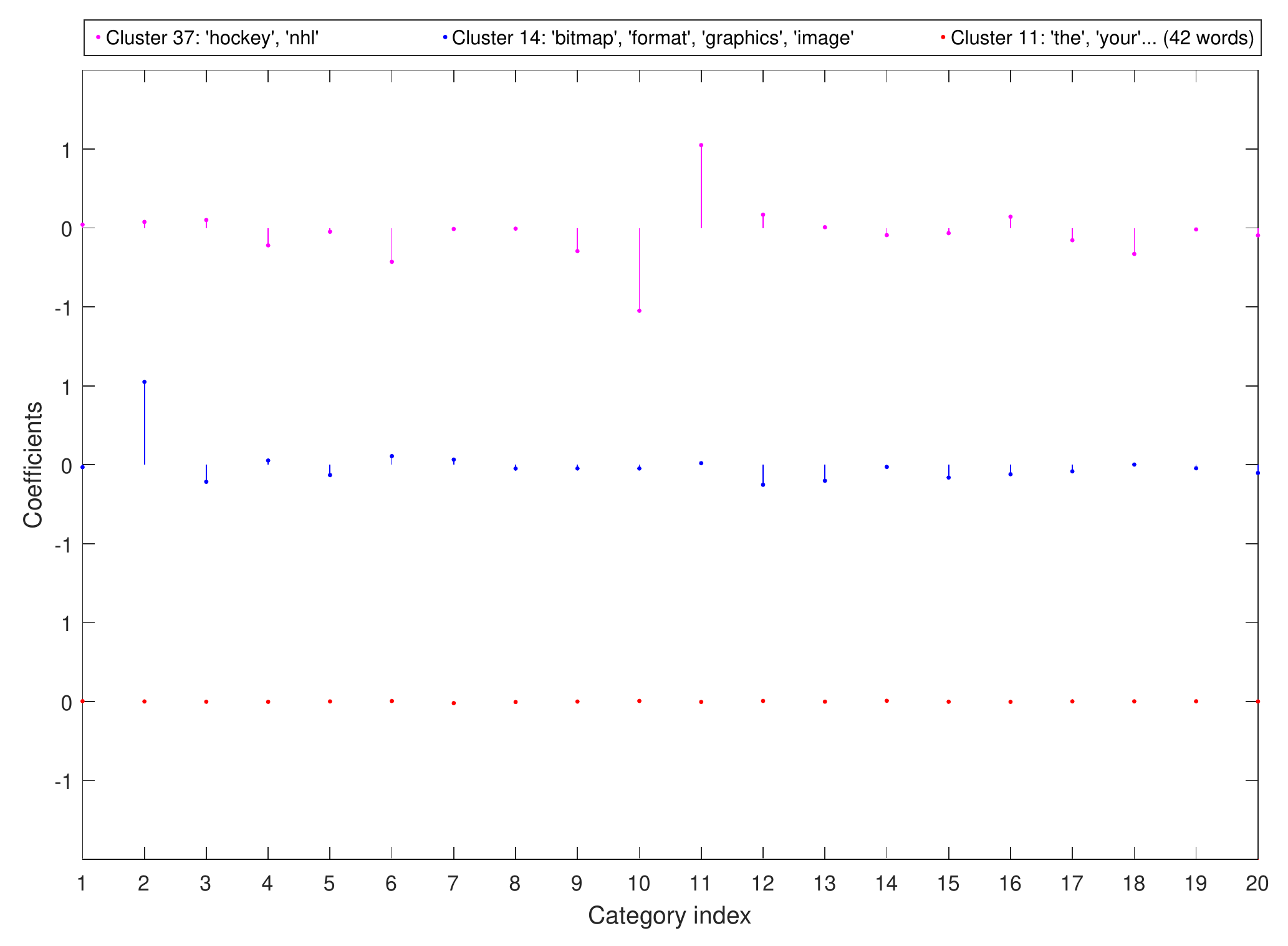}
        \caption{ Newsgroup data: the coefficients of some word clusters obtained by CRL in response to the 20 categories. 
}
        \label{cluster_11_14_37}
\end{figure}

  Next, we study  the interpretability of the CRL model. Figure \ref{cluster_11_14_37} plots the coefficients for three clusters for illustration purposes.
Note that we did not use any   available word groups from the literature,  which may or may not be useful for modeling the  responses here.

  First,   cluster 14, composed  of words {\small\textsf{bitmap}}, {\small\textsf{format}}, {\small\textsf{graphics}}, and {\small\textsf{image}},  shows  a single  large  coefficient in response to category 2.  This is sensible, as  the documentation shows that  the category corresponds to   computer graphics.

Cluster 37 contains two words only, {\small\textsf{hockey}} and {\small\textsf{nhl}}.  This group has   two big coefficients in magnitude, $+2.25$ and $-1.05$  for the  categories of   hockey and baseball, respectively.  So the occurrence of these two words   seems helpful  for      {differentiating}
the two related sport categories.

Finally, let's turn to cluster 11  which   consists of 42 words. All its coefficients  are pretty small, varying  between $-0.02$ and $ 0.01$ for   different  responses. A careful examination of its composition explains the mild effects:    almost all are   the so-called `{stop words}', such as {\small\textsf{the}},  {\small\textsf{very}}, and {\small\textsf{yours}}, and removing this cluster gave almost identical results.   CRL was able to capture these essentially irrelevant features and group them together.
To sum up, CRL        contributes as a beneficial complement to   conventional variable selection. 

\section{Conclusions}
\label{sec:summ}
Many high-dimensional methods   adopt the ``bet  on sparsity''   principle  \citep{ESL2},      but in  real multi-response applications, statisticians often face ``dense'' problems  with such a large number of relevant features  that variable selection may   be ineffective.
 This paper proposed a clustered reduced-rank learning  framework  to build  a predictive and interpretable model through feature auto-grouping and dimension reduction. 

The joint matrix regularization formulation   poses intriguing challenges in both theory and computation. We provided  universal information-theoretical limits  to  reveal the intrinsic cost of seeking for clusters, as well as the benefit of accumulating a large number of response variables in multivariate learning. The obtained error rates   are   strikingly different from those  assuming sparsity. Moreover,   we proved that CRL, unlike   the class of methods  based on pairwise-difference penalization,  achieves the minimax optimal  rate    in some common scenarios.   The remarkable
fact that the CRL estimators need not be global minimizers  but just fixed points in some regular problems paved the way for the design of an efficient optimization algorithm in the nonconvex setup. Furthermore, a new information criterion, along with its scale-free form, was proposed to address  cluster and rank selection.  Overall, our new method is as interpretable as variable selection, and     is advantageous when the numbers of relevant features and irrelevant features are of the same order,   or when   the number of responses is greater than   the number of  features up to a multiplicative constant.

CRL can be extended to tensors, and one possible application is    model segmentation  in a multi-task setting. For example, given $\bsbY = [\tilde \bsby_1, \ldots \tilde \bsby_n]^T$, $\bsbX = [\tilde \bsbx_1, \ldots \tilde \bsbx_n]^T$ and an unknown order-3 tensor  $\bsbB\in \mathbb R^{p\times m \times n}$, we can  fit a model $\tilde \bsby_i^T \sim    \tilde \bsbx_i^T \bsbB_{::i}$ ($1\le i\le n$) by enforcing low rank in $\bsbB$ and equisparsity along its third dimension. This can be applied to    heterogeneous populations. Moreover, our  algorithm often shows a  linear convergence rate for small  $q$ and $r$, which deserves  further study.
Finally, to  reduce the  search cost (thereby the overall error rate), one possible way is to limit the min/max cluster size; a   new form of regularization (convex or nonconvex) that guarantees both interpretability and efficiency  is an interesting topic that merits future research.




\appendix
\appendixpage
\numberwithin{equation}{section}
\numberwithin{theorem}{section}
\numberwithin{table}{section}
\numberwithin{figure}{section}
\section{Technical Details}
\label{app:proofs}

Throughout the proofs, we use $C$, $c$, $L$ to denote    positive constants unless otherwise mentioned. They are not necessarily the same at each occurrence. Given any matrix $\bsbA$, we use $cs(\bsbA)$ and $rs(\bsbA)$ to denote its column space and row space, respectively. Denote by    $\Proj_{\bsbA}$  the orthogonal projection matrix onto   $cs(\bsbA)$, i.e., $\Proj_{\bsbA}=\bsbA(\bsbA^{T}\bsbA)^{+}\bsbA^{T}$, and by $\Proj_{\bsbA}^{\perp}$ the projection onto  its orthogonal complement.  We use  $\bsbA[\mathcal I, \mathcal J]$ to  denote   a submatrix of $\bsbA$ with  rows and columns indexed by $\mathcal I$ and $\mathcal J$, respectively. The standard vectorization operator is denoted by  $\vect(\cdot)$.
Let $\mathbb Z$ be the set of integers and   $[p]$ be $\{1,\ldots,p\}$. Given $x\in \mathbb R$, $\lfloor x \rfloor $ and $\lceil x \rceil$ denote the floor and ceiling functions, respectively. \textit{Throughout the whole section, we
   abbreviate  $\| \bsbB\|_{2, \mathcal C}$ as $q(\bsbB)$, and $\rank(\bsbB)$ as $r(\bsbB)$.}
\\

Beyond the regression model, the canonical generalized linear models (GLMs) are widely used in statistics modeling. Here,  $\bsb{Y}\in \mathcal Y^{n\times m}  \subset \mathbb R^{n\times m}$ has density $P_{ \bsb{\eta}}(\cdot )=\exp\{(\langle \cdot, \bsb{\eta}\rangle -  {b}(\bsb{\eta} ))/\sigma^{2}  - c(\cdot,\sigma^2)\}\ $ with respect to measure $\nu_0$ defined on $\mathcal Y^{n\times m}$ (typically the counting measure or   Lebesgue measure), where   $\bsb{\eta}\in \mathbb R^{n\times m}$ represents the systematic component of interest, and $\sigma$  is the scale parameter; see  \cite{Jor87}.
Since $\sigma$ is not the parameter of interest, it is more convenient to define the density $\exp\{   (\langle \cdot, \bsb{\eta}\rangle -  {b}(\bsb{\eta} ))/\sigma^{2}\}$ (still written as $P_{ \bsb{\eta}}(\cdot )$ with a slight abuse of notation) with respect to the base measure  $\rd \nu = \exp(- c(\cdot,\sigma^2))\rd \nu_0$.    The  scaled loss for   $\bsb{\eta}$  can be written as
\begin{align}\label{l0glmass}
l_0 ( \bsb{\eta} ; \bsb{Y}  ) =  \{-\langle \bsb{Y},   \bsb{\eta}\rangle + b(\bsb{\eta} )\}/\sigma^2 .\end{align} That is, $l_0$  corresponds to  a distribution in the exponential dispersion family with  cumulant function $b(\cdot)$,  dispersion $\sigma^2$ and natural parameter $\bsb{\eta} $.
In the Gaussian case, $b (\cdot ) = \| \cdot \|_F^2/2$.
Define the natural parameter space (assumed to be nonempty)
\begin{align}\label{glmdomainb}\Omega = \{\bsb{\eta} \in \mathbb R^{n\times m}: b(\bsb{\eta})< \infty\}.
\end{align}  When $\Omega$ is open,    $p_{\bsb{\eta}}$ is called regular, and   $ b$ can be shown to be differentiable to any order and  convex, as is well known in the statistics literature (see, e.g.,  \cite{WainJordan08}).

\subsection{Minimax lower bounds}
\label{subsec:proofminimax}

We prove a more general theorem including  the lower bounds for both the
estimation error
and the prediction  error.



\begin{theorem}
        \label{th:minimax-complete}

        Assume $\bsbY | \bsbX \bsbB^*$ follows  a distribution in the regular  {exponential   family} with dispersion $\sigma^2$ with      $l_0$ the associated  negative log-likelihood  function defined in   \eqref{l0glmass},     and define a signal class by    $$\bsbB^*\in \mathcal S(q, r)= \{\bsbB \in \mathbb R^{p\times m}:   q(\bsbB)\leq q, r(\bsbB)\leq r \},$$ where      $p   \ge q\ge   r \ge 2$,    $r(q\wedge  r(\bsbX)+m - r )\ge 4$. Let $b, \zeta$ be any integers satisfying
        $ \sum_{i=0}^{\zeta} {r\choose i} (b-1)^i\ge q$, with $
        b\ge 2 ,  1\le \zeta \le r$, and define a complexity function   \begin{align}
        P (q, r) =   (q +m  ) r  + p      \{\log  (er) -\log \log q\}. \label{specailminimaxrate-complete}
        \end{align}

        (a) Assume that  $  \breg_{l_0} (\bsb0, \bsbX   \bsbB  )\sigma^2 \le
        \kappa\Breg_2(   \bsb0,      \bsbB)$ for any $\bsbB  \in   \mathcal S(q, r)$.  Then  there exist positive constants $c, c'$, depending on $I(\cdot)$ only,  such that
        \begin{align}
        \inf_{\hat \bsbB} \sup_{\bsbB^* \in \mathcal S(q, r)} \EE\left\{I\Big(\|  \bsbB^* -   \hat\bsbB\|_F^2{\big/}\Big[c \sigma^2 \big\{(  q +m  ) r  + \frac{p\log q}{b^2 \zeta}\big\}/\kappa\Big]\Big)\right\} \geq c' >0, \label{estgeneralminimaxrate-complete}
        \end{align}
        In particular, under   $8\le q \le \exp(r)$,
        \begin{equation}
        \inf_{ \hat \bsbB}\,\sup_{ \bsbB^*\in \mathcal S(q,r)} \mathbb E\big[I\big( \|  \bsbB^* - \hat\bsbB\|_F^2/\{ c \sigma^2  P(q, r) / \kappa\}\big)\big] \ge c' >0.
        \end{equation}

        (b) Let the SVD of $\bsbX$ be  $  \bsbU \bsbD \bsbV^T$ and  define  $\bsbZ= \bsbU\bsbD\in\mathbb R^{ n\times r(\bsbX)}$.
        Assume  the following  restricted condition-number condition:  when $q\le r(\bsbX)$,  there exist $
        \underline{\kappa},  \overline{\kappa}$   such that  $
        \underline{\kappa} \Breg_2(     \bsbB_1,    \bsbB_2)\leq  \Breg_2( \bsbX    \bsbB_1,   \bsbX \bsbB_2)$, $ \breg_{l_0} (  \bsb0,  \bsbX   \bsbB_1)\sigma^2  \leq  \overline{\kappa}  \Breg_2(    \bsb0,        \bsbB_1)
        $
        for any $\bsbB_i\in \mathcal S(q,r)$ and $
        \underline{\kappa} / \overline{\kappa}$ is a positive constant; when $q>r(\bsbX)$,   $
        \underline{\kappa} \Breg_2(     \bsbC_1,    \bsbC_2)\leq  \Breg_2(  \bsbZ   \bsbC_1,    \bsbZ\bsbC_2)$, $ \breg_{l_0} (  \bsb0,\bsbZ \bsbC_1   )\sigma^2
        \leq   \overline{\kappa}  \Breg_2(     \bsb0,     \bsbC_1  )
        $
        for any $\bsbC_i:   r(\bsbC_i)\le r$ and   $
        \underline{\kappa} / \overline{\kappa}$ is a positive constant.

        Then  there exist  constants $c, c'>0$, depending on $I(\cdot)$ only,  such that when $q\le r(\bsbX)$
        \begin{align}
        \inf_{\hat \bsbB} \sup_{\bsbB^* \in \mathcal S(q, r)} \EE\left\{I\Big(\|  \bsbX\bsbB^* -   \bsbX\hat\bsbB\|_F^2{\big/}\Big[c \sigma^2 \big\{(  q +m
) r
+ \frac{p\log q}{b^2 \zeta}\big\}\Big]\Big)\right\} \geq c' >0, \label{generalminimaxrate-complete}
        \end{align}
        and when $q> r(\bsbX)$,
        \begin{align}\inf_{\hat \bsbB} \sup_{\bsbB^* \in \mathcal S(q, r)} \EE(I[\|\bsbX \bsbB^* - \bsbX \hat\bsbB\|_F^2/\{c \sigma^2  (r(\bsbX)  +m  ) r \}]) \geq c' >0.
        \end{align}
        In particular, when  $8\le q \le \exp(r)$ and $q\le r(\bsbX)$,
        $$
        \inf_{\hat \bsbB} \sup_{\bsbB^* \in {\mathcal S}(q, r)} \allowbreak \EE\big[\|\bsbX \bsbB^* - \bsbX \hat\bsbB\|_F^2\big] \geq c \sigma^2 P (q, r).
        $$
\end{theorem}

\begin{proof}
        We first introduce some standard notations and symbols in coding theory.
        Let $V_q(n, d)$ denote the volume of a Hamming ball of radius $d\le q$ in $\{0, 1, \ldots, q-1\}$, i.e.,
        $$
        V_q(n, d) =   \sum_{i=0}^{\lfloor d\rfloor} {n\choose i} (q-1)^i.
        $$
        Given $q\ge 2$, $x \in [0, 1]$, define the $q$-ary entropy function    $$h_q(x) = x \log_q (q - 1) - x \log_q x - (1-x) \log_q (1-x),$$
        and $  h_2(x)$ is the Shannon entropy function. It is easy to see some basic facts:  $ h_q(x)
        \log q \allowbreak + \mbox{KL}  (x \|  \frac{q-1}{q})=\log q$,   $h_q(0) = 0$, $h_q(1-1/q) = 1$ and   $h_q(\cdot)$ is continuous
        and increasing in $[0, 1-1/q]$. Here, the Kullback-Leibler divergence $\mbox{KL}(x\|  y) =  x \log (x /y ) + (1-x) \log ((1-x)/(1-y))$ for $x, y \in[0, 1]$.

        The following result is well known; see, e.g., \cite{van2012introduction}. 
        \begin{lemma}\label{volbounds}
                Let $q\ge 2$,   $d\in [0, n(1-1/q)]$, and $q, d\in \mathbb Z$. Then
                $$
                V_q(n, d) \le q^{h_q(\theta) n}, \
                V_q(n, d)\ge {n\choose d} (q-1)^d   \ge q^{ h_q(\theta)   n } \exp(- c \log  n -c'),
                $$ 
                where    $\theta = d/n$, and $c, c'$ are positive constants.
        \end{lemma}

        The   next lemma   is essentially the   Gilbert-Varshamov bound for $q$-ary codes, adapted for our purposes.

        \begin{lemma}
                Let $\Omega = \{\bsba= [a_1, \ldots, a_n]^T: a_j \in \mathcal A  \}$, where   $|\mathcal A| = q$, $n\ge q > 1 $. Then there exists a subset $ \{\bsba^0, \ldots, \bsba^M\} \subset \Omega$ such that $\bsba^{0}\in \mathcal A^n$ is arbitrarily chosen,      and
                \begin{align}
                &\log M \ge \log (q^n/V_q(n, \lceil d \rceil-1)-1)\ge c_1 n\log q , \\&  \rho (\bsba^{j}, \bsba^{k}) \ge c_2 n , \ \forall 0\le j < k  \le M,
                \end{align}
                where $\rho(\bsba, \bsba') = \sum_{i=1}^n 1_{a_j \ne a_j'}$ and $c_1, c_2$ are universal positive constants.\label{qarysepa}
        \end{lemma}


        \begin{proof}
                We can use   the `greedy algorithm' \citep{tsybakov2009introduction,pellikaan_wu_bulygin_jurrius_2017} and Lemma
                \ref{volbounds} to prove the result. Some details are as follows.

                Let $d = c_2 n$.  Given any $\bsba\in \Omega$, the number of elements in $\{ \bsbb\in \Omega$:
                $\rho (\bsbb , \bsba) \le  l
                \}=:   B_{q}(\bsba; l)$  is no more than
                $
                \sum_{i=0}^l {n\choose i} (q-1)^i$ or $ V_q(n, l).
                $

                Consider the following procedure to partition $\Omega$. Let $\bsba^0 \in \mathcal A^n$ be arbitrarily chosen and  $\Omega^0 = \Omega$. Given $\bsba^{t}\in\Omega^t$ ($t\ge 0$), construct $L^t = \{\bsba \in   B_{q}(\bsba^t; \lceil  d \rceil-1)\}\cap \Omega^t$ and $\Omega^{t+1} = \Omega^t \setminus L^t$, and choose an arbitrary $\bsba^{t+1}\in \Omega^{t+1}$. Repeat the process until $\Omega^{M+1} = \emptyset$. Then $L^t$ $(0\le t\le M)$ form a partition of $\Omega$, and so
                $$
                (M + 1)V_q(n, \lceil d \rceil -1) \ge q^n.
                $$
                By Lemma \ref{volbounds}, for $d/n\le 1 -  1/q$, or $c_2 \le 1 - 1/q$,
                \begin{align*}
                { M + 1}  \ge   q^{  (1-  h_q(c_2)) n}.    \end{align*}
                It is not difficult  to show that with a small enough $c_2$, $q^{
                        (1-
                        h_q(c_2)) n} \ge 1 +
                q^{c_1 n}$ holds for  all $  n\ge q\ge 2$ and  some constant $c_1>0$. The conclusion  follows.
        \end{proof}

        \begin{lemma} \label{lemKLBreg} For  $P_{\bsbX\bsbB }$  in the regular exponential family, the Kullback-Leibler divergence of   $P_{\bsbX\bsbB_2} $ from  $P_{\bsbX\bsbB_1} $ satisfies\begin{equation} \nonumber
                \mbox{KL}(  P_{\bsbX\bsbB_1} \|    P_{\bsbX\bsbB_2})   = \breg_{l_0} (   \bsbX  \bsbB_2,   \bsbX  \bsbB_1),
                \end{equation}
                where $\breg_{l_0}(\cdot, \cdot)$ is the generalized Bregman function defined in \eqref{genbregdef}. 
        \end{lemma}
        See Lemma 3(iii) of \cite{SheBregman}.\\

        We prove (b). The proof for (a) is similar and simpler.   Consider three cases.

        \textit{Case (i):} $  (p\log q)/(\zeta b^2) \ge  ( q  +m
        ) r  $ and $q\le r(\bsbX)$.
        Recall that   $b, \zeta$ are  integers satisfying
        \begin{align}V_b(r, \zeta) =\sum_{i=0}^{\zeta} {r\choose i} (b-1)^i\ge q, \
        b\ge 2 ,  1\le \zeta \le r . \label{bzetacond}\end{align}
        Hence  we can pick $q$ $r$-dimensional vectors to form a set $\mathcal C$ satisfying   $$\mathcal |\mathcal C|=q, \ \bsb0\in \mathcal C, \ C \subset\{   \bsbc\in \mathbb R^r:   \| \bsbc\|_0 \le   \zeta, c_k  \in \{ 0, 1, \ldots, b-1\}, 1\le k\le r  \}.$$ Next, construct   \begin{align*}
        {\mathcal B}^1(q, r)   =\{  \bsbB &  =[\bsbb_{1}, \ldots, \bsbb_j,\ldots, \bsbb_p]^T: \bsbb_j[1:r] / ( \gamma R) \in \mathcal C   \mbox{ and }   \bsbb_j[(r+1):m]=\bsb0     \},
        \end{align*}
        where  $\gamma>0$ is a small constant to be chosen later, and $$R=   \sigma ( \log q )^{1/2}/( \zeta b^2 \overline{ \kappa })^{1/2}.      $$
        Clearly,  ${\mathcal B}^{1}(q, r) \subset \mathcal {S}(q,r) $, $r(\bsbB_1 - \bsbB_2) \le r$ and $q(\bsbB_1  - \bsbB_2) \le q^{2} \wedge V_{2b-1}(r, 2\zeta) =q^2$ for any $\bsbB_i \in{\mathcal B}^1(q, r) $.

        Define a   vector version of Hamming distance by  $$\rho(\bsbB_1, \bsbB_2) = \sum_{j=1}^p 1_{\bsbB_1[j,:]\ne \bsbB_2[j,:]}.$$
        From  Lemma \ref{qarysepa},  there exists  a subset ${\mathcal B}^{10}(q, r)\subset {\mathcal B}^{1}(q,r)$ such that $\bsb0\in {\mathcal B}^{10}(q, r)$ and
        \begin{eqnarray*}
                &\log( | {\mathcal B}^{10}(q, r)| -1)\geq \log (q^p/(V_q(p, \lceil d \rceil -1)-1))\ge c_1
                p \log q,
                 \\
                &\rho(\bsbB_1, \bsbB_2) \geq c_2 p, \forall \bsbB_1, \bsbB_{2} \in \mathcal B^{10},  \bsbB_1\neq \bsbB_2
        \end{eqnarray*}
        for some universal constants $c_1, c_2>0$.
        Hence $$\| \bsbB_1 - \bsbB_2\|_F^2 \ge   \rho(\bsbB_1, \bsbB_2) \cdot 1 \cdot \gamma^2 R^2\geq   c_2     \gamma^2 R^2 p, \ \forall \bsbB_1, \bsbB_{2} \in \mathcal B^{10}.$$  From the restricted conditional number assumption,
        \begin{align}
        \| \bsbX \bsbB_1 - \bsbX \bsbB_2  \|_F^2 \geq  c_2  \underline{\kappa}   \gamma^2 R^2 p  =   \frac{c_{2 }  \underline{\kappa}}{ \overline{\kappa} }\frac{p\log q}{\zeta b^2}\sigma^2\gamma^2   \label{separationLBoundcase1}
        \end{align}
        for any $\bsbB_1, \bsbB_{2} \in \mathcal B^{10}$,  $\bsbB_1\neq \bsbB_2$, where $c_2$, $\underline{\kappa}/\overline{\kappa}$ are positive constants.
        \\


        From Lemma \ref{lemKLBreg} and the regularity condition, for any $\bsbB \in {\mathcal B}^{10}(q, r)$, we have
        \begin{align*}
        \mbox{KL} (P_{\bsbX\bsbB}\| P_{\bsb{0}})\le \frac{\overline{\kappa}}{\sigma^2} \Breg_2 ( \bsb0,\bsbB) \leq \frac{1}{2\sigma^2}\overline{\kappa}    (b-1)^{2}\gamma^2 R^2 \rho(\bsb{0}, \bsbB)\zeta \leq \frac{\zeta b^{2}\gamma^2}{2\sigma^2}   \overline{\kappa}R^2 p.
        \end{align*}
        Therefore,
        \begin{align}
        \frac{1}{|\mathcal B^{10}|-1}\sum_{\bsbB\in \mathcal B^{10}\setminus \{\bsb0\}}   \mbox{KL} (  P_{\bsbX\bsbB}\|P_{\bsb{0}}) \leq    \gamma^2  p \log q. \label{KLUBoundcase1}
        \end{align}
        Combining  \eqref{separationLBoundcase1} and \eqref{KLUBoundcase1} and choosing a sufficiently small value for  $\gamma$, we can apply Theorem 2.7 of \cite{tsybakov2009introduction} to get the    desired lower bound  $(p\log q)/(\zeta b^2)$.
        \\

        \textit{Case (ii):} $(  q  +m  ) r \ge  (p\log q)/(\zeta b^2)$ and  $r(\bsbX)\ge q $. Consider a signal  subclass
        \begin{align*}
        {\mathcal B}^2(q,r)=\{\bsbB =  [b_{jk}]: &    \ b_{jk}=0 \mbox{ or } \gamma R   \mbox{ if }  1\le j \le q-1, 1\le k \le r/2 \\
        & \mbox{ or } 1\le j \le r/2, 1\le k \le m, \mbox{ and }   b_{jk}=0 \mbox{ otherwise} \},
        \end{align*}
        where $R=  {\sigma}/{{\overline{\kappa}^{1/2}}  } $ and $\gamma>0$ is a small constant to be chosen later.
        Clearly,  $|{\mathcal B}^2 (q,r)|= 2^{(q-1+m-r/2)r/2}\ge 2^{c (q+m)r}$, $\mathcal B^2 (q,r)\subset \mathcal S(q, r)$, and $r(\bsbB_1 - \bsbB_2)\leq r$, $q(\bsbB_1 - \bsbB_2)\leq q,$ $\forall \bsbB_1, \bsbB_2 \in \mathcal B^2(q,r)$. In this case,
        we define the  elementwise  Hamming distance  by  $$\rho(\bsbB_1, \bsbB_2)=\|\vect(\bsbB_1) - \vect(\bsbB_2)\|_0.$$ 

        By  Lemma \ref{qarysepa}  and  $(q-1+m-r/2)r/2\ge 2$,
        there exists  a subset ${\mathcal B}^{20}(q,r)\subset {\mathcal B}^{2}(q,r)$ such that $\bsb0\in {\mathcal B}^{20},$
        \begin{eqnarray*}
                \log (| {\mathcal B}^{20}(q,r)|-1) \geq c_1 r(q+m),
        \end{eqnarray*}
        and
        \begin{eqnarray*}
                \rho(\bsbB_1, \bsbB_2) \geq c_2 r(q+m), \forall \bsbB_1, \bsbB_{2} \in \mathcal B^{20},  \bsbB_1\neq \bsbB_2
        \end{eqnarray*}
        for some universal constants $c_1, c_2>0$.
        Then  for any $\bsbB_1, \bsbB_{2} \in \mathcal B^{20}$ $$\| \bsbX \bsbB_1 - \bsbX \bsbB_2  \|_F^2 \ge   \underline{\kappa}\gamma^2 R^2 \rho(\bsbB_1, \bsbB_2) \geq c_2 (\underline{\kappa}/\overline{\kappa}) \sigma^2\gamma^2  (q+m) r.$$ Furthermore,
        for any $\bsbB\in \mathcal B^{20}(q,r )$, we have
        \begin{align*}
        \mbox{KL} (P_{\bsbX\bsbB}\| P_{\bsb{0}})  \leq \frac{1}{2\sigma^2}\overline{\kappa}   \gamma^2 R^2 \rho(\bsb{0}, \bsbB) \leq  {\gamma^2}    (q+m)r.
        \end{align*}
        The afterward treatment follows 
        same lines as in (i) and the details are omitted.
        \\

        \textit{Case (iii):} $(r(\bsbX)\wedge q  +m  ) r\ge (p\log q)/(\zeta b^2) $ and  $r(\bsbX)<q $.  In this case, we have a more relaxed  regularity condition that is imposed on the matrix  $\bsbZ$ with full column rank (and so it   can hold even when  $p>n$).  Define  $  \tilde{\mathcal S} (
r) =\{ \bsbC  \in \mathbb R^{r(\bsbX)\times m}: r(\bsbC)\leq r \}$. Then
        for any estimator $\hat\bsbB$, and $\hat \bsbC= \bsbV^T \hat \bsbB$, we have  \begin{align*}
        &\sup_{ \bsbB^*   \in \mathcal S(q, r)} \EP\big[\|\bsbX \bsbB^*  - \bsbX \hat\bsbB  \|_F^2\geq c \sigma^2 \{r(\bsbX)\wedge  q +m  \} r   \big]\\
        \geq \ &\  \sup_{  \bsbC   \in  \tilde{\mathcal S}(  r)} \EP\big[\|\bsbZ\bsbC - \bsbZ \hat\bsbC  \|_F^2\geq c \sigma^2 \{r(\bsbX)\wedge  q +m  \} r\big],
        \end{align*} because $\tilde {\mathcal S}(  r) \subset \{ \bsbV^T \bsbB^*: \bsbB^*\in \mathcal  S(q,r )\}$ under $r(\bsbX)<q$.   
        Repeating the argument in case (ii) gives  the result. (So        for the purpose of prediction,  using a $  q$ greater than $ r(\bsbX)$ is unnecessary.)
        \\

        To show the specific   rate   \eqref{specailminimaxrate-complete} under   $8\le q \le \exp(r)$, we   need to pick proper   $b, \zeta$
satisfying \eqref{bzetacond}. Due to  $$  V_b(r, \zeta)\ge {r\choose \zeta} (b-1)^\zeta  \ge \Big\{\frac{(b-1) r}{\zeta}\Big\}^\zeta, $$ our goal is to
minimize $b^2 \zeta$ subject to  the constraint $(r(b-1)/\zeta)^\zeta\ge q$ or \begin{align}b\ge q^{1/\zeta}\zeta/r + 1. \label{bconstr}\end{align} Introduce $$\phi = q^{1/r}> 1 ,$$
and reparametrize $\zeta$ by  $$  z = q^{1/\zeta}\in [\phi, q].$$ Because
$\zeta = (\log q) /(
        \log z $) and \eqref{bconstr}  means $b \ge (z \log q) / (r \log z) + 1  (\ge 2) $,  $b^2 \zeta$ is greater than or equal to  $$f(z)  = \Big (\frac{z}{r}\frac{\log q}{\log z}+1\Big)^2 \, \frac{\log q}{\log z}.
        $$
        Simple calculation shows $$f'(z) =  {   g(z)(z   \log q  +r \log z)(\log q)^2}/\{z r^{2} (\log z)^3\} $$ where   $$g(z) = 2z -3 \frac{z}{\log z} - \frac{1}{\log \phi}.$$
        It is easy to verify that  $g(\cdot)$ is strictly increasing in its domain $[\phi, q]\subset(1, q]$. Since $\phi \le e \le 4.95$,
 $$g(\phi) =  (2\phi \log \phi - 3 \phi -1) /\log \phi< 0.$$
        Moreover, from as $q\ge 8 >e^2$, $$g(q) = (1/\log q)[ q(2 \log q - 3) - r]>0.$$ Therefore, the minimum of $f$ is achieved at a  point $z_o$ that satisfies $ g(z_o )= 0$. 
  But    \begin{align*}  g(5.3  /\log \phi) &\ge (5.3 /\log \phi) (2 - 3/(\log (1 /\log \phi) + 5.3))-1/\log \phi\\ & \ge   (1/\log \phi) 5.3(2 - 3/
        \log 5.3)-1/\log \phi \ge 0, \end{align*}  so $z_o \le 5.3 r /\log q$.  Therefore,
        $$
        f(z_o) \le f(5.3r/\log q)\le \frac{c\log q}{\log (5.3r) - \log \log q},
        $$
        and the corresponding $\frac{p\log q}{\zeta b^2}\ge  cp ( \log (e
r)  \allowbreak - \log \log q)     $ for some constant $c>0$. \end{proof}

\begin{remark} 
The above scheme can yield other useful bounds in different situations. For example, simple calculation shows $f(\phi)= (q^{1/r}+1)^2 r$, $f(q) = (q/r + 1)^2$. Therefore, we also know that under $q\le r(\bsbX)$,    the minimax rate for the prediction error is  greater than or equal to   $
\sigma^2 \big[(q  +m
) r+  \{ (q^{1/r}+1)^{-2}  r^{-1} \vee(q/r + 1)^{-2} \}p\log q\big].
$
So when $\phi$  (or $ q^{1/r}$)  is sufficiently large, say,    $\phi \ge 5$ or $q \ge 5^r$,    the optimal $z_o$   takes    $ \phi$, and so    the
minimax error rate     is no lower than     $\sigma^2 \{ p   q^{-2/r}  r^{-1}\log q  +  ( q  +m
) r\}   $.
\end{remark}

\subsection{CRL's upper error bounds  }
\label{subsec:prooforacle}

Given any  $\bsbA\in\mathbb R^{p\times m}$ satisfying $q(\bsbA)\le q$, using the symbols introduced in  Section \ref{subsec:algdesign}, we can      write $\bsbA = \bsbF \bsbmu$ with  $\bsbF \in \mathbb R^{p\times q}$, $\bsbmu \in \mathbb R^{q\times m}$. The  binary  membership matrix $\bsbF$    fully characterizes the row clusters of $\bsbA $.   We give it a new notation     $\mathcal Q(\bsbA)$, and abbreviate $\bsbX \Proj_{\mathcal Q(\bsbA)} =\bsbX \bsbF (\bsbF^T \bsbF)^{+} \bsbF^T $  as $\bsbX_{ \mathcal Q(\bsbA)}$. 
Obviously,   $\Proj_{\mathcal Q(\bsbA)} \bsbA = \bsbA$.  Note that $\Proj_{\mathcal Q(\bsbA)} $ is permutation invariant: $\Proj_{\mathcal Q(\bsbA)\bsbP} = \Proj_{\mathcal Q(\bsbA)} $ for any permutation matrix $\bsbP\in \mathbb R^{q\times q}$, namely, the row clusters are unlabeled in nature. Let $\hat r = r(\hat \bsbB)$, $\hat q = q(\hat \bsbB)$, $r^* = r(\bsbB^*)$, $q^* = q(\bsbB^*)$.

First, we prove Theorem \ref{th_local}. Using the $\iota$ notations in Section \ref{sec:comp}, we  write $$G_{\rho}(\bsbB; \bsbB^-) = l_0(\bsbX\bsbB)    - \breg_{l_0} (\bsbX \bsbB, \bsbX\bsbB^-)+    \rho \Breg_2(  \bsbB,  \bsbB^-)+ \iota(\bsbS)+\iota_{\bsbV}(\bsbV).$$
It follows from   $
G_{\rho}(\hat \bsbB; \hat \bsbB) \le G_{\rho}(\bsbB^*; \hat \bsbB)
$
that
$$
l_0(\bsbX \hat \bsbB) \le l_0(\bsbX \bsbB^*)   - \breg_{l_0}  (\bsbX\bsbB^*, \bsbX \hat \bsbB) + \rho \Breg_2(\bsbB^*, \hat\bsbB),
$$
or
\begin{align}
2 \bar \breg_{l_0}(\bsbX \bsbB^*, \bsbX \hat \bsbB) \le \langle \bsbE, \bsbX \hat \bsbB - \bsbX \bsbB^*\rangle+ \rho \Breg_2(\bsbB^*, \hat\bsbB),\label{fixedpoint-auxeq1}
\end{align}
based on the definition of the generalized Bregman function \eqref{genbregdef}.

To make our proof applicable to  oracle inequalities,  we study a more general stochastic term  $$
\langle \bsbE, \bsbX   \bsbDelta  \rangle
$$
where    $\bsbDelta =\hat \bsbB -  \bsbB$ with $q(\bsbB)\le q, r(\bsbB)\le r$.
Let  $\hat{\mathcal Q}= \mathcal Q(\hat\bsbB)$, $\mathcal Q= \mathcal Q(\bsbB)$.
Then $q(\bsbDelta)\le q^2$ and $r(\bsbDelta)\le 2r$. The stochastic term $\langle \bsbE,   \bsbX\bsbDelta \rangle$ must be treated carefully; otherwise   $q^2$ will arise in the error bound.   The following decomposition is  key to our analysis:
\begin{align*}
&\quad \  \langle \bsbE,   \bsbX\bsbDelta \rangle \notag \\ & =    \langle \bsbE,    \bsbX \bsbDelta \Proj_{rs} + \bsbX \bsbDelta \Proj_{rs}^{\perp}\rangle \notag
\\&=    \langle  \bsbE,   \bsbX \bsbDelta \Proj_{rs} + \bsbX \hat \bsbB \Proj_{rs}^{\perp}\rangle\notag\\ & =   \langle  \bsbE,   \Proj_{\bsbX_{\mathcal Q}} \bsbX \bsbDelta \Proj_{rs}\rangle + \langle \bsbE,   \Proj_{\bsbX_{\mathcal Q}}^{\perp}  \bsbX  \hat\bsbB  \Proj_{rs}  \rangle  + \langle \bsbE,      \bsbX  \hat\bsbB  \Proj_{rs}^{\perp}\rangle\notag\\ & =   \langle  \bsbE,   \Proj_{\bsbX_{\mathcal Q}} \bsbX \bsbDelta \Proj_{rs}\rangle + \langle \bsbE,   \Proj_{\bsbX_{\mathcal Q}}^{\perp} \Proj_{\bsbX_{\hat {\mathcal Q}}}  \bsbX  \hat\bsbB  \Proj_{rs}  \rangle  + \langle \bsbE,  \Proj_{\bsbX_{\hat {\mathcal Q}}}   \bsbX  \hat\bsbB  \Proj_{rs}^{\perp}\rangle\notag\\
& =:    \langle \bsbE,    \bsbA_1  \rangle +  \langle \bsbE,   \bsbA_2\rangle+  \langle \bsbE,   \bsbA_3\rangle, 
\end{align*}
where  $\Proj_{rs}$  is  the orthogonal projection onto the  row space   of $\bsbX\bsbB $ and $\Proj_{rs}^{\perp}$  its orthogonal complement.
Because of the orthogonality between $\bsbA_i$, $$\| \bsbA_1\|_{ F}^2 + \| \bsbA_2\|_{ F}^2 + \| \bsbA_3\|_{ F}^2 = \|\bsbX \bsbDelta \|_{F}^2.$$
It is easy to see that   $r(\bsbA_1) \le r( \Proj_{rs})\le r$, $r(\bsbA_2) \le r(\hat \bsbB) \wedge r( \Proj_{rs})\le r$,
$r(\bsbA_3) = r ( \Proj_{\bsbX_{\hat {\mathcal Q}}}   \bsbX  \hat\bsbB  \Proj_{rs}^{\perp})\le \hat r\le r $. 

In the following, we use    ${p\brace q}$ to denote the ``{Stirling number of the second kind}'', i.e.,  the number of ways to partition a set of $p$ labeled objects into $q$ non-empty unlabeled subsets.

\begin{lemma}\label{emprocBnd}
        Suppose that  $ \vect(\bsbE)$ is  sub-Gaussian with mean zero and  scale bounded by $\sigma$.  
        Given $\bsbX\in \mathbb R^{n\times p}$, $1\leq q\leq p$, $1\leq r \leq q$, define  $ \Gamma_{  r, q} = \{  \bsbA: \bsbA\in \mathbb R^{n\times m}, \|\bsbA\|_F  \le 1,  r(\bsbA)\le r, cs(\bsbA) \subset cs(\bsbX_{\mathcal Q   }) \mbox{ for some } \mathcal Q:  q(  \mathcal Q)\le q \}$.  Let $$P(r, q) =    (r(\bsbX)\wedge  q  + m) r  +  \log {p\brace q}.$$
        Then for any $t\geq 0$,
        \begin{align}
        \EP [\sup_{\bsbA \in \Gamma_{r,q}} \langle \bsbE, \bsbA \rangle \geq t \sigma +  \{L\sigma^2   P(r,q)\}^{1/2} ] \leq C\exp(- ct^2),
        \end{align}
        where $L, C, c>0$ are universal constants.
\end{lemma}
\begin{lemma}\label{emprocBnd2}
        Suppose $\vect(\bsbE)$ is  sub-Gaussian with mean zero and  scale  bounded by $\sigma$.  
        Given $\bsbX\in \mathbb R^{n\times p}$, $1\leq q_{1}, q_2\leq p$, $1\leq r \leq  q_{2}$, define  $ \Gamma_{  r, q_1, q_2} = \{  \bsbA: \bsbA\in \mathbb R^{n\times m}, \|\bsbA\|_F  \le 1,  r(\bsbA)\le r, cs(\bsbA) \subset cs( \Proj_{\bsbX_{\mathcal Q_1}^{\perp}} \Proj_{\bsbX_{ {\mathcal Q}_2}}   ) \mbox{ with }  q(\mathcal Q_1)\allowbreak \le q_1,  q(\mathcal Q_2)\le q_2 \}$.  Let $$P(r, q_{1}, q_{2}) =     (r(\bsbX)\wedge  q_{2}  + m) r  +   \log {p\brace q_1} +   \log {p\brace q_2}.$$
        Then for any $t\geq 0$,
        \begin{align}
        \EP [\sup_{\bsbA \in \Gamma_{r,q_1, q_2}} \langle \bsbE, \bsbA \rangle \geq t \sigma +  \{L  \sigma^2 P(r,q_{1}, q_{2})\}^{1/2} ] \leq C\exp(- ct^2),
        \end{align}
        where $L, C, c>0$ are universal constants.  \end{lemma}

\begin{lemma}\label{stirling2bound}
        For any $p\ge 1$, $1\le q \le p$,
        \begin{align}
        c_1 (p -q ) \log q \le \log {p\brace q} \le c_2 (p - q) \log q, \label{st2ndbound}
        \end{align}
        where $c_1, c_2$ are positive constants.
\end{lemma}


Let $P_{o}(q, r) =   (r(\bsbX)\wedge  q  + m) r  + (p - q) \log q  $.
For any   $a, b, a'>0$ and $4b > a$,
\begin{align*}
& \langle \bsbE, \bsbA_1\rangle -   \|\bsbA_1\|_F^2/a -  b L \sigma^2P_{o}(q, r)\\
\leq & \,   \|\bsbA_1\|_F  \langle \bsbE,  \bsbA_1/ \|\bsbA_1\|_F\rangle - 2({{  b}/{a}})^{1/2}  \|\bsbA_1\|_F \{{L\sigma^2 P_{o}(q, r)\}^{1/2}}
\\
\le  &  \,  \|\bsbA_1\|_F^2/{a'} +  ({a'}/{4}) R_{1}^2,
\end{align*}
where    $L$ is a   constant and
\begin{align}
R_1 &:=
\sup_{\bsbA\in \Gamma_{ r, q}}\big [\langle \bsbE, \bsbA \rangle - 2(b/a)^{1/2} \{{L\sigma^2 P_o(q, r)\}^{1/2}}\big]_+. \label{residualterm1}
\end{align}
We claim  an expectation-form bound:     $\EE R_1^2\le C \sigma^2$.
In fact,
from Lemmas \ref{emprocBnd} and \ref{stirling2bound},
\begin{align}\label{eqR1probtailbound}
\begin{split}
&   \EP\big [
\sup_{\bsbA\in \Gamma_{ r, q}}   \langle \bsbE, \bsbA \rangle -\{L\sigma^2  P_o (q, r)\}^{1/2}   \\
& \qquad \ge t \sigma  +   2(b/a)^{1/2} \sigma\{{L P_o(q, r)\}^{1/2}}-\sigma\{{L  P_o (q, r)}\}^{1/2}\big ]  \\
\leq &   C\exp(-c t^2)\exp\Big(- c \big[\{2(b/a)^{1/2}-1\}^2 L  P_o(q, r)\big]\Big)\\
\leq &  C  \exp(-c t^2)    ,
\end{split}
\end{align}
and so
$ 
\EP(R_{1}\ge  t \sigma )  \leq   C \exp(-c t^2)
$. 

Since $\hat r \le r$, $\langle \bsbE, \bsbA_2\rangle $ and $\langle \bsbE, \bsbA_3\rangle $ can be similarly bounded   using  Lemma \ref{emprocBnd}  and Lemma \ref{emprocBnd2}. To sum up,
\begin{align}
\EE \langle \bsbE,   \bsbX\bsbDelta \rangle \le\EE\left\{\Big(\frac{2}{a}+\frac{2}{a'}\Big)\frac{\|\bsbX\bsbDelta\|_F^2}{2} + 4 b L \sigma^2 P_o(q, r) + C a'\sigma^2 \right\}\label{stochtermbound-exp}
\end{align}
for any    $a, b, a'>0$ and $4b > a$.

 Substituting $\bsbB^*$ for $\bsbB$ in \eqref{stochtermbound-exp} and plugging it back to \eqref{fixedpoint-auxeq1},   we obtain
\begin{align*}
& \EE [ 2 \bar \breg_{l_0}(\bsbX \bsbB^*, \bsbX \hat \bsbB)] \\  \le \ &
\EE\Big\{\big(\frac{2}{a}+\frac{2}{a'}\big)\Breg_{2} (\bsbX \hat \bsbB, \bsbX \bsbB^*)      + \rho \Breg_2(\bsbB^*, \hat\bsbB)\Big\}  \\ & + 4bL \sigma^2[(r(\bsbX)\wedge  q   + m)   r +  (p - q ) \log q] + C a'\sigma^2.
\end{align*}
From the regularity condition,
$$
\rho \Breg_2(\bsbB^*, \hat\bsbB)   \le (2  \bar \breg_{l_0}-\delta \Breg_2) (\bsbX \bsbB^*, \bsbX \hat \bsbB)  + K \sigma^2 \{(r(\bsbX)\wedge  q   + m) r +
(p - q  ) \log q\},
$$
and so \begin{align*}
\EE\Big\{\big(\delta- \frac{2}{a}-\frac{2}{a'}\big)\Breg_{2} (\bsbX \hat \bsbB, \bsbX \bsbB^*)    \Big\} \le
(4bL +K)\sigma^2\big[(r(\bsbX)\wedge  q   + m)   r   +  (p - q) \log q\big]  + Ca' \sigma^2      .
\end{align*}
Choosing  $a = a' =8/\delta, b = 3/\delta$ gives the conclusion.\\

To prove  Theorem \ref{th_genlosserr}, we first notice that the  optimality of $\hat \bsbB$ and the definition of the effective  noise  imply  $$
\breg_{l_0} (\bsbX \hat \bsbB, \bsbX \bsbB^*) \le \langle \bsbE, \bsbX (\hat  \bsbB - \bsbB^*)\rangle,
$$
in place of \eqref{fixedpoint-auxeq1}.
Based on  the   stochastic term decomposition and the  bounds  in the lemmas,   we have
for any   $a, a', b>0$ and $4b > a$ \begin{align*}
&\EE \big\{ \breg_{l_0} (\bsbX \hat \bsbB, \bsbX \bsbB^*) - (\frac{2}{a}+\frac{2}{a'})\Breg_{2} (\bsbX \hat \bsbB, \bsbX \bsbB^*)     \big \}\\
\le \ &   4bL \sigma^2\{(r(\bsbX)\wedge  q    + m)   r    +(p - q  )   \log q    \}+Ca'\sigma^2 .
\end{align*}
The conclusion can be shown similarly.

 Next, we prove Corollary \ref{cor:estbnd}. By the optimality of $\hat \bsbB$,   we have
$l_0 (  \bsbX  \hat \bsbB;  \bsbY)  \le l_0(  \bsbX \bsbB;   \bsbY)$ for any $\bsbB: r(\bsbB) \leq r$, $q( \bsbB)\le q$,  from which it follows that
$$
\Breg_{l_0}(   \bsbX \hat \bsbB,       \bsbX \bsbB^*)  \le \Breg_{l_0}(\bsbX  \bsbB,   \bsbX    \bsbB^* ) + \langle \bsbE,   \bsbX\bsbDelta    \rangle  ,$$ where   $\bsbDelta =\hat \bsbB -  \bsbB$.
 From  \eqref{stochtermbound-exp} and the Cauchy-Schwartz inequality,
\begin{align*}
\EE \Breg_{l_0}(\bsbX \hat \bsbB, \bsbX   \bsbB^*) \le \ & \EE \big[\Breg_{l_0}(\bsbX   \bsbB, \bsbX   \bsbB^*) +  ( {2}/{a}+ {2}/{a'})(1+1/b')\Breg_2(\bsbX\hat \bsbB,  \bsbX \bsbB^*) \\ &  + ( {2}/{a}+ {2}/{a'})(1+b')\Breg_2( \bsbX  \bsbB,   \bsbX \bsbB^*)\big]+ 4b L \sigma^2 P_o(q, r) + C a'\sigma^2,\end{align*}
for any $a, a',  b'>0$ and  $4  b>a$. Choosing    $a = a'=16/\mu$, $b=5/\mu$, and $b'=1$,  and using the strongly convexity of $l_0$, we obtain the oracle inequality.

To prove the second result in the corollary, we  set  $\bsbB = \bsbB^*$    and  use a high-probability form of \eqref{stochtermbound-exp}: with probability at least $1 - C \exp(-c (r(\bsbX)+m))$,       
\begin{equation} \label{intermediatestep}
\begin{aligned}
 \langle \bsbE,   \bsbX\bsbDelta \rangle \le  \Big(\frac{2}{a}+\frac{2}{a'}\Big)\frac{\|\bsbX\bsbDelta\|_F^2}{2} + 4 b L \sigma^2 P_o(q, r), \end{aligned}
\end{equation}
where as before, $L$ is  a constant   and   $a, a', b$ are any positive numbers satisfying $  4b> a  $. We use \eqref{eqR1probtailbound} to verify the probability bound under $2\le q \le p$:  if $r = 0$,    $R_1= 0$, otherwise
$$
 \{q \wedge r(\bsbX)\}r + (p-q)\log q \gtrsim   r(\bsbX)  \wedge \{q + (p-q)\log q\}\gtrsim  r(\bsbX) \wedge p =  r(\bsbX).
 $$
 (To see the last inequality, consider two cases $q\le cp$ and $q\ge cp$.) Hence  $(q \wedge r(\bsbX)+m)r + (p-q)\log q \gtrsim m + r(\bsbX)$.

Now, using $\log q \le  \vartheta + \log q^* \le \vartheta \log q^*$, \eqref{intermediatestep} implies
\begin{align*}
\Big(\mu -\frac{2}{a}-\frac{2}{a'}\Big) \|   \bsbX \hat \bsbB    -  \bsbX \bsbB^*  \|_F^2\le        8b  \vartheta^2 L \sigma^2 \{(q^{*}\wedge r(\bsbX) +m  ) r^{*} + (p  -q^*)\log q^{*} \}
\end{align*}
The remaining details are  omitted.\\

\noindent \emph{Proofs of Lemma \ref{emprocBnd} and Lemma \ref{emprocBnd2}.}
We prove Lemma \ref{emprocBnd2} as follows. The proof of Lemma \ref{emprocBnd}  is similar.  
By definition, for any fixed $\bsbA$, $\langle \bsbE,  \bsbA\rangle$ is a   sub-Gaussian random variable with scale bounded by $\sigma\|\bsbA\|_F$. Therefore,   $\{\langle \bsbE,  \bsbA\rangle: \bsbA \in \Gamma_{ r, q_1, q_2}\}$ defines a stochastic process with sub-Gaussian increments.  The induced metric on $\Gamma_{ r,q_{1}, q_{2}}$ is Euclidean: $d(\bsbA_1, \bsbA_2) = \sigma \|\bsbA_1 - \bsbA_2\|_F$.

We will use Dudley's entropy integral to bound $\sup_{\bsbA \in \Gamma_{r,q_1, q_2}} \langle \bsbE, \bsbA \rangle$. Toward this, we need to calculate   the metric entropy $\log {\mathcal N}(\varepsilon, \Gamma_{ r, q_{1}, q_{2}}, d)$, where ${\mathcal N}(\varepsilon, \Gamma_{ r, q_{1}, q_{2}} , d)$ is the minimum cardinality of all $\varepsilon$-nets   covering $ \Gamma_{ r, q_{1}, q_{2}}$  under  $d$.

 Given $\bsbA\in  \Gamma_{ r, q_{1}, q_{2}}$,   its column space must be contained in  $cs( \Proj_{\bsbX_{\mathcal Q_1}^{\perp}} \Proj_{\bsbX_{ {\mathcal Q}_2}}   ) $ for some $\mathcal Q_i$  with $q(\mathcal Q_i) =q_i$.  When varying $\mathcal Q_i$,  the number of all $cs( \Proj_{\bsbX_{\mathcal Q_1}^{\perp}} \Proj_{\bsbX_{ {\mathcal Q}_2}}   ) $   is no more than $ {p\brace q_1} \times {p\brace q_2}$.  Let $\mathcal U$ be a set of $ \bsbU_i$  with $\bsbU_i \in \mathbb R^{n\times (r(\bsbX)\wedge q_2 )}$, $\bsbU_i^T \bsbU_i = \bsbI$, and $\bsbU_i \bsbU_i^T$ characterizing all these column spaces, $ 1\le i\le {p\brace q_1} \times {p\brace q_2}  $. Then we can represent  $\bsbA $  by
\begin{align}
\bsbA  =\bsbU \bsbSig \bsbV^{T},\label{rcdecomp}
\end{align}
where  $ \bsbU\in \mathcal U$, $\bsbV\in \mathbb O^{m\times r}$, and $\bsbSig \in \mathbb R^{(r(\bsbX)\wedge q_2 )\times r}$.

 An $\varepsilon$-net can be designed noticing that $\bsbV \bsbV^T$ corresponds to
a point on a Grassmann manifold  $G_{m, r}$ (all $r$-dimensional subspaces of $\mathbb R^m$), and   $\bsbSig=\bsbU^{T} \bsbA \bsbV$ is in a unit ball of dimensionality ${(r(\bsbX)\wedge q_2 )\times r}$, denoted by $B_{(r(\bsbX)\wedge q_2 )\times r}$. In fact, given any $\bsbA_1\in \Gamma_{r, q_1, q_2}$, we can write  $\bsbA_1 = \bsbU_1 \bsbSig_1 \bsbV_1^{T}$ according to \eqref{rcdecomp} and  find $\bsbV_2$ and $\bsbSig_2$ such that $\|\bsbV_1\bsbV_1^{T} - \bsbV_2\bsbV_2^{T}\|_2 \leq \varepsilon$ and   $\|\bsbSig_1 \bsbV_1^{T} \bsbV_2 - \bsbSig_2\|_F\leq \varepsilon$.  Then,  for $\bsbA_2=\bsbU_1 \bsbSig_2 \bsbV_2^{T}$,
\begin{align*}
\|\bsbA_1 -\bsbA_2\|_F &\leq \|\bsbA_1 - \bsbA_1 \bsbV_2 \bsbV_2^{T}\|_F + \| \bsbU_1 \bsbSig_1 \bsbV_1^{T} \bsbV_2 \bsbV_2^{T} - \bsbU_1 \bsbSig_2 \bsbV_2^{T}\|_F \\&\leq \big[\mathrm{tr}\{ \bsbA_1^{T} \bsbA_1 (\Proj_{\bsbV_1}- \Proj_{\bsbV_2})^2\}\big]^{1/2} + \|  \bsbSig_1 \bsbV_1^{T} \bsbV_2 - \bsbSig_2 \|_F\\
&\leq  \big(\|\bsbA_1\|_F^2 \|\Proj_{\bsbV_1}- \Proj_{\bsbV_2}\|_2^2\big)^{1/2}+\varepsilon \leq 2\varepsilon.
\end{align*}
By a standard volume argument, $${\mathcal N}(\varepsilon, B_{(r(\bsbX)\wedge q_2 )\times r}, \| \cdot \|_F) \leq (C_0/\varepsilon)^{(r(\bsbX)\wedge q_2 )\times r},$$ and    from \cite{szarek82},   $${\mathcal N}(\varepsilon, G_{m,r}, \|\cdot \|_2) \leq   \left({C_1}/{\varepsilon}\right)^{r(m-r)},
$$
where   $C_0, C_1$ is are universal constants. Therefore,
under  metric $d$, \begin{align}
\log  {\mathcal N}(\varepsilon, \Gamma_{ r, q_{1}, q_{2}} , d)  \leq  \log {p\brace q_1} +   \log {p\brace q_2}  + \{(r(\bsbX)\wedge q_2 ) r +r m \}\log  ({C \sigma} /{\varepsilon} ).
\end{align}

From Dudley's integral bound, we obtain
\begin{align*}
\EP \left\{\sup_{\bsbA \in \Gamma_{ r, q_1, q_2}} \langle \bsbE, \bsbA \rangle \geq t \sigma + L \int_0^{\sigma}  {\log^{1/2}  {\mathcal N}(\varepsilon, \Gamma_{J,r}, d)}\rd \varepsilon\right\} \leq C\exp(- ct^2).
\end{align*}
By   computation,
\begin{align*}
\int_0^{\sigma}  \log^{1/2} {\mathcal N}(\varepsilon, \Gamma_{ r, q_1, q_2}, d) \rd \varepsilon  & \lesssim \sigma \log^{1/2} {p\brace q_1} +   \sigma \log^{1/2} {p\brace q_2}  \\ & \quad + \sigma \{(r\wedge q_{2})\times r +r(m-r)\}^{1/2} \\ &\lesssim  \{P ( r, q_1, q_2)\}^{1/2},\end{align*} and  so the probability bound in Lemma \ref{emprocBnd2} follows.  \\

\noindent \emph{Proof of Lemma \ref{stirling2bound}.}
The Stirling number of the second kind can be shown to be  $  {p\brace q} = \frac{1}{q!}\sum_{j=0}^q (-1)^{q - j} {q \choose j} j^p$, using   the inclusion-exclusion
formula.

When $q = 1, 2,  p -1$, and $ p$, ${p\brace q}= 1, 2^{p-1}-1, {p \choose 2}$, and  $1$, respectively. So the conclusion is straightforward for these special cases. By \cite{Rennie_1969},  for $p\ge 2$ and $1\le q\le p-1$,
\begin{align}
\frac{1}{2} (q^2 + q + 2) q^{p - q -1} -1 \le   {p\brace q} \le\frac{1}{2} {p \choose q} q^{p-q}.  \label{lubndS2}
\end{align}
The lower bound in \eqref{st2ndbound} can be shown from \eqref{lubndS2} directly. To prove the upper bound, we consider two cases.  (i) $q \le p/2$, $q\ge 3$. Using Stirling's formula and  the monotone property of $t/\log t$ when $t\ge 3$, we have
$$
{p \choose q} \le (ep/q)^q \le p^q \le q^p \le q^{2(p-q)}.
$$
(ii)  $q\ge p/2$, $q\le p-2$, $p\ge 4$. Again, by Stirling's formula, $${p \choose q} =  {p \choose p-q}\le (\frac{ep}{p-q })^{p-q}.$$ Therefore,
\begin{align*}
\log  {p\brace q} & \le (p-q)\log \frac{ep}{p-q} + (p-q)\log q \\
&\le (p-q)\log \frac{ep}{2} + (p-q)\log q \\ & \le 3(p-q)\log \frac{p}{2} + (p-q)\log q \\ & \le4(p-q) \log q .
\end{align*}
The proof is complete.

\begin{remark} \label{rem:svbounds}Some similar results can be shown if we replace $\bsbB$ by   $\bar \bsbB = \left[\begin{array}{c} \bsbS \\ \bsbV \end{array}\right]$. Let  $\bar \bsbX = \mbox{diag}\{\bsbX, \bsbI\} = \left [\begin{array}{cc}\bsbX & \\ & \bsbI \end{array} \right ]$, and denote $l_0(\bsbX \bsbB )$ by $ \bar l_0 ( \bar \bsbX \bar \bsbB)$.   Given a differentiable $l_0$, define a surrogate by a separate linearization with respect to $\bsbS, \bsbV$:
\begin{align}
\bar G_{\bsb{\rho}}(\bar\bsbB; \bar\bsbB^{-})=  \bar l_0(\bar\bsbX \bar\bsbB; \bsbY)  - \breg_{\bar l_0} (\bar\bsbX\bar\bsbB, \bar\bsbX\bar\bsbB^{-})+  {\rho}_S \Breg_{2}  ( \bsbS,  \bsbS^{-}) +   {\rho}_V \Breg_{2}  ( \bsbV,  \bsbV^{-}) , \label{surrodefTH-SV}
\end{align}
where $\bsb{\rho}=(\rho_S, \rho_V)$ with $\rho_S, \rho_V\ge 0$. For simplicity, we write
$$
\Breg_{2, \bsb{\rho}}(\bar\bsbB; \bar\bsbB^{-}) = {\rho}_S \Breg_{2}  ( \bsbS,  \bsbS^{-}) +   {\rho}_V \Breg_{2}  ( \bsbV,  \bsbV^{-}).
$$
The first part of Theorem \ref{th_local-SV}  studies the accuracy of  all    ``\emph{fixed points}'' satisfying\begin{align}
\hat {\bar \bsbB }\in \argmin_{\bar \bsbB:  q(\bsbS)\le q, \bsbV^T\bsbV= \bsbI } \bar G_\rho(\bar\bsbB; \bar\bsbB^-)|_{\bar\bsbB^- = \hat {\bar\bsbB}} \label{BiterDefTH-SV}
\end{align}in terms of a discrepancy measure between $\bar \bsbB_1$ and $\bar \bsbB_2$ ($\bar \bsbB_i =  [ \bsbS_i^T,   \bsbV_i^T]^T\in \mathbb R^{(p+m)\times r}$),
\begin{align}
d(\bar \bsbX\bar \bsbB_1, \bar \bsbX\bar\bsbB_2) = \frac{1}{2}\|\bsbX (\bsbS_1 - \bsbS_2) \bsbV_2^T + \bsbX \bsbS_2 (\bsbV_1 - \bsbV_2)^T\|_F^2.\end{align}
$d(\bar \bsbX\bar \bsbB_1, \bar \bsbX\bar\bsbB_2)$  is similar to but different from  $\Breg_2  (  \bsbX   {  \bsbB}_1 ,   \bsbX   \bsbB_2  ) =  \|\bsbX (\bsbS_1 - \bsbS_2) \bsbV_1^T + \bsbX \bsbS_2 (\bsbV_1 - \bsbV_2)^T\|_F^2 /2$, but both are bounded above by $2d'(\bar \bsbX\bar \bsbB_1, \bar \bsbX\bar\bsbB_2)$ with
\begin{align}
d'(\bar \bsbX\bar \bsbB_1, \bar \bsbX\bar\bsbB_2) & 
= \frac{1}{2} \|\bsbX (\bsbS_1 - \bsbS_2 )\|_F^2  + \frac{1}{2}  \| \bsbX \bsbS_2 (\bsbV_1 - \bsbV_2)^T\|_F^2.
\end{align} The second part of Theorem \ref{th_local-SV}  analyzes  the ``\emph{alternatively optimal}'' solutions, typically arising from block coordinate descent, and reveals how a high-quality starting point can relax the regularity condition (consider an extreme case $M=1$). 
\begin{theorem}\label{th_local-SV}
        Assume the same effective noise defined in \eqref{noise-def} is sub-Gaussian with scale bounded above by $\sigma$.         Define
$$
         P_o (q, r) =  \{q \wedge r(\bsbX)  +m
\} r +(p -q )\log q. $$
 (i) Let $\hat{\bar \bsbB}$ be any fixed point satisfying \eqref{BiterDefTH-SV} with  $r= r^*$, $q= q^*$.
  Assume   $\bsb{\rho}$ is chosen so that
the following condition holds for  some   $\delta>0$ and  sufficiently large   $K\ge0$ \begin{align*}
          \Breg_{2,\bsb{\rho}}(\bar \bsbB_1, \bar  \bsbB_2)   \le  (2  \bar \breg_{\bar l_0} - \delta d) (\bar \bsbX \bar \bsbB_1, \bar \bsbX  \bar \bsbB_2)   + K \sigma^2  P_o(q,r), \\ \forall \bar \bsbB_i\in \mathbb R^{(p+m)\times r}: q(\bsbS_i) \le q,  \bsbV_i^T \bsbV_i =\bsbI.
        \end{align*}
         Then
        \begin{align}
        \EE   d(\bar \bsbX\hat {\bar \bsbB},\bar \bsbX \bar\bsbB^*)     \lesssim \frac{ K\delta\vee 1}{\delta^2} \big\{ \sigma^2(q\wedge r(\bsbX) + m)r  +   \sigma^2(p- q) \log q + \sigma^2\big\}.
        \label{genlosserrrate-fixedpoint0-SV}
        \end{align}
        (ii) Let $l_0(\bsbX \bsbB ; \bsbY ) = \| \bsbX \bsbB-\bsbY\|_F^2/2$  with  $r= r^*$, $q= q^*$.
   Let $(\hat \bsbS, \hat \bsbV)$ be an alternatively optimal solution starting from a feasible initial point  $(\bsbS^{[0]}, \bsbV^{[0]})$, which satisfies $$\hat \bsbV\in \arg\min_{\bsbV:\bsbV^T \bsbV=\bsbI} l_0(\bsbX\hat \bsbS \bsbV), \ \hat \bsbS\in \arg\min_{\bsbS:q(\bsbS)\le q} l_0( \bsbX \bsbS \hat \bsbV),$$
and
$$
l_0(\bsbX \hat \bsbS \hat \bsbV^T)\le l_0(\bsbX   \bsbS^{[0]}   \bsbV^{[0]T}).
$$
Without loss of generality, assume  for some $M$: $  +\infty\ge M \ge 1$,   $$\EE [\Breg_2 ( \bsbX    {  \bsbB}^{[0]},  \bsbX  \bsbB^*  ) ]= O(M) \allowbreak\{\sigma^2P_o (q, r) +\sigma^2 \}.$$    If there exist   some   $\delta>0$,   $K\ge0$ and  constant $C>0$ such that  for all $   \bar \bsbB_i\in \mathbb R^{(p+m)\times r}: q(\bsbS_i) \le q,  \bsbV_i^T \bsbV_i =\bsbI$,
\begin{align*}
&\big(1-\frac{1}{M}\big) (2\bar\breg_{\bar l_0}   -  d' )(\bar \bsbX  \bar \bsbB_1 , \bar \bsbX  \bar \bsbB_2)
  +\frac{C}{  M( M\delta\vee1)}    \Breg_2 (\bsbX \bsbB_1, \bsbX \bsbB_2)+
 K \sigma^2  P_o(q,r) \nonumber\\ & \ge \delta \cdot \max\Big\{ \big(1-\frac{1}{M}\big)      d(\bar \bsbX  \bar \bsbB_1 , \bar \bsbX  \bar \bsbB_2), \frac{1}{M}\Breg_2 (\bsbX \bsbB_1, \bsbX \bsbB_2)  \Big\},
\end{align*}
then
         the same conclusion \eqref{genlosserrrate-fixedpoint0-SV} holds.
\end{theorem}

Finally, we point out  that the techniques developed in this section can be used to analyze  nonglobal estimators of selective reduced rank regression, as raised in  \cite{She2017selfact}.
\begin{proof}
(i)  By definition,
\begin{align*}
\bar G_{\bsb{\rho}}(\bar\bsbB; \bar\bsbB^{-}) & =  \bar l_0(\bar\bsbX \bar\bsbB; \bsbY)    - \breg_{\bar l_0}(\bar\bsbX\bar\bsbB, \bar\bsbX\bar\bsbB^{-})+  \Breg_{2, \bsb{\rho}}(\bar\bsbB, \bar\bsbB^{-}) \\
&=  \bar l_0(\bar\bsbX \bar\bsbB^-; \bsbY)    + \langle  \nabla  {\bar l_0}(\bar\bsbX\bar\bsbB^-),\bar\bsbX \bar \bsbB -\bar\bsbX  \bar\bsbB^{-}\rangle+   \Breg_{2, \bsb{\rho}}(\bar\bsbB, \bar\bsbB^{-}) \\
&=  \bar l_0(\bar\bsbX \bar\bsbB^-; \bsbY)    + \langle  \nabla_{\bar \bsbB}  {\bar l_0}(\bar\bsbX\bar\bsbB^-),  \bar \bsbB -  \bar\bsbB^{-}\rangle+   \Breg_{2,\bsb{\rho}}(\bar\bsbB, \bar\bsbB^{-}).
\end{align*}
By matrix differentiation,
\begin{align*}
&\nabla_{\bsbS} \bar l_0(\bar \bsbX \bar \bsbB) = \nabla_{\bsbS} l_0(\bsbX \bsbS \bsbV^T) =\bsbX^T \nabla l_0(\bsbX \bsbB) \bsbV\\
&\nabla_{\bsbV} \bar l_0(\bar \bsbX \bar \bsbB) = \nabla_{\bsbV} l_0(\bsbX \bsbS \bsbV^T) = (\nabla l_0(\bsbX \bsbB))^T \bsbX \bsbS.
\end{align*}
Hence
\begin{align*}
 \langle  \nabla_{\bar \bsbB}  {\bar l_0}(\bar\bsbX\bar\bsbB^-),  \bar \bsbB -  \bar\bsbB^{-}\rangle  = \ &     \langle   \nabla {  l_0}( \bsbX \bsbB^-)  ,  \bsbX (  \bsbS - \bsbS^-)(\bsbV^-)^T  + \bsbX\bsbS^-   ( \bsbV - \bsbV^-)^T\rangle. \end{align*}
Similar to the derivation of \eqref{fixedpoint-auxeq1}, we get
\begin{align*}
2 \bar \breg_{\bar l_0}(\bar\bsbX \bar\bsbB^*, \bar \bsbX \hat {\bar\bsbB}) \le \langle \bsbE, \bsbX ( \hat  \bsbS - \bsbS^*) \bsbV^{* T}  + \bsbX\bsbS^*   ( \hat \bsbV - \bsbV^*)^T\rangle+   \Breg_{2, \bsb{\rho}}(\bar \bsbB^*, \hat{\bar\bsbB}). \end{align*}
The following orthogonal decomposition can be used to simplify the stochastic term
\begin{align*}
& \bsbX (  \hat \bsbS  -   \bsbS^* )   \bsbV ^{*T}  + \bsbX   \bsbS^*    ( \hat \bsbV  -   \bsbV^* )^T  \\ = \ & \mathcal P_{\bsbX_{\mathcal Q(\bsbB^*)}}\{\bsbX (  \hat \bsbS  -   \bsbS^* )   \bsbV ^{*T}  + \bsbX   \bsbS^*    ( \hat \bsbV  -   \bsbV^* )^T\}\mathcal P_{ \bsbV^*} \\ & + \mathcal P_{\bsbX_{\mathcal Q(\bsbB^*)}}  \mathcal \bsbX   \bsbS^*   \hat \bsbV ^T \mathcal P_{\bsbV^*}^\perp+ \mathcal P_{\bsbX_{\mathcal Q(\bsbB^*)}}^\perp\mathcal P_{\bsbX_{\mathcal Q(\hat \bsbB )}}\bsbX \hat \bsbS   \bsbV ^{*T}  .
\end{align*}
 Applying the same scaling argument and the empirical process theory as before gives the desired conclusion. The details are omitted.

 (ii) The following lemma recharacterizes the alternatively optimal $(\hat \bsbS,\hat \bsbV)$ as a fixed-point solution to a \textit{joint} optimization problem.
The proof is omitted.\begin{lemma}\label{lem:AtoF}
Let $l_0(\bsbX \bsbB ; \bsbY ) = \| \bsbX \bsbB-\bsbY\|_F^2/2$, and  $\hat \bsbV\in \arg\min_{\bsbV:\bsbV^T \bsbV=\bsbI} l_0(\bsbX\hat \bsbS \bsbV)$,   $\hat \bsbS\in \arg\min_{\bsbS:q(\bsbS)\le q} l_0( \bsbX \bsbS \hat \bsbV)$. Construct \begin{align*}
 G (\bsbS,\bsbV; \bsbS^{-},\bsbV^{-})= \ &  \bar l_0(\bar\bsbX \bar\bsbB; \bsbY)  - \breg_{\bar l_0} (\bar\bsbX\bar\bsbB, \bar\bsbX\bar\bsbB^{-})\\ &  + \Breg_2(\bsbX \bsbS (\bsbV^{-})^T, \bsbX \bsbS^- (\bsbV^{-})^T) +\Breg_2(\bsbX \bsbS^-\bsbV^T, \bsbX \bsbS^-(\bsbV^-)^T)\\ =  \ &  \bar l_0(\bar\bsbX \bar\bsbB; \bsbY)  - \breg_{\bar l_0} (\bar\bsbX\bar\bsbB, \bar\bsbX\bar\bsbB^{-}) +d'(\bar\bsbX\bar\bsbB, \bar\bsbX\bar\bsbB^{-}).
 \end{align*} Then  $(\hat \bsbS, \hat \bsbV) \in \argmin_{(\bsbS, \bsbV)} G (\bsbS,\bsbV; \bsbS^{-},\bsbV^{-})  |_{\bsbS^- = \hat \bsbS, \bsbV^-=\hat \bsbV}.$
\end{lemma}
Let  $E={\sigma^2P_o (q, r) +\sigma^2 }$. Then
\begin{equation}
\EE [\Breg_2 (  \bsbX    {  \bsbB}^{[0]},   \bsbX   \bsbB^*  ) ]\le   C M  \allowbreak E\label{auxcond1}
\end{equation} where   $ 1\le  M \le  +\infty$, and   $C\ge 0$ is a constant.  Assume the following regularity condition
\begin{align}
&\big(1-\frac{1}{M}\big) (2\bar\breg_{\bar l_0}   -  d' )(\bar \bsbX  \bar \bsbB_1 , \bar \bsbX  \bar \bsbB_2)
  +\frac{C_0}{  M( M\delta_{0} \vee c_{0})}    \Breg_2 (\bsbX \bsbB_1, \bsbX \bsbB_2)+
 K \sigma^2  P_o(q,r) \nonumber\\ & \ge \max\Big\{ (1-\frac{1}{M})      \delta_0 d(\bar \bsbX  \bar \bsbB_1 , \bar \bsbX  \bar \bsbB_2), \frac{1}{M}\delta_{0}\Breg_2 (\bsbX \bsbB_1, \bsbX \bsbB_2)  \Big\}\label{auxcond2}
\end{align}
for all $ \bar \bsbB_i\in \mathbb R^{(p+m)\times r}: q(\bsbS_i) \le q,  \bsbV_i^T \bsbV_i =\bsbI$ for some $\delta_{0}>0$, $K\ge 0$, and   constants $C_0, c_0>0$.

By Lemma \ref{lem:AtoF}, we can repeat the analysis in (i) to obtain for any   $\delta>0$,
\begin{equation}
  \EE [(2 \delta\bar\breg_{\bar l_0}  - \delta^2 d -   {\delta  }     d')(  \bar\bsbX \hat {\bar \bsbB},  \bar\bsbX  \bar\bsbB^*)   ] \le  C E. \label{auxeq1}
\end{equation}
Next, since $  l_0(   \bsbX \hat {  \bsbB} )\le   l_0(   \bsbX   \bsbB^{[0]})$,  we have
\begin{align*}
\Breg_2 (   \bsbX \hat {  \bsbB},   \bsbX   \bsbB^*) \le  \Breg_2 (     \bsbX   \bsbB^{(0)},  \bsbX   \bsbB^*) + \langle \bsbE,\bsbX   (\hat  \bsbB - \bsbB^*)     \rangle
  -  \langle \bsbE, \bsbX (    \bsbB^{[0]} - \bsbB^*)  \rangle,
\end{align*}
from which it follows that for any $\delta'>0$,
\begin{equation}
\EE [   (1 -  \delta')\Breg_2   (  \bsbX \hat {  \bsbB} ,   \bsbX   \bsbB^*  ) ]  \le \big ( 1 + \frac{1}{\sqrt M} \big ) \Breg_2(   \bsbX   \bsbB^{[0]},   \bsbX   \bsbB^*) + C \big (\frac{1}{\delta'} +\sqrt M\big)E.
\end{equation}
Together with \eqref{auxcond1}  and   $M\ge 1$, we obtain
\begin{equation}
\EE [   ( 1  - \delta'   ) \Breg_2 (  \bsbX \hat {  \bsbB} ,  \bsbX  \bsbB^*  ) ]  \le C   \big(\frac{1}{\delta'} +  M\big)E \le \frac{C}{c_1 \wedge c_2 } \big(\frac{c_{1}}{\delta'} + c_{2} M\big) E  \label{auxeq2raw}
\end{equation}
where    $c_1, c_2>0$ are arbitrary constants and recall that     $C$ may not be the same constant at each occurrence.
Taking $\delta^2 = \delta'^2 /(c_1 + c_2 M \delta')$, we get
\begin{equation}
\EE \Big[  \big (\frac{\delta^{2}}{\delta'}   - \delta^2 \big )\Breg_2 (  \bsbX \hat {  \bsbB} ,   \bsbX   \bsbB^*  ) \Big]  \le    \frac{C}{c_1 \wedge c_2 }E.    \label{auxeq2}
\end{equation}
Multiplying \eqref{auxeq1} by $(1 - 1/M)$ and   \eqref{auxeq2} by $1/M$ and adding the two inequalities yield
\begin{align*}
 & \EE \Big[ \big (1-\frac{1}{M}\big)(2 \bar{\bm\Delta}_{\bar l_0}  - \delta d - d')  (\bar \bsbX \hat {\bar \bsbB} , \bar \bsbX \bar \bsbB^*  )     +\frac{1}{M}\big( \frac{\delta}{\delta'} - \delta   \big )\Breg_2( \bsbX \hat { \bsbB} , \bsbX  \bsbB^*  )\Big ]   \le     \frac{E}{\delta}\big (1+ \frac{C}{c_1 \wedge c_2}\big).
 \end{align*}
 Simple calculation shows $$ \delta'/   \delta  = \frac{1}{2} \Big(  c_2 M \delta + \sqrt{c_2^2 M^2 \delta^2 + 4c_1}\Big ) \le  (2c_2 M \delta) \vee (4c_1). $$
 Under    \eqref{auxcond2} with $\delta_0 = 4 \delta, C_0 = 2/c_2, c_0 = 8c_1/c_2$, we get a stronger conclusion $$\EE [ d  (\bar \bsbX \hat {\bar \bsbB} , \bar \bsbX \bar \bsbB^*  ) \vee  \Breg_2  (  \bsbX \hat {  \bsbB} ,   \bsbX   \bsbB^*  ) ] \lesssim  \{ K\delta_0 \vee 1)/\delta_0^2\} E.$$ A reparameterization of     \eqref{auxcond2}  gives the  regularity condition in the theorem. \end{proof}
\end{remark}

\subsection{Predictive information criterion for model comparison}
\label{subsec:proofpic}
Since $P_o(\bsbB)$ depends on $r(\bsbB)$ and $q(\bsbB)$ only, we also write it as $P_o( q(\bsbB),r(\bsbB))$.
Theorems  \ref{th:pic},  \ref{cor:sf-pic}
require  some finer analysis of   the stochastic term than Section \ref{subsec:prooforacle}.

For example, $R_1$ in \eqref{residualterm1} is now redefined as
\begin{align*}
R_1 &:=
\sup_{  q } \sup _r   \sup_{\bsbA\in \Gamma_{ r, q}} \big[\langle \bsbE, \bsbA \rangle - 2(b/a)^{1/2} \{{L\sigma^2 P_o(  q, r)\}^{1/2}}\big]_+,
\end{align*}
but we claim that             $\EE R_1^2\le C \sigma^2$ still holds.

First, notice that when  $r=0$ or $r(\bsbX) = 0$, $$\sup_{\bsbA\in \Gamma_{ r, q}} \big[\langle \bsbE, \bsbA \rangle - 2(b/a)^{1/2}\allowbreak \{{L\sigma^2 P_o(  q, r)\}^{1/2}}\big]_+ = 0,$$ and when $q = 1 $,    $$\EP\Big(\sup_{\bsbA\in \Gamma_{ r, q}} [\langle \bsbE, \bsbA \rangle - 2(b/a)^{1/2} \{{L\sigma^2 P_o(  q, r)\}^{1/2}}]_+ \ge  t \sigma\Big) \le C\exp(-c t^2 )\exp(- c m).$$
Assume $r(\bsbX) \ge 1$ and $q\ge 2$ (and so $r\ge 1$). From Lemma \ref{emprocBnd} and Lemma \ref{stirling2bound},
\begingroup
\allowdisplaybreaks
\begin{align}
&   \EP \Big[
\sup_{2\leq q \leq p, 1\leq r \leq q}  \sup_{\bsbA\in \Gamma_{ r, q}}   \langle \bsbE, \bsbA \rangle -\{L\sigma^2  P_o ( q, r)\}^{1/2} \ge t \sigma\notag\\
& \quad  +   2(b/a)^{1/2} \sigma\{{L P_o( q, r)\}^{1/2}}-\sigma\{{L  P_o (q, r)}\}^{1/2} \Big] \notag\\
\leq &   C\exp(-c t^2) \sum_{q=2}^p \sum_{r=1}^q \exp\Big(- c [\{2(b/a)^{1/2}-1\}^2 L  P_o(q, r)]\Big)\notag\\
\leq &  C  \exp(-c t^2) \sum_{q=2}^p \exp\big[-c\{r(\bsbX)\wedge  q  + m + (p-q)\log q\}\big]  \notag \\
\le & C \exp(-c t^2) \exp(-cm)  \Big( \sum_{q=2}^{r(\bsbX)}  \exp\big[-c \{ q + (p-q)\log q\}\big] \notag \\ &  + \sum_{q=r(\bsbX) + 1}^{p}  \exp\big[-c \{ r(\bsbX) + (p-q)\log q)\}\big] \Big) \notag \\
\le & C \exp(-c t^2) \exp(-cm)  \Big\{ r(\bsbX)  \exp(-c  p )\notag  \\ &   + \exp( - c r(\bsbX))   \sum_{q=r(\bsbX) + 1}^{p}  \exp\big[-c(p-q)\log (r(\bsbX) + 1)\big] \Big\} \notag \\
\le & C \exp(-c t^2) \exp(-cm)  \{   \exp(-c  p )+    C\exp( - c r(\bsbX))    \}\ \notag \\
\le &   C \exp(-c t^2) \exp \{-c  ( m + r(\bsbX))\}\ \le     C \exp(-c t^2), \notag
\end{align}
\endgroup
where the positive constants  $C, c$ are not necessarily the same at each occurrence.

From the definition of $\hat \bsbB$, $\breg_{l_0}(\bsbX\hat \bsbB,\bsbX\bsbB^*)  + A\sigma^2P_{o}(\hat \bsbB)\le A\sigma^2P_o(\bsbB^*)   +   \langle \bsbE, \bsbX \hat \bsbB -       \bsbX \bsbB^* \rangle $,
and so for any  $a, b,  a'>0$,  $4  b>a$,
\begin{align*}
&\EE \{\breg_{l_0}(\bsbX\hat \bsbB,\bsbX\bsbB^*) + A\sigma^2P_{o}(\hat \bsbB)\}\\ \le & \EE\Big\{ A\sigma^2P_o(\bsbB^*)    +(\frac{2}{a}+\frac{2}{a'})\Breg_{2}(\bsbX\hat \bsbB,\bsbX\bsbB^*)+c a' \sigma^2 + b L \sigma^2 \big[P_o( \bsbB^* )+ P_o( \hat \bsbB )  \\& + (r(\bsbX)\wedge  q(\hat \bsbB)  + m)   r(\hat \bsbB) + (p - q(\bsbB^* )) \log q(\bsbB^*) + (p - q(\hat \bsbB)) \log q(\hat \bsbB)\big]  \Big \}\\
\le & \EE\Big\{A\sigma^2P_o(\bsbB^*)  +(\frac{2}{a}+\frac{2}{a'})\Breg_{2}(\bsbX\hat \bsbB,\bsbX\bsbB^*) +c a' \sigma^2+ 4b L \sigma^2 [P_o(\bsbB^*)+ P_o( \hat \bsbB)]\Big\},
\end{align*}
using the orthogonal decomposition and Lemmas   \ref{emprocBnd}--\ref{stirling2bound}.


Combining it with the regularity condition gives
\begin{align*}
& \EE \big\{(\delta -\frac{2}{a}-\frac{2}{a'} )\Breg_{2}(\bsbX\hat \bsbB,\bsbX\bsbB^*) + (A - 4bL -C) P_o(   \hat  \bsbB ) \big\} \\ \le\  &    (A+4bL+C) \sigma^2 P_o(\bsbB^*) + c a'  \sigma^{2}. \end{align*}
Since   $\EE P_o(\bsbB^*) \ge c > 0$,
choosing the constants satisfying $({1}/{a}+{1}/{a'})(1+{1}/{b'})<{\delta}/{2}$,  $4  b>a$, and $A > 4bL+C$  yields the   conclusion in Theorem \ref{th:pic}.
\\

Next, we   prove  Theorem \ref{cor:sf-pic}.
Let $\mathcal M = \{\nabla b(\bsbe): \bsbe\in \Omega\}$ which is an open set as well. We use      $\overline {\mathcal M}$   to denote its  closure. Since $b$ is strongly convex, $ b^*(\cdot) = \sup_{\bsbe} \langle \bsbe, \cdot \rangle - b(\bsbe)$ is  differentiable on $\mathcal M$  and $\nabla b$ is a one-to-one mapping from $\Omega$ onto $\mathcal M$ \citep{Rockafellar1970}.

Let $h(\bsbB; A) = 1/\{  mn/\kappa - A P_o (\bsbB)\}$.
From   the  optimality of   $\hat \bsbB$, we have $\{l_0(\bsbX \hat \bsbB; \bsbY)  + b^*(\bsbY)\} h(\hat \bsbB; A)   \le \{l_0(\bsbX   \bsbB^*; \bsbY)+ b^*(\bsbY)\}   h(  \bsbB^*; A)$ or
\begin{align*}
 & l_0(\bsbX \hat \bsbB; \bsbY)+ b^*(\bsbY) - \{l_0(\bsbX   \bsbB^*; \bsbY)+ b^*(\bsbY)\}   \\ \le  \ &  \{l_0(\bsbX   \bsbB^*; \bsbY)+ b^*(\bsbY) \} \Big( \frac{ h(  \bsbB^*; A)}{ h(\hat \bsbB; A)}-1\Big),
\end{align*}
because $ h(\hat \bsbB; A)>0$.
Using the fact that $\Breg_{l_0} = \Breg_{b}$ and the definition of the effective noise, we get
\begin{align}
   \Breg_b  (\bsbX \hat \bsbB, \bsbX \bsbB^* )     \le \{l_0(\bsbX   \bsbB^*; \bsbY)+ b^*(\bsbY) \} \Big( \frac{ h(  \bsbB^*; A)}{ h(\hat \bsbB; A)}-1\Big) + \langle \bsbE, \bsbX \hat \bsbB - \bsbX \bsbB^*\rangle. \label{sfpictempeq2}
\end{align}

We claim that
\begin{align}
 \frac{1}{2\mu'}\|\bsbE\|_F^2  \le l_0(\bsbX   \bsbB^*; \bsbY)+ b^*(\bsbY)  \le   \frac{1}{2\mu}\|\bsbE\|_F^2,\ \forall \bsbY \in \overline{\mathcal M}. \label{eqsfpicEsq}
\end{align}
Let $\varphi = b^*$. Given   $  \bsbe\in \Omega$, let $\bsbe^* (\bsb{\eta}) = \nabla b(\bsb{\eta})$ or $\bsbe^*$ for brevity. Because of the assumptions on $b$, $(\bsbe, \bsbe^*)$ makes a so-called conjugate pair, and we have $\bsbe = \nabla \varphi(\bsbe^*)$ and  $\langle \bsbe, \bsbe^*\rangle = b(\bsbe) + \varphi(\bsbe^*)$ \citep{Rockafellar1970}.  For any $\bsb{z} \in \mathcal M$, $\bsbe\in \Omega$,
\begin{align*}
- \langle \bsb{z}, \bsb{\eta}\rangle  + b(\bsb{\eta}) + b^{*}(\bsb{z}) & =- \langle \bsb{z}, \bsb{\eta}\rangle +  \langle \bsbe^*, \bsb{\eta}\rangle - \varphi(\bsbe^*) + \varphi(\bsb{z})  \\
& = - \langle \bsb{z}- \bsbe^*, \bsb{\eta}\rangle   - \varphi(\bsbe^*) + \varphi(\bsb{z})\\
& = - \langle \nabla \varphi (\bsbe^*),   \bsb{z}- \bsbe^*)- \varphi(\bsbe^*) + \varphi(\bsb{z}) = \breg_{\varphi}(\bsb{z}, \bsbe^*).
\end{align*}
Next,  consider the conjugate of $w(\bsbdelta)=\breg_\varphi(\bsbmu + \bsbdelta, \bsbmu)$, or $(\breg_\varphi(\bsbmu + \cdot, \bsbmu))^*$.  Given  $\bsbmu\in \mathcal M$, let $(\bsbmu, \bsbmu^*)$ be a conjugate pair with $\bsbmu^* = \nabla \varphi (\bsbmu)$. Since $w$ is a proper function and $b^{**}=b$,
  \begin{align*}
  w^*(\bsbtau) & = \sup_{\bsbdelta} \langle \bsbtau + \nabla\varphi(\bsbmu), \bsbdelta\rangle - \varphi(\bsbmu+ \bsbdelta) + \varphi(\bsbmu)\\
 &= \sup_{\bsbdelta} \langle \bsbtau + \bsbmu^{*},\bsbmu+ \bsbdelta\rangle - \varphi(\bsbmu+ \bsbdelta) + \varphi(\bsbmu)-\langle \bsbtau + \bsbmu^*, \bsbmu\rangle\\
 &= b(   \bsbmu^{*} + \bsbtau )  + \varphi(\bsbmu)-\langle \bsbmu^*, \bsbmu\rangle-\langle\bsbtau , \bsbmu\rangle \\
& = b(\bsbmu^{*} + \bsbtau )  -b  ( \bsbmu^*)-\langle \bsbtau , \bsbmu\rangle =  \Breg_{b}(\bsbmu^*+\bsbtau, \bsbmu^*).
\end{align*}
Therefore, given  $  \bsbe\in \Omega$,  we get
$$(\breg_{\varphi}(\bsbe^* + \cdot, \bsbe^*))^*(\bsbtau) = \Breg_{b}(\bsbtau +  \bsbe^{*  *},  \bsbe^{*  *})=\Breg_{b}(\bsbtau +  \bsbe ,  \bsbe )\ge \frac{\mu}{2} \| \bsbtau\|_F^2 $$
and taking the conjugate gives  $$ \breg_{\varphi}(\bsb{z}, \bsbe^* )  \le \frac{1}{2\mu}\| \bsb{z}- \bsbe^* \|_F^2.$$
Letting $\bsbe = \bsbX \bsbB^*$ and $\bsb{z}= \bsbY$,  we obtain
\begin{align}
l_0(\bsbX   \bsbB^*; \bsbY)+ b^*(\bsbY) \le   \frac{1}{2\mu}\| \bsbY -\nabla b(\bsbX \bsbB^*)\|_F^2,\  \forall \bsbY  \in \mathcal M. \label{eqsfpicEsq-1}
\end{align}
For       $\bsbY\in\overline {\mathcal M}\setminus \mathcal M$, we can use the lower semi-continuity of the conjugate function $b^*$    to get the same bound. Similarly, we can prove the lower bound in \eqref{eqsfpicEsq}.

With \eqref{eqsfpicEsq} available, \eqref{sfpictempeq2} becomes
\begin{align}
&  \mu \Breg_2  (\bsbX \hat \bsbB, \bsbX \bsbB^* )  \nonumber \\  \le \ &  (l_0(\bsbX   \bsbB^*; \bsbY)+ b^*(\bsbY)  ) \frac{ A P_o(\bsbB^*)- A P_o(\hat \bsbB)} {   mn/\kappa - A P_o(\bsbB^*)}  + \langle \bsbE, \bsbX \hat \bsbB - \bsbX \bsbB^*\rangle \nonumber\\
= \ &  \frac{1}{2\mu}  \frac{A  \|\bsbE\|_F^2}{   mn\sigma^2/\kappa - A \sigma^2 P_o(\bsbB^*)}\sigma^2 P_o(\bsbB^*) -\frac{1}{2\mu'} \frac{A  \|\bsbE\|_F^2}{  mn/\kappa - AP_o(\bsbB^*)}\sigma^2P_o(\hat\bsbB)  + \langle \bsbE, \bsbX \hat \bsbB - \bsbX \bsbB^*\rangle\nonumber\\
\le \ &  \frac{1}{2\mu}  \frac{A  \|\bsbE\|_F^2}{  (1- A/A_{0}  )  mn\sigma^2/\kappa }\sigma^2 P_o(\bsbB^*) -\frac{1}{2\mu } \frac{A  \|\bsbE\|_F^2}{   mn \sigma^2}\sigma^2P_o(\hat\bsbB) + \langle \bsbE, \bsbX \hat \bsbB - \bsbX \bsbB^*\rangle. \label{eqsfpictemp1}
\end{align}

The stochastic term $ \langle \bsbE, \bsbX \hat \bsbB - \bsbX \bsbB^*\rangle$ can be bounded similarly as in   the proof of Theorem \ref{th:pic};
  we use a  high-probability form   here. For example, based on Lemma \ref{emprocBnd},    for any  $a_{1}, b_{1}, a_{2}>0$ satisfying $4b_{1} >a_{1}$, the following event
\begin{align*}
&  \langle \bsbE, \bsbA_1 \rangle \leq \mu ({1}/{a_{1}} +{1}/{a_{2}}) \|\bsbA_1 \|_F^2 +  (b_{1}/\mu) L\sigma^2   P_o(r, q) + (b_{1}/\mu) L_{0}\sigma^2P_o(r^* , q^*))
\end{align*}
can be shown to occur with   probability at least $1-C \exp(-c  m)  \exp(- P_o (  r^*, q^*))$ for a sufficiently large value of $L$. The overall bound is
$$
\langle \bsbE, \bsbX  \hat \bsbB - \bsbX \bsbB^* \rangle
\leq2 \mu ({1}/{a_{1}}+{1}/{a_{2}})\Breg_2  (\bsbX \hat \bsbB, \bsbX \bsbB^* )   + (b_{1}/\mu) L_{1} \sigma^2\{P_o(\hat\bsbB) + P_o(\bsbB^*)\},
$$
with probability at least $1-C  \exp\{-c  (P_o(\bsbB^*) +m) \}$ for some $c, C, L_{1}>0$.
Notice that $P_o(\bsbB^*) \gtrsim m + r(\bsbX) $ when $q^*\ge 2$.
Plugging the bound into \eqref{eqsfpictemp1} gives
\begin{align*}
&\mu\big(  1  -\frac{2}{a_{1}}-\frac{2}{a_2}\big)\Breg_2  (\bsbX \hat \bsbB, \bsbX \bsbB^* ) \\  \le   \; & \frac{1}{2\mu}\Big\{ \frac{\kappa A  \|\bsbE\|_F^2}{ (1 - A/A_{0}  ) mn\sigma^2    }+ 2 b_{1} L_1 \Big\}\sigma^2 P_o(\bsbB^*)   -\frac{1}{2\mu} \Big\{\frac{A  \|\bsbE\|_F^2}{   mn \sigma^2} -  {2b_{1} L_1  }  \Big\}\sigma^2P_o(\hat\bsbB).
\end{align*}

Since $e_{i,k}$ are independent and  non-degenerate,  $c_{1} mn\sigma^2 \le \EE \| \bsbE\|_F^2 \le c_2mn \sigma^2$ for some  constants $c_{1}, c_2>0$. Let $\gamma$  be  some constant  satisfying $0< \gamma  <  1$.  On $\mathcal E =\{c_{1}(1-\gamma ) {mn\sigma^2}\leq \|\bsbE\|_F^2 \leq c_2(1+\gamma ) {mn\sigma^2} \}$,  we have
\begin{align*}
& \frac{ A  \|\bsbE\|_F^2}{(1- A/A_{0})mn\sigma^2 }  \le    \frac{c_2(1+\gamma )A_{0} A }{A_0 - A} \  \mbox{ and }  \ \frac{A  \|\bsbE\|_F^2}{mn \sigma^2} \ge  {c_{1}(1-\gamma) A }.
\end{align*}
Regarding the probability of the event, we write  $\| \bsbE\|_F^2 = \vect(\bsbE) \bsbA \vect(\bsbE)^T$ with $\bsbA = \bsbI\in \mathbb R^{nm\times nm}$ and bound it with a generalized Hanson-Wright inequality  \citep[Theorem 1.1]{sambale2020some}. In fact, from $\mbox{Tr}(\bsbA) =mn, \|\bsbA\|_2 =1,\| \bsbA\|_F=\sqrt{mn}$, the complement of $\mathcal E$ occurs with probability at most $C' \exp\{-c'(m n)^{ \alpha/2}\}$.

Now, with $A_0, A, a_{1}, a_{2}, b_{1}$   large enough such that       $({1}/{a_{1}}+{1}/{a_{2}})<{1}/{2}$,  $4  b_{1}>a_{1}$,   $A > 2b_{1} L_{1} /\{c_{1}(1-\gamma) \}\ $ and $A_0 >  A$,  the conclusion follows.


\begin{remark}\label{rmk:pic} Our nonasymptotic analysis involves the use of a lot of union bounds,  and so may not yield the optimal numerical constants. However, these absolute constants    can be  determined by Monte Carlo experiments. For regression, we recommend $$\frac{\| \bsbY - \bsbX \bsbB\|_F^2 }{mn - A_1\{\|\bsbB\|_{2, \mathcal C}\wedge r(\bsbX) +m  \} r(\bsbB) -A_2 \{p -\|\bsbB\|_{2, \mathcal C}\}\log \|\bsbB\|_{2, \mathcal C}}$$ with $A_1 = 3$ and  $A_2 = 2.5$ based on empirical experience.

  Using the techniques in the proof of Theorem 3 of \cite{SCV},  one can change the  fractional form of PIC to some other scale-free forms,  and the conclusion remains the same (but the constants may change). For example,   for    the model segmentation problem \eqref{tracemodel} with an $\ell_2$ loss and $r=q$,  we can use the following \emph{log-form} of PIC
$$
n \log \Big\{\sum (y_i -  \langle \bsbX_i, \bsbB \rangle)^2\Big\} +A_{1} p\| \bsbB\|_{2, \mathcal C} +A_2 (n -\|\bsbB\|_{2, \mathcal C})\log \|\bsbB\|_{2, \mathcal C}
$$
with $A_1 = 1.5$ and $A_2=1.1$. Finally,  in applying the scale-free PICs,  those over-complex  models with $\delta(\bsbB)\ge 1$ should be eliminated beforehand.\end{remark}

%

\subsection{Canonical correlation analysis and whitening}
\label{app:ccatoweighted}
        \begin{lemma}\label{lem:ccatoweighted} Let $\bsbX\in \mathbb R^{n\times p}$,  $\bsbY \in \mathbb R^{n\times m}$ and $r\le r(\bsbY)$.      Then
\begin{equation}
        \label{eq:projected_CRL}
        \min_{(\bsbS,\bsbA ) \, \in \, \mathbb{R}^{p\times r} \times \mathbb{R}^{m\times r}} \frac{1}{2}\| \bsbY \bsbA  -\bsbX\bsbS \|_F^2 \mbox{   s.t. }   (\bsbY\bsbA)^T \bsbY \bsbA=n\bsbI, \| \bsbS\|_{2, \mathcal C}\le q,
\end{equation}
is equivalent to
\begin{equation}
        \label{eq:mpweighted_CRL}
        \min_{\bsbB \, \in \, \mathbb{R}^{p\times m} } \frac{1}{2} \mbox{Tr} \{(  \bsbY -\bsbX\bsbB) (\bsbY^T\bsbY/n)^{+}  ( \bsbY -\bsbX\bsbB)^T\} \mbox{   s.t. }   r(\bsbB)\le r, \| \bsbB\|_{2, \mathcal C}\le q.
\end{equation}
The same conclusion holds when the $(2, \mathcal C)$-constraint is replaced by a  $(2, 0)$-constraint.
        \end{lemma}

The lemma does not require  $\bsbY$ to have full column rank. When the row-wise  constraint  is inactive (e.g., $q = p$),  canonical correlation analysis converts to reduced rank regression with the Moore-Penrose inverse of $\bsbY^T \bsbY$ as the weighting matrix.
\begin{proof}
Let     $\bsbSig_{ Y}= \bsbY^T\bsbY/n$ and suppose its  spectral decomposition is given by  $\bsbU \bsbD \bsbU^T$, where  the diagonal matrix $\bsbD$ is  of size     $r(\bsbY)\times r(\bsbY)$ and is nonsingular.

First, given any feasible pair $(\bsbA, \bsbS)$ satisfying  $(\bsbY\bsbA)^T \bsbY \bsbA=n\bsbI$ and $ \| \bsbS\|_{2, \mathcal C}\le q$, we can construct  $\bsbW = \bsbD^{1/2}\bsbU^T \bsbA$, and $\bsbB = \bsbS \bsbW^T \bsbD^{1/2} \bsbU^T$ such that $\bsbW^T \bsbW  = \bsbI$,  $r(\bsbB)\le r$ and $\| \bsbB\|_{2, \mathcal C}\le \| \bsbS\|_{2, \mathcal C}\le q$. We claim that
\begin{align}
     \|\bsbY \bsbA - \bsbX \bsbS\|_F^2 = \mbox{Tr} \{(  \bsbY -\bsbX\bsbB) \bsbSig_Y^{+}  ( \bsbY -\bsbX\bsbB)^T\} +n( r-   r(\bsbY)).\label{ccalossconn}
\end{align}
 In fact,    $\bsbY \bsbA =\bsbY \bsbU \bsbU^T \bsbA= \bsbY \bsbU \bsbD^{-1/2} \bsbW$ gives
$\langle \bsbY \bsbA, \bsbX \bsbS\rangle  =\langle \bsbY \bsbU \bsbD^{-1/2}, \bsbX \bsbS\bsbW^T\rangle,$
from which it follows that
\begin{align*}
\|\bsbY \bsbA - \bsbX \bsbS\|_F^2  & =\|  \bsbY \bsbU \bsbD^{-1/2}- \bsbX \bsbS\bsbW^T\|_F^2+ \| \bsbY \bsbA\|_F^2- \|\bsbY \bsbU \bsbD^{-1/2} \|_F^2 \\
& =\|  \bsbY \bsbU \bsbD^{-1/2}\bsbU^T- \bsbX \bsbS\bsbW^T\bsbU^T\|_F^2+n(r- r(\bsbY)) \\ & =\|(  \bsbY  - \bsbX \bsbS\bsbW^T\bsbD^{1/2} \bsbU^T)\bsbU\bsbD^{-1/2}\bsbU^T\|_F^2+n(r- r(\bsbY))\\ & =\|(  \bsbY  - \bsbX \bsbB)(\bsbSig^{+})^{1/2}\|_F^2+n(r- r(\bsbY)).
\end{align*}

Conversely, given a feasible $\bsbB: r(\bsbB)\le r, \| \bsbB\|_{2, \mathcal C}\le q$, we can write $\bsbB \bsbU \bsbD^{-1/2}$ as $\bsbS \bsbW^T$ with $\bsbW\in \mathbb R^{r(\bsbY)\times r}$: $\bsbW^T \bsbW = \bsbI$, and thus  $\|\bsbS\|_{2, \mathcal C}\le \|\bsbB\bsbU \bsbD^{-1/2} \bsbW\|_{2, \mathcal C}\le q$. Let $\bsbA = \bsbU \bsbD^{-1/2} \bsbW$. Then $(\bsbY \bsbA)^T \bsbY \bsbA=n \bsbW ^T \bsbW =n \bsbI $. It is easy to verify that \eqref{ccalossconn} still holds.
The proof applies to a $(2, 0)$-constraint as well.
\end{proof}

\subsection{Pairwise-difference penalization}
\label{app:pairwise}
Recall an alternative to enforce row-wise equisparsity in  $\bsbB = [  \bsbb_1, \ldots,  \bsbb_p]^T$ as mentioned in   Remark  \ref{rem:pairwisepen}:
\begin{align}
\sum_{1\le j< j'\le p} P(\|  \bsbb_j - \bsbb_{j'}\|_2; \lambda). \label{eq:pairwisediffpen}
\end{align}
 \eqref{eq:pairwisediffpen} involves $\mathcal O(p^2)$ many terms;  its optimization   typically requires the techniques of \textit{operator splitting} and results in high computational complexity. Although we will not report  detailed  analysis  in this paper, even using an   ideal $\ell_0$ penalty in \eqref{eq:pairwisediffpen} will result in   a much worse statistical error rate         than   CRL,  and will      discourage size-balanced clustering as well.

More specifically, let's  penalize the pairwise row-differences of $\bsbB$ via an $\ell_0$ function (arguably an  ideal choice of $P$ from a theoretical perspective):
\begin{align}\label{lowrankpairwisel0}
\min l_0(\bsbX \bsbB; \bsbY) + \frac{\lambda^2}{2} \sum_{1\le j<j'\le p} 1_{\|  \bsbb_j - \bsbb_{j'}\|_2\ne 0}.
\end{align}
Given the number of groups $q(\bsbB)$, if we use $\bsb{g}(\bsbB) = \{g_1, \ldots, g_{q(\bsbB)} \}$ to denote the group sizes, the penalty can be written as $\lambda^2/2$ times
\begin{align}
\mathcal C(\bsbB) = \frac{1}{2}\sum_{i=1}^{q(\bsbB)} g_i(p - g_i).
\end{align}
To get an error bound of the resultant estimator, the regularization parameter   $\lambda$ must be large enough to suppress the noise. When $l_0$ is $\mu$-strongly convex,  our analysis below shows that  $\lambda$ should be as large as $$  \sigma \sqrt{(p\vee m)/ \mu p}$$ up to a multiplicative constant.

\begin{theorem}\label{th:pairwisel0}
 Let   $\lambda_o=\sigma\sqrt {(p\vee m)/p}$   and assume there exist some $\delta>0$, $K\ge 0$  such that $\bm\Delta_{l_0}(\bsbX\bsbB_1,\bsbX\bsbB_2) +K\sigma^2\lambda_o^2  \mathcal C( \bsbB_1) + K\sigma^2\lambda_o^2 \mathcal C(\bsbB_2)\ge \delta\Breg_2(\bsbX\bsbB_1,\bsbX\bsbB_2)$ for all $\bsbB_1, \bsbB_2$. Let $\lambda = A  \sqrt{ K \vee (1/  \delta)} \lambda_o$ with $A$ a sufficiently large constant and  $\hat\bsbB $ be an optimal  solution to \eqref{lowrankpairwisel0}. Then
                \begin{align}
                \EE [\| \bsbX \hat \bsbB - \bsbX   \bsbB^*\|_F^2]  \lesssim \ & \frac{(K\delta\vee 1)\sigma^2}{\delta^2} \Big\{     m    + \frac{m\vee p}{p}\mathcal C(\bsbB^*) \Big\} \label{pairwisel0err1}\\ \le \ & \frac{(K\delta\vee 1)\sigma^2}{\delta^2}  \Big\{     m + (m\vee p)\big(p - \frac{p}{q^*} \big)\Big \}. \label{pairwisel0err2}\end{align}
\end{theorem}
\begin{proof}
Since $\mathcal C(\bsbB)$ depends on $\bsbB$ through $\bsb{g}(\bsbB)$ only,   we define   $\mathcal C(\bsb{g}) = \sum_{i=1}^{q} g_i(p - g_i)$ for any  $\bsb{g} =\{g_1, \ldots, g_q\}$ with a slight abuse of notation.
\begin{lemma}\label{lem:boundsofpairwisepen}
Given $q\in [p]$, let $ \mathcal G_q= \{ \bsb{g}:   |\bsb{g}|=q,  \sum g_i = p, g_i \in \mathbb N  \}$. Then
 \begin{align*}
 &\frac{(q-1) (2p-q)}{2} = \mathcal C(\{1, \ldots,1,p-q+1\}) \le  \min_{\bsb{g}\in \mathcal G_q}\mathcal C(\bsb{g}),  \\  &  \max_{\bsb{g}\in \mathcal G_q}\mathcal C(\bsb{g})\le  \mathcal C\Big(\big\{\frac{p}{ q}, \ldots, \frac{p}{q}\big\}\Big) = \frac{p^2(q-1)}{2q}.\end{align*}
\end{lemma}
The first inequality can be proved by induction and the second inequality follows from the Cauchy-Schwartz inequality. The proof details are omitted.

Using the definition of the generalized Bregman function \eqref{genbregdef} and the bound of  the stochastic term developed  in the last subsection, we obtain for any   $a, a', b>0$ and $4b > a$
\begin{align*}
&\EE \big\{ \breg_{l_0} (\bsbX \hat \bsbB, \bsbX \bsbB^*) - (\frac{2}{a}+\frac{2}{a'})\Breg_{2} (\bsbX \hat \bsbB, \bsbX \bsbB^*)  +\frac{\lambda^2}{2}\mathcal C(\hat \bsbB)    \big \}\\
\le \ &  \EE \big\{4bL \sigma^2P_o(\hat \bsbB )+ 4bL \sigma^2P_o(\bsbB^*) +Ca'\sigma^2 + \frac{\lambda^2}{2} \mathcal C( \bsbB^*)\big\}.
\end{align*}
We claim that when $\lambda_1^2 = A \lambda_o^2 =A \sigma^2(p\vee m)/p$ with a large constant $A$,  $$ {\lambda_1^2} \mathcal C(  \bsbB) + 2 m \sigma^2 \ge  \sigma^2P_o(  \bsbB ), \forall \bsbB.$$ When $q(\bsbB) = 1$, $P_o(\bsbB )\le2 m$.  Otherwise, by Lemma \ref{lem:boundsofpairwisepen}, we need to study the size of  $$
\max_{\bsbB: q(\bsbB)\ge 2} \frac{(r(\bsbX)\wedge  q(\bsbB)    + m )r(\bsbB) + (p-q(\bsbB))\log q(\bsbB)}{ (q(\bsbB)-1) (2p-q(\bsbB))  }.
$$
Because $r(\bsbB)\le m\wedge q(\bsbB)$, it is easy to verify that the maximal value is of the order $(m+p)/p$. The remaining derivations follow the lines of the proof of Theorem \ref{th:pic}.
\end{proof}

 From Lemma \ref{lem:boundsofpairwisepen}, the CRL rate in Theorem \ref{th_local} beats  \eqref{pairwisel0err1} and \eqref{pairwisel0err2} all the time, and  in light of the minimax studies (e.g.,  Theorem \ref{th:minimax-complete}), the performance of  the type of pairwise regularization is  merely suboptimal.

One potential remedy   is to replace the uniform  $\lambda$\ by    $\lambda_{j,j'} $ or $w_{j,j'}\lambda $, with a set of      data-dependent weights $w_{j,j'}$.         But the weight construction     is notoriously difficult in   high-dimensional supervised learning,         and         to the best of our knowledge, there is no sound  scheme yet with finite-sample theoretical support. In addition, the penalty parameters $\lambda_{j,j'} $ often need much finer grids and are less intuitive than          $q$  which is  more convenient to specify practically. Therefore, sparsifying pairwise differences is not our favorable regularization.
\\

Finally, as suggested by one reviewer,   the  pairwise-difference based regularizations might appear similar to   the fused LASSO \citep{tibshirani2005sparsity}  at first glance, but there are some significant differences. Fused LASSO  considers   \textit{successive}  differences of the coefficients only to impose sparsity. Indeed, if one could  rearrange the features  using the authentic equisparse model,  so that the true coefficients are well sorted to be equal in consecutive blocks, then  the successive-difference based  regularization  would do the job as  those penalizing all pairwise differences; but even so,  with more than one response, the  preferred orderings according to different columns of the true coefficient matrix may be   incompatible.

\section{More  Experiments }
\label{app:exp}
\subsection{Implementation details}
\label{appsub:impldetails}

First, we make a discussion of the inverse stepsize $\rho$.
The algorithm derivation in Section \ref{subsec:algdesign} shows

\begin{align}   G(\bsbS^{[k]},\bsbV^{[k]}; \bsbS^{[k-1]},\bsbV^{[k-1]})-f(\bsbS^{[k]}, \bsbV^{[k]})\ge \frac{\rho}{2}  \|  \bsbB^{[k]}- \bsbB^{[k-1]} \|_F^2  - \frac{L}{2} \| \bsbX( \bsbB^{[k]}- \bsbB^{[k-1]})\|_F^2. \label{ineqformajor0}\end{align}
Indeed, \begin{align*}
& \ l(\bsbS^{[k]},\bsbV^{[k]}) -   l(\bsbS^{[k-1]},\bsbV^{[k-1]})-\langle \nabla l_0(\bsbX\bsbB^{[k-1]}), \bsbX(\bsbB^{[k]}-\bsbB^{[k-1]}) \rangle \\= & \int_{0}^{1}
\langle \nabla  {l}_0(\bsbX\bsbB^{[k-1]}+t\bsbX(\bsbB^{[k]} - \bsbB^{[k-1]})),
\bsbX\bsbB^{[k]} -\bsbX\bsbB^{[k-1]} \rangle \rd t
\\ &\quad - \int_{0}^{1}\langle \nabla  {l_{0}}(\bsbX\bsbB^{[k-1]}), \allowbreak\bsbX\bsbB^{[k]} - \bsbX\bsbB^{[k-1]}\rangle \rd t \\
=  & \int_{0}^{1} \langle \nabla {l_{0}}(\bsbX\bsbB^{[k-1]}+t\bsbX(\bsbB^{[k]}- \bsbB^{[k-1]}))
- \nabla  {l}_{0}(\bsbX\bsbB^{[k-1]}), \bsbX \bsbB^{[k]} - \bsbX\bsbB^{[k-1]} \rangle \rd t \\\leq & \int_{0}^{1} L t\rd t \, \|\bsbX\bsbB^{[k]}- \bsbX\bsbB^{[k-1]}\|_F^2  =\ \frac{L}{2} \| \bsbX( \bsbB^{[k]}- \bsbB^{[k-1]})\|_F^2.
\end{align*}
Hence  $\rho  = L\|\bsbX \|_2^2$ suffices to secure the numerical convergence. But a smaller   $\rho$, such as $\rho =L    \overline{\kappa}_2(q, r) $ seen from \eqref{ineqformajor0}, is favored by our statistical analysis in  Theorem \ref{th_local}. Therefore,  we strongly recommend performing a  line search  in implementation with          \eqref{surrogate_ineqs0}  as the  search criterion.  

RRR can be used to initialize the algorithm.  Specifically, let $\bsbB_{\mbox{\tiny rrr}} = \bsbB_{\mbox{\tiny ols}} \bsbV_r\bsbV_r^T$, where $\bsbB_{\mbox{\tiny ols}}= (\bsbX^T\bsbX)^{+}\bsbX^T\bsbY$   and $\bsbV_r$ is formed by the leading $r$ eigenvectors of $\bsbY^T\Proj_{\bsbX}\bsbY$.
Then one can set $\bsbV^{[0]}=\bsbV_r$ and perform K-means on $\bsbB_{\mbox{\tiny rrr}}\bsbV^{[0]}=\bsbB_{\mbox{\tiny ols}}\bsbV^{[0]}$ to obtain $\bsbF^0$ and $\bsbmu^0$. Other initialization schemes are possible; in particular, the multi-start  strategy of \cite{rousseeuw1999fast} is quite effective in some hard cases in our experience.

 Some  experiments for  unsupervised learning in later subsections  involve the application of kernel CRL on  similarity data (cf. Remark \ref{rmk:unsup}).  Given a symmetric normalized (or unnormalized) graph Laplacian $\bsbL = \bsbI - \bsbD^{-1/2}\bsbW \bsbD^{-1/2}$ (or $\bsbL=\bsbD -\bsbW$) with $\bsbW$ representing the symmetric data similarity matrix and  $\bsbD=\mbox{diag}\{\bsbW\bold{1}\}$   \citep{von2007tutorial}, it can be shown that setting   $\bsbK=\rho \bsbI -\bsbL$ for any  $\rho \geq \sigma_{\max}(\bsbL)$ in kernel CRL and dropping its equisparsity regularization    gives an equivalent characterization of    spectral clustering.  But the complete  CRL criterion   enforces equisparsity and low rank simultaneously, resulting in an iterative pursuit of the optimal subspace and clusters. Concretely,  when only  a positive semi-definite similarity matrix $\bsbK$ is given, we can perform spectral decomposition  $\bsbK=\bsbU\bsbD\bsbU^T$ with $\bsbD = \mbox{diag}\{ d_1, \cdots , d_n\}$ and $d_1\ge \cdots \ge d_n\ge 0$, and  set $\bsbY=\bsbU\bsbD^{{1}/{2}}$  to run the CRL algorithm. With a rank-$\bar m$ SVD truncation   and  whitening,   $\bsbY$ becomes $ \bsbU[:,1:\bar m]$. One can set $\bar m=m$ if $m$ is known, or simply $\bar m =\alpha q$ with say $\alpha=2$,   which shows good performance in general.

\subsection{Unsupervised data} 
\label{subsec:simu1}

This part performs synthetic data experiments for clustering.
The   simulation datasets are generated according to   two settings. 1) Pick    $10$ points   at random in a two-dimensional square with side length $500$ as cluster centers, then generate $100$ observations for each cluster  in $\mathbb R^{50}$  by adding   noise $N(0,\sigma^2)$ with $\sigma^2 = 1,\mbox{10000}$. This  resembles a common  scenario where the cluster centroids    lie in a lower  dimensional subspace ($r^*<q^*$). 2) Generate $20$ cluster centers by sampling in a hypercube with side length $500$ in $\mathbb R^{50}$, then add noise as in  first setting.

Apart from CRL, we tested    K-means,
K-means++  and AFK-MC$^2$ \citep{arthur2007k,bachem2016fast}. {As a matter of fact, we  tested many other methods,  including, say, \cite{chi2015splitting},  in the two settings, the computational and clustering performances of which are however  much worse than K-means++ or AFK-MC$^2$, and so their results are not included.}
To remove the influence of    tuning   and to make a fair comparison,    we   set $q=q^*$ in all experiments.  Other  parameters were taken to be their default values.
In each setup, we repeated the experiment   100 times and evaluated the performance of an algorithm
using  the  metrics of \emph{clustering accuracy} (CA) \citep{cai2005document} and   \emph{mean squared error} (MSE). Specifically, $\mbox{CA}= |\{ l_i= \mbox{map}(f_i) \}|/n$, where $n$ is the total number of data samples, $l_i$ and $f_i$ denote the true cluster and the assigned cluster of the $i$-th observation, respectively, and $\mbox{map}(f_i)$ is a permutation mapping   by the Kuhn-Munkres algorithm \citep{lovasz2009matching}.  In experience,   CA  is    more sensitive  than the Rand Index \citep{rand1971objective}.  MSE is given by $\|\bsbY^* - \hat \bsbB\|_F^2/mn$, to measure the error in   data approximation, where $\hat \bsbB$ is the estimated approximation matrix in $\mathbb R^{n \times m}$.
\begin{table}
    \caption{Performance comparison in terms of the clustering accuracy (CA) and mean square error (MSE) ($n=500, m=50$)\label{sim_uic}}
        \setlength\tabcolsep{1.5pt}
        \renewcommand{\arraystretch}{1}
        \centering
        \footnotesize

        \begin{tabular}{l c c c c c c c c c c c}
                \hline
                &  \multicolumn{2}{c}{\textbf{Setting 1}}& &  \multicolumn{2}{c}{\textbf{Setting 1}}&  &\multicolumn{2}{c}{\textbf{Setting 2}}&  &\multicolumn{2}{c}{\textbf{Setting 2}} \\
                &  \multicolumn{2}{c}{{\scriptsize$q^*=10,\sigma^{2}=1$}}& &\multicolumn{2}{c}{{\scriptsize$q^*=10,\sigma^{2}=$1e+4}} & &\multicolumn{2}{c}{{\scriptsize$q^*=20,\sigma^{2}=1$}}& &\multicolumn{2}{c}{{\scriptsize$q^*=20,\sigma^{2}=$1e+4}}\\
                &  \multicolumn{2}{c}{{\scriptsize$r^* = 2$}}& &  \multicolumn{2}{c}{{\scriptsize$r^* = 2$}}&  &\multicolumn{2}{c}{{\scriptsize$r^* = q^*$}}&  &\multicolumn{2}{c}{{\scriptsize$r^* = q^*$}} \\
                \cmidrule(lr){2-3} \cmidrule(lr){5-6} \cmidrule(lr){8-9}\cmidrule(lr){11-12}

                &\mbox{CA} &  \mbox{MSE}& &\mbox{CA} &  \mbox{MSE}& &\mbox{CA}  &\mbox{MSE}& &\mbox{CA}&  \mbox{MSE}\\
                \hline
                K-means & 0.72&   447&  &0.70 &626  &  &0.67 & 3.9e+3 & &0.81  & 2.5e+3 \\
                K-means++ & 0.94&  277 &  &0.92&  375&  &0.97 & 305 & &0.83& 2.4e+3 \\
                AFK-MC$^2$ &  0.88&  329&  &0.87&  491&  &0.82 & 1.4e+3 & &0.87 & 2.2e+3\\
                CRL ($r=q$) & 1&  0.01 &  &0.94 & 195 &  &1& 0.01 & &0.98 & 323 \\
                CRL ($r=0.5q$) & 1&  0.01 &  &0.96 & 143& & 1& 4.1e+3 & &0.98 & 4.6e+3 \\
                \hline
        \end{tabular}

        %

\end{table}

Table \ref{sim_uic} shows the results 
averaged over $100$ independent realizations, where
CRL gives excellent    CA and  MSE rates.   In the presence of a large number of clusters and/or large noise, CRL's improvement over K-means, K-means++ and AFK-MC$^2$ is particularly impressive.
Moreover, according to the  last row, it is possible  to enforce a lower rank on these datasets.  The finding  is    quite surprising in Setting $2)$, where the original cluster centroids do \textit{not} lie in a lower-dimensional subspace. Indeed, setting $r=0.5q$ resulted in poor  data approximation, but still succeeded  in discovering the clustering structure. (This is appealing in computation since many modern clustering  algorithms   have excellent performance in low dimensions.) In Setting $1)$, the simultaneous dimension reduction even brought some improvement when the noise is large,  owing to its power of removing some nuisance dimensions.
\\

We also tested kernel-based CRL on twelve two-dimensional benchmark datasets that have nonconvex clusters \citep{jain2005data,chang2008robust}.   Figure \ref{kernel_res} plots the data  and CRL clusters. Although these are considered to be challenging    tasks in the literature, CRL handled all of them with ease.

\begin{figure}[htbp]
        \begin{minipage}[t]{0.325\linewidth}
                \centering
                \includegraphics[width=\textwidth]{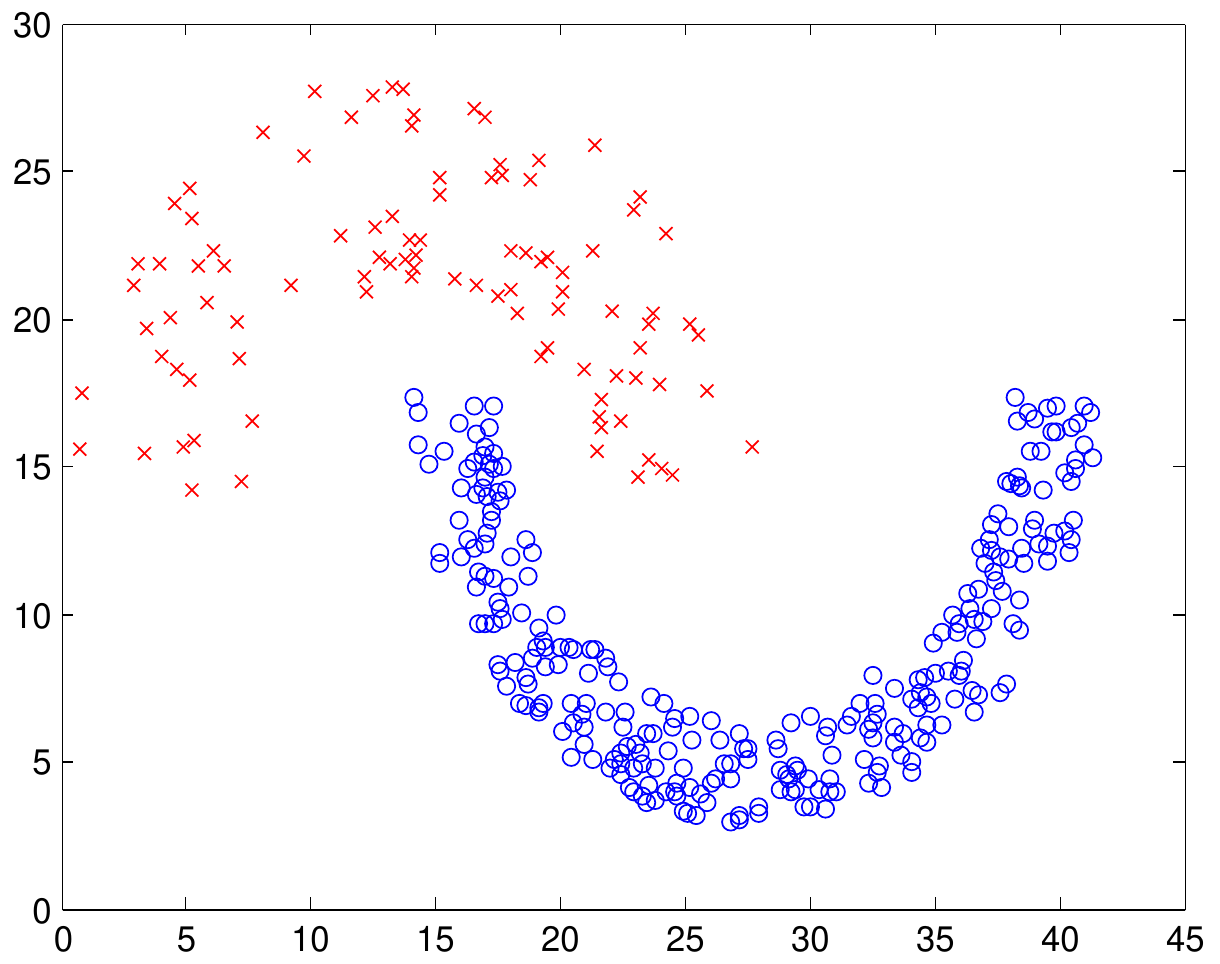}\\
                a. Double moon
        \end{minipage}
        \begin{minipage}[t]{0.325\linewidth}
                \centering
                \includegraphics[width=\textwidth]{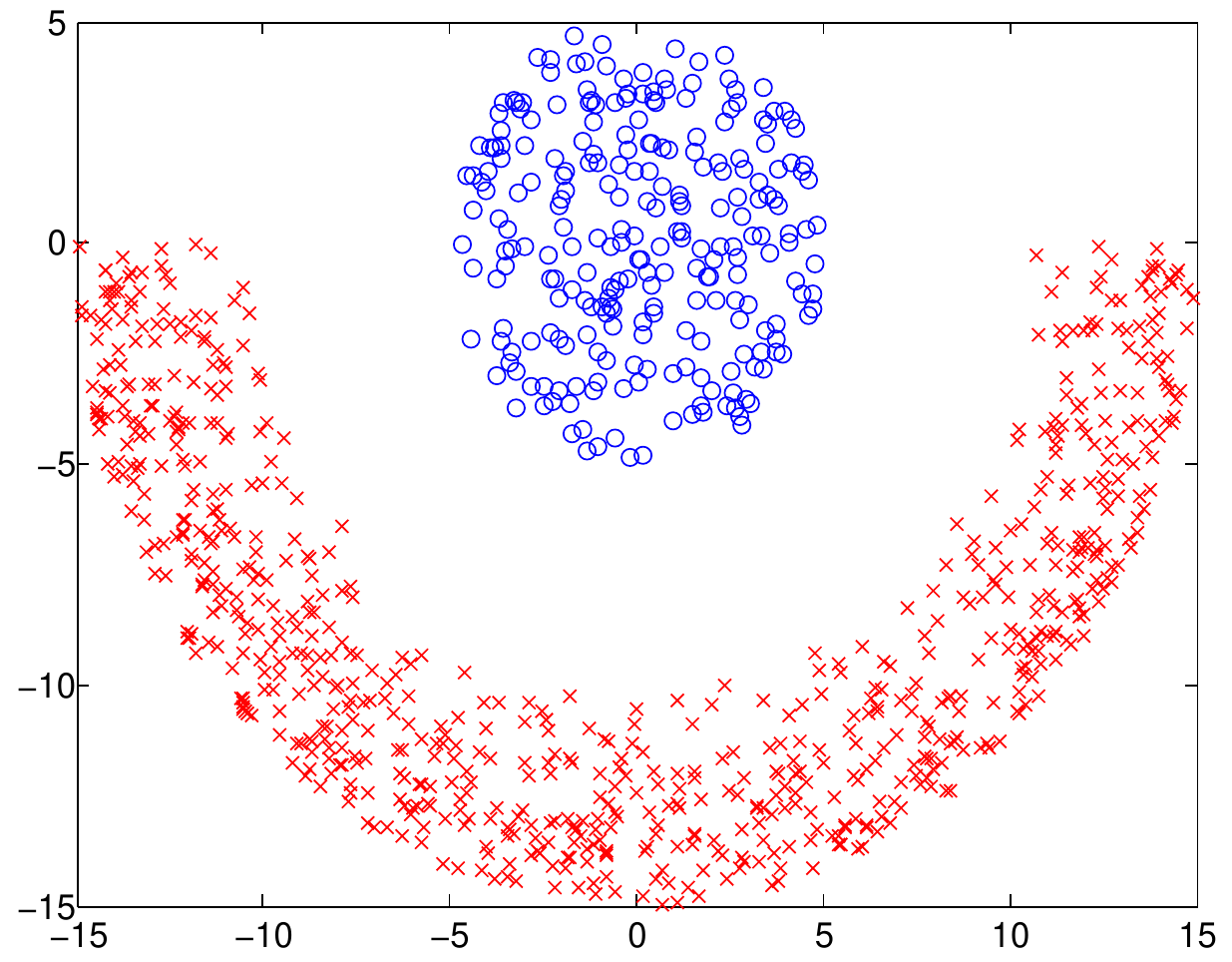}\\
                b. Full moon
        \end{minipage}
        \begin{minipage}[t]{0.325\linewidth}
                \centering
                \includegraphics[width=\textwidth]{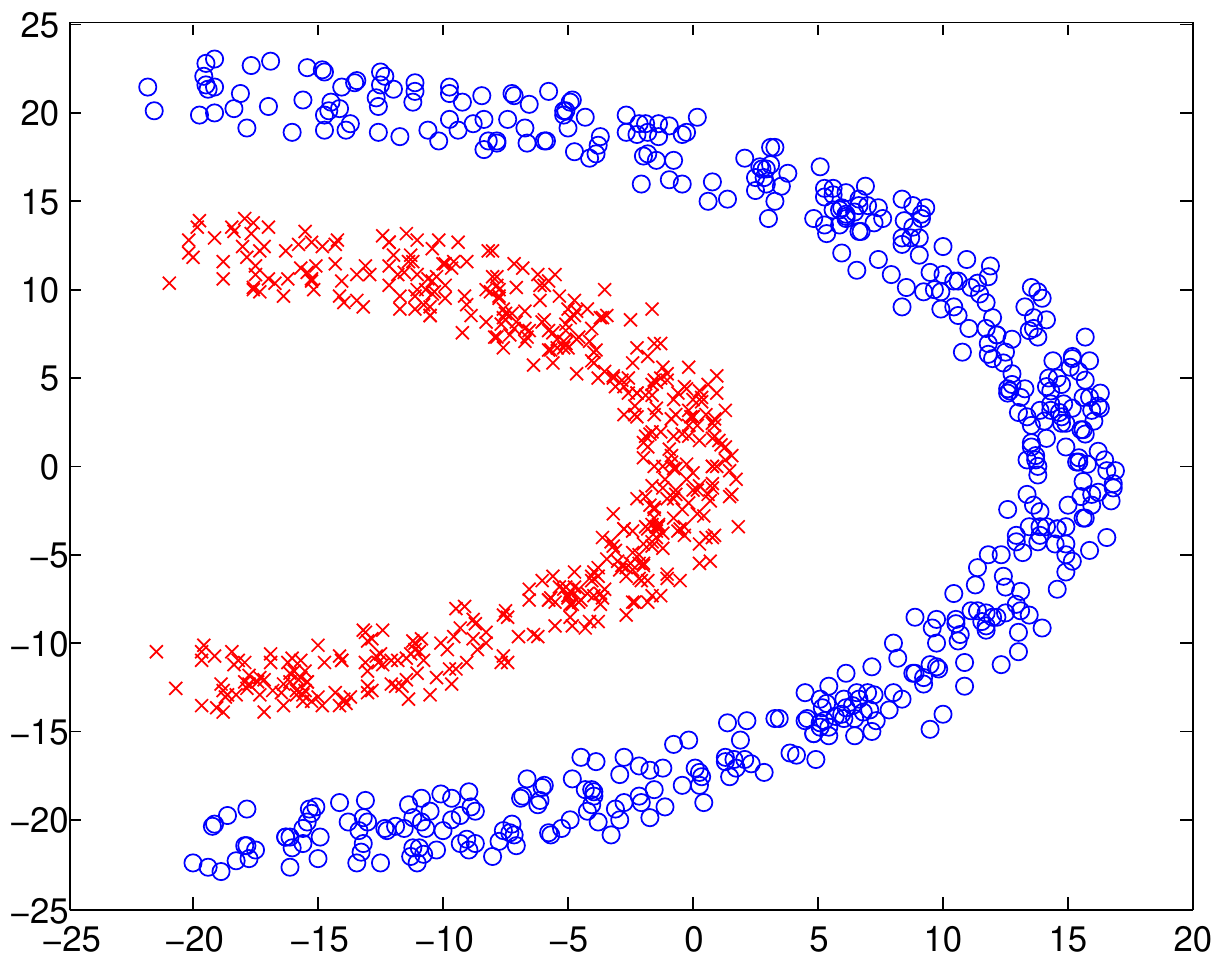}\\
                c. Half kernel
        \end{minipage}
        \begin{minipage}[t]{0.325\linewidth}
                \centering
                \includegraphics[width=\textwidth]{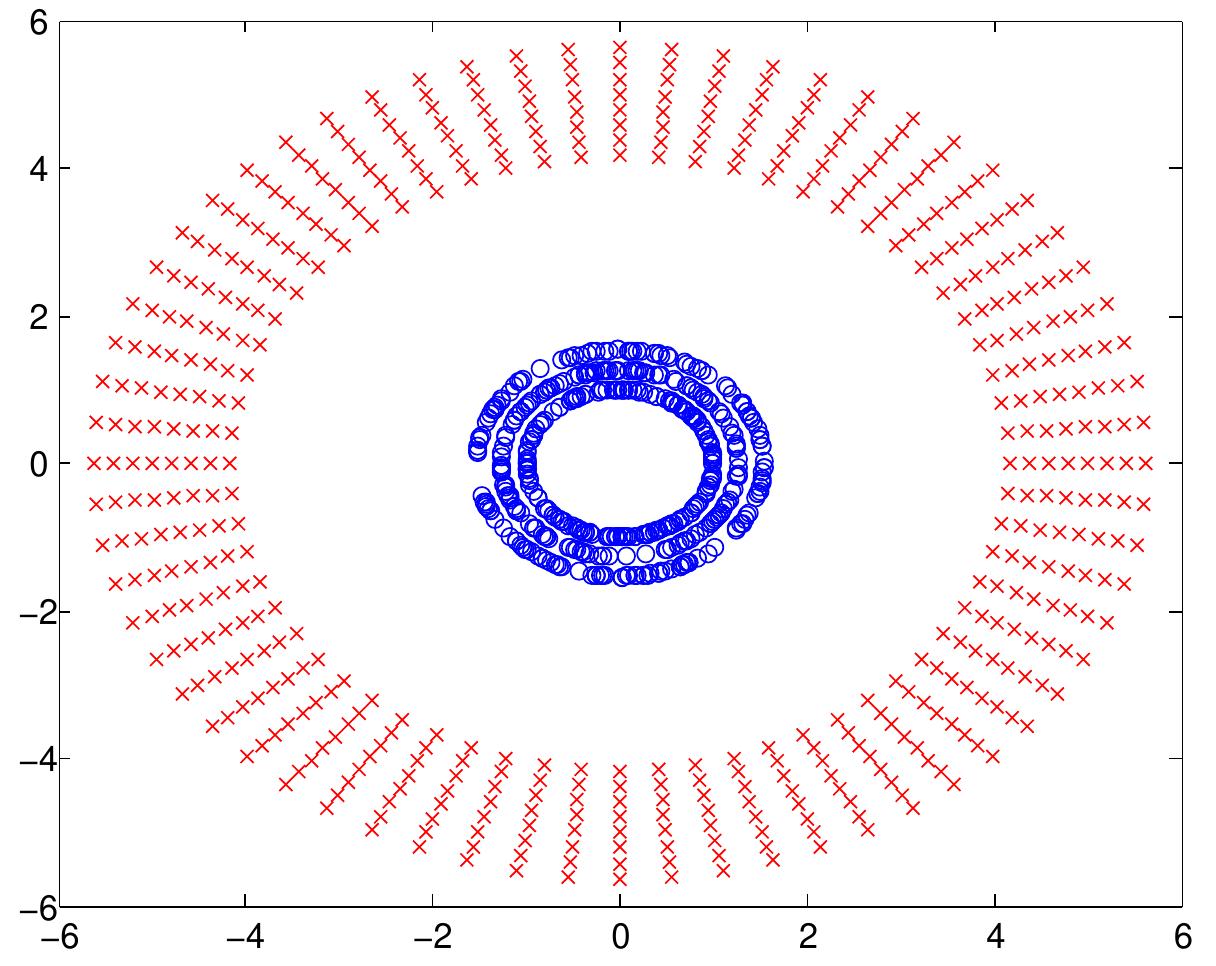}\\
                d. Cluster in cluster
        \end{minipage}
        \begin{minipage}[t]{0.325\linewidth}
                \centering
                \includegraphics[width=\textwidth]{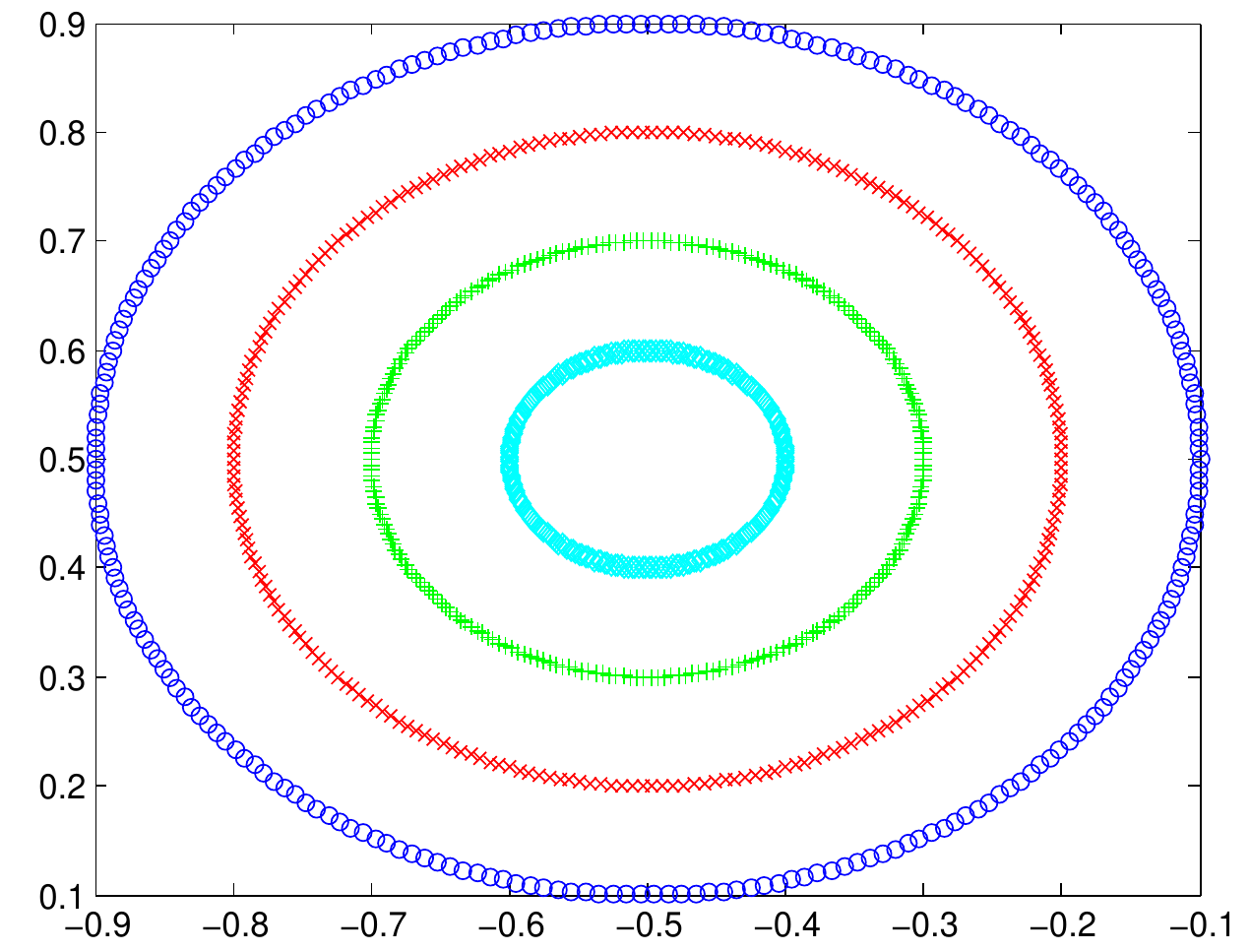}\\
                e. Dartboard
        \end{minipage}
        \begin{minipage}[t]{0.325\linewidth}
                \centering
                \includegraphics[width=\textwidth]{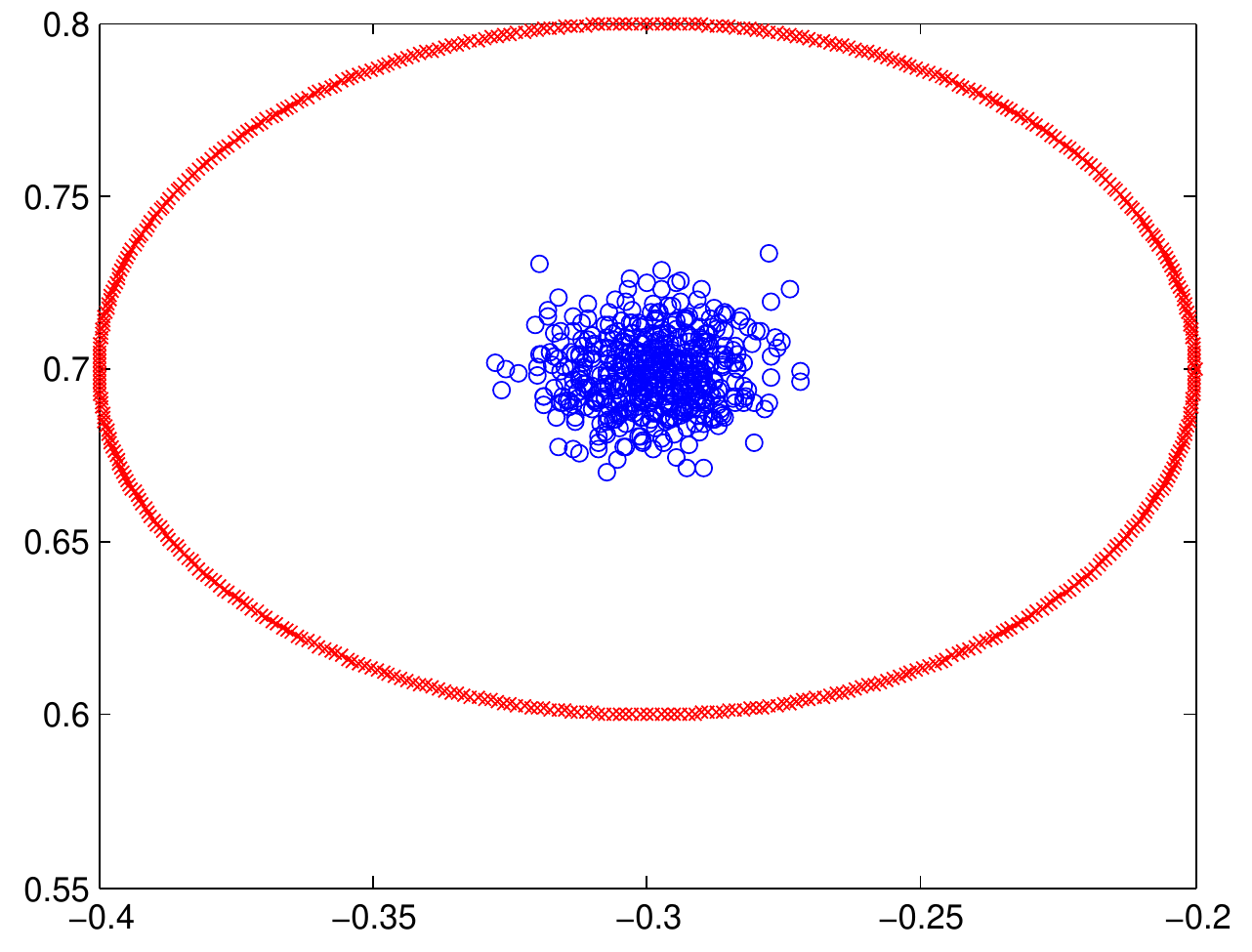}\\
                f. Donut
        \end{minipage}
        \begin{minipage}[t]{0.325\linewidth}
                \centering
                \includegraphics[width=\textwidth]{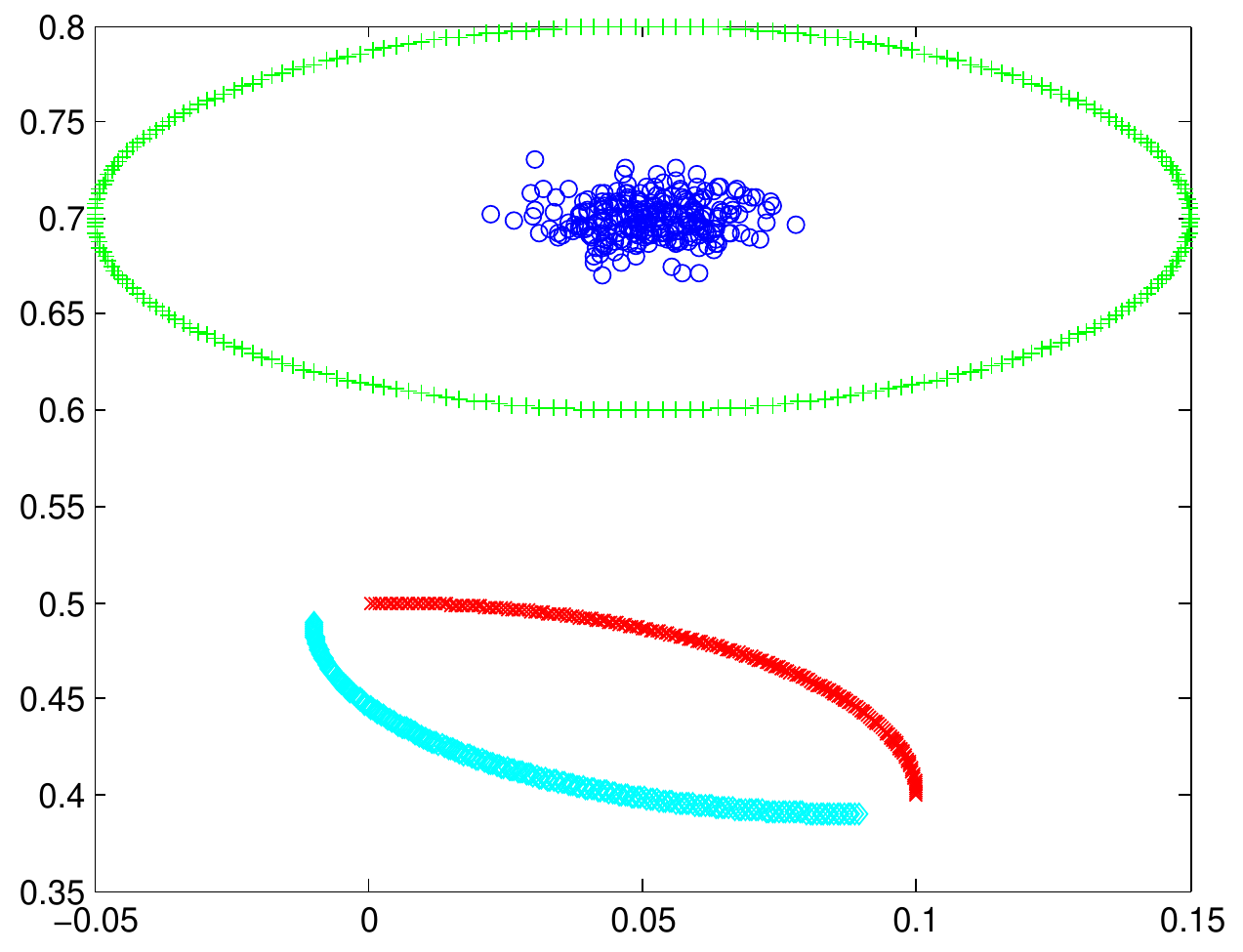}\\
                g. Donutcurves
        \end{minipage}
        \begin{minipage}[t]{0.325\linewidth}
                \centering
                \includegraphics[width=\textwidth]{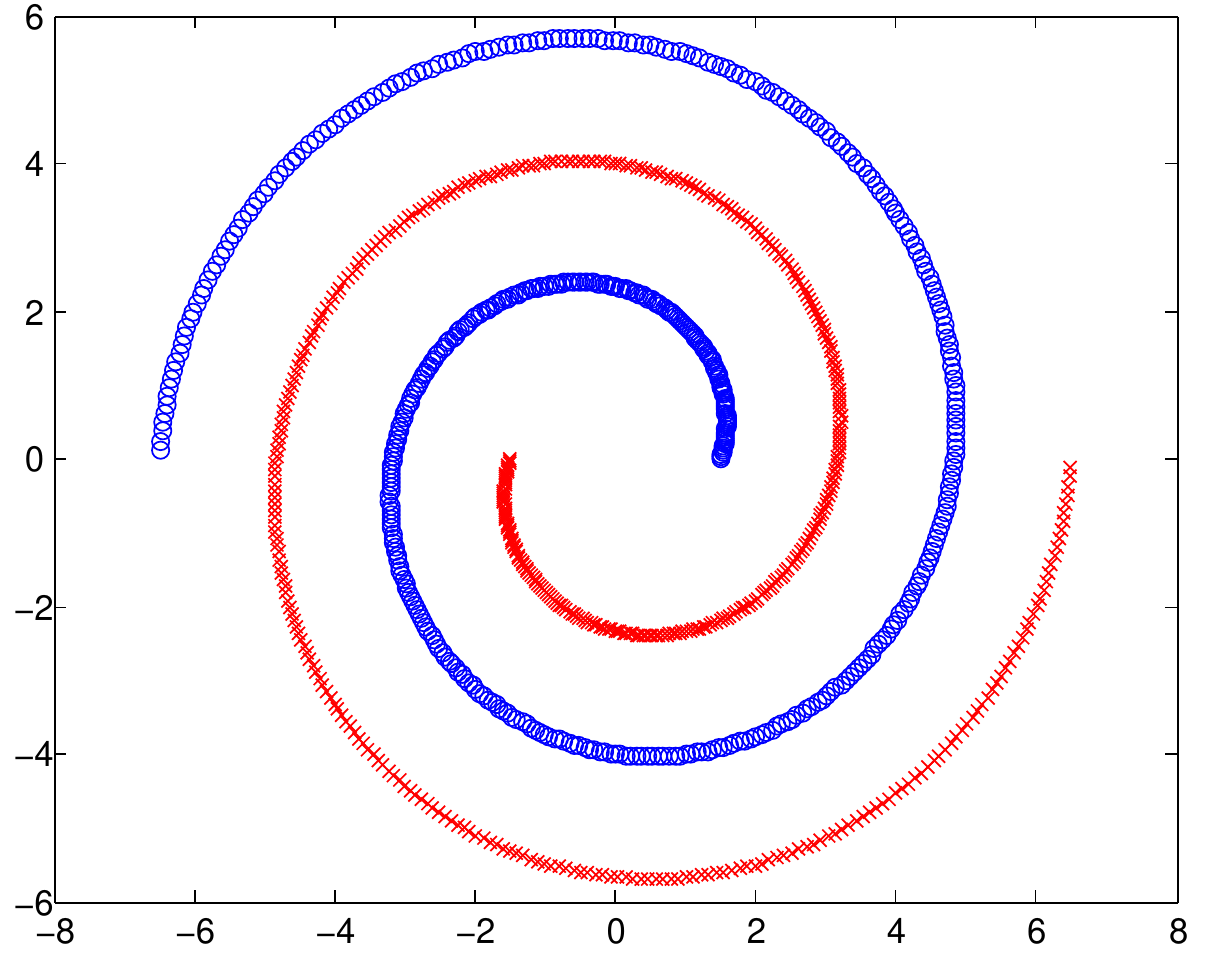}\\
                h. Spiral1
        \end{minipage}
        \begin{minipage}[t]{0.325\linewidth}
                \centering
                \includegraphics[width=\textwidth]{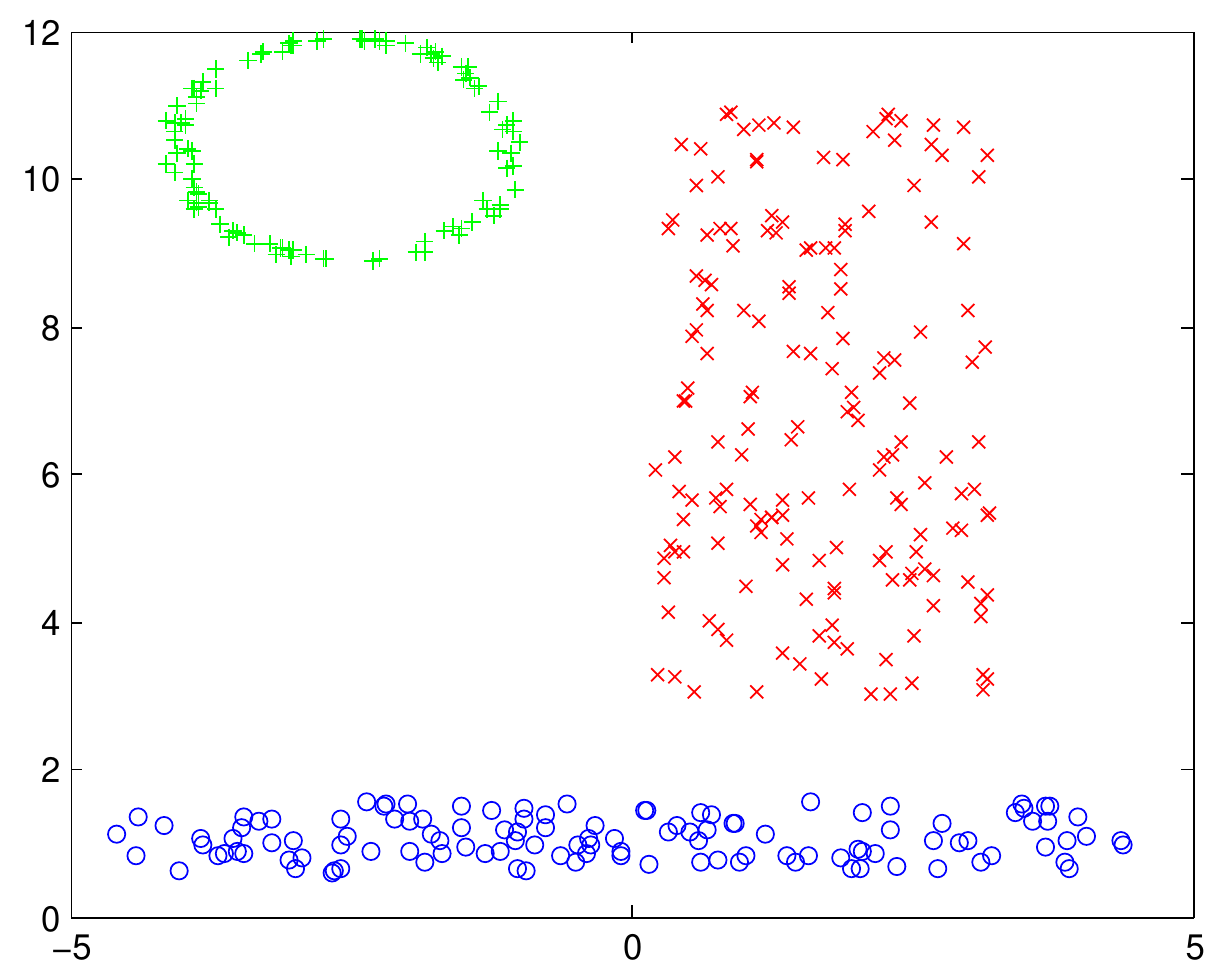}\\
                i. 3MC
        \end{minipage}
        \begin{minipage}[t]{0.325\linewidth}
                \centering
                \includegraphics[width=\textwidth]{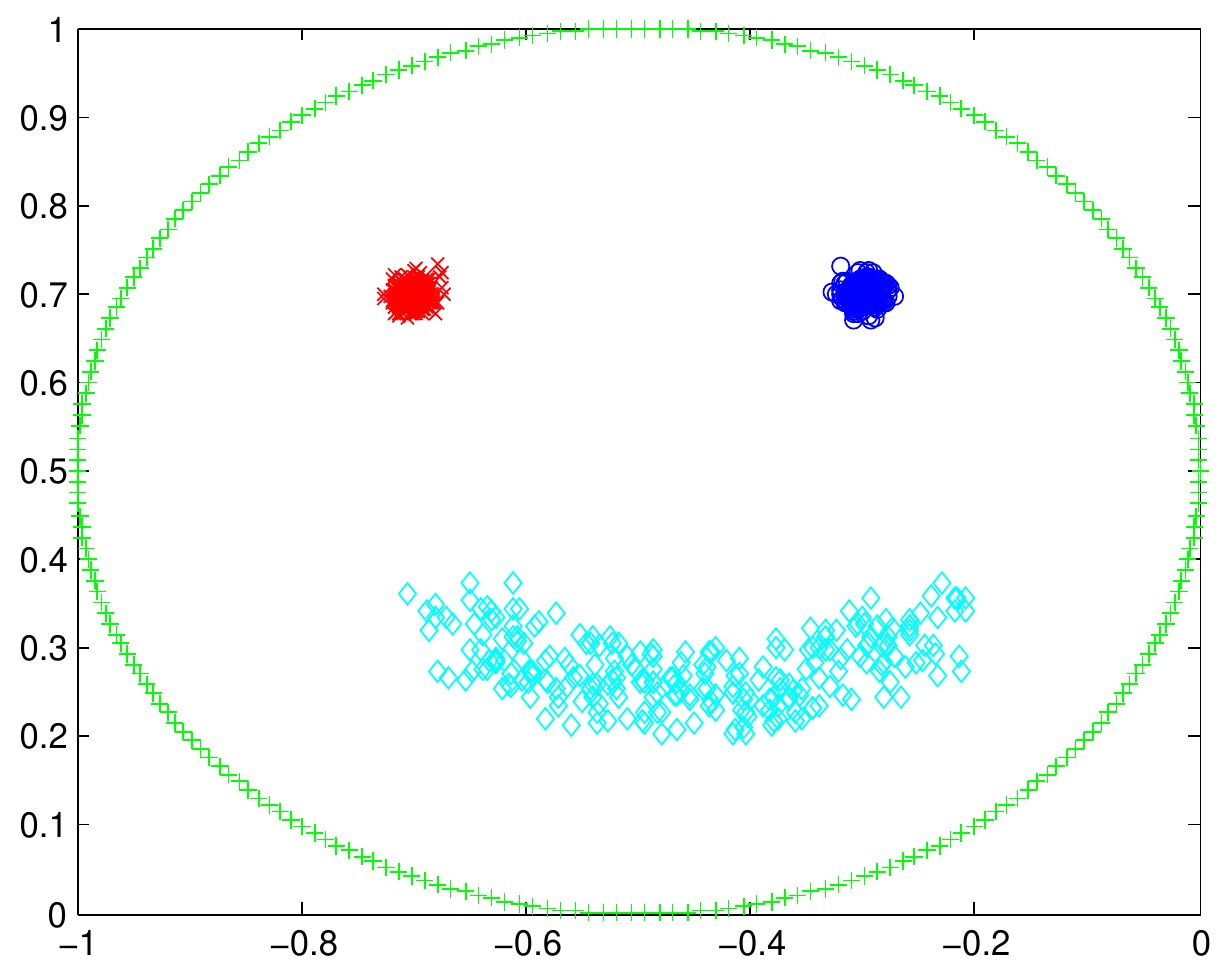}\\
                j. Smile
        \end{minipage}
        \begin{minipage}[t]{0.325\linewidth}
                \centering
                \includegraphics[width=\textwidth]{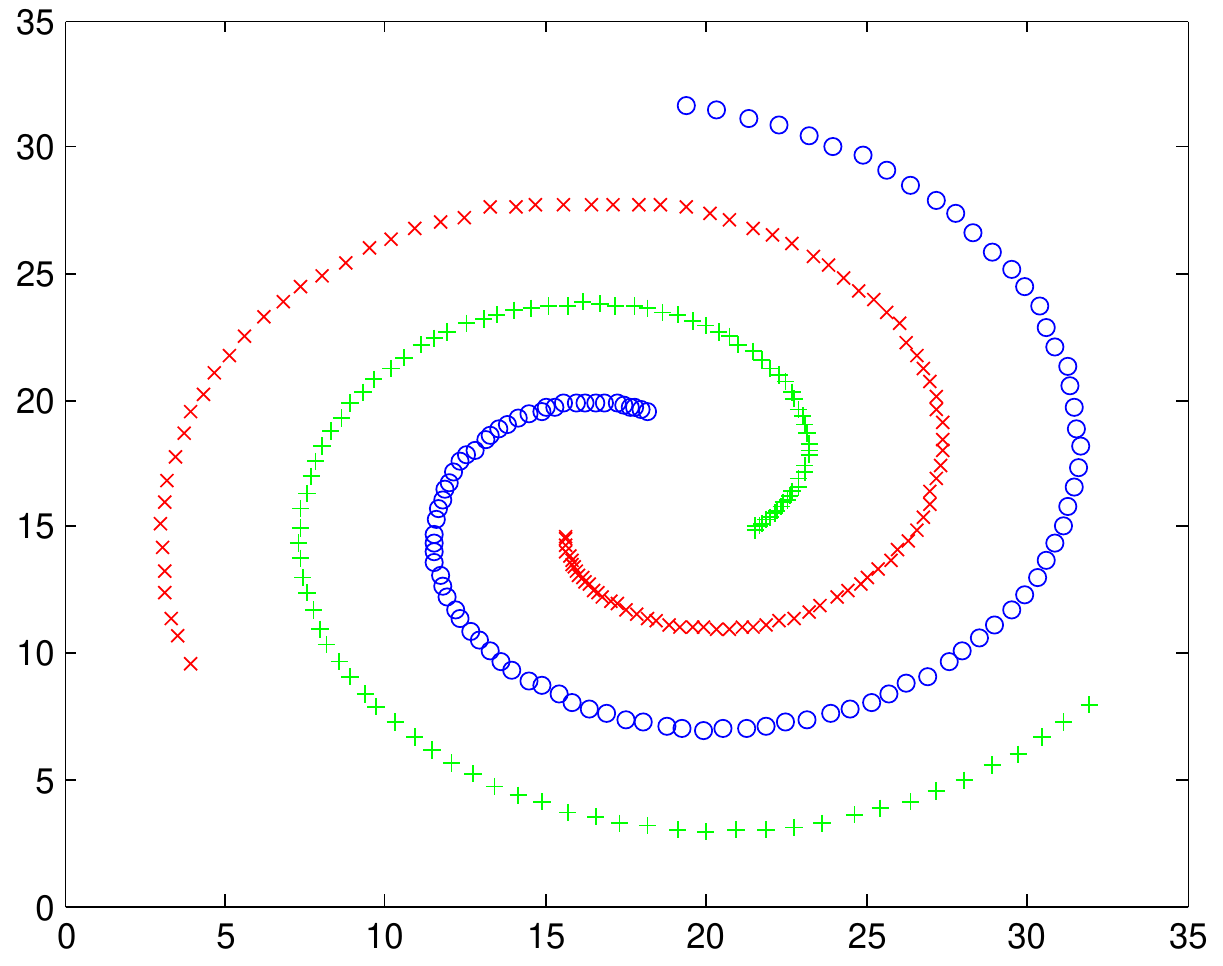}\\
                k. Spiral2
        \end{minipage}
        \begin{minipage}[t]{0.325\linewidth}
                \centering
                \includegraphics[width=\textwidth]{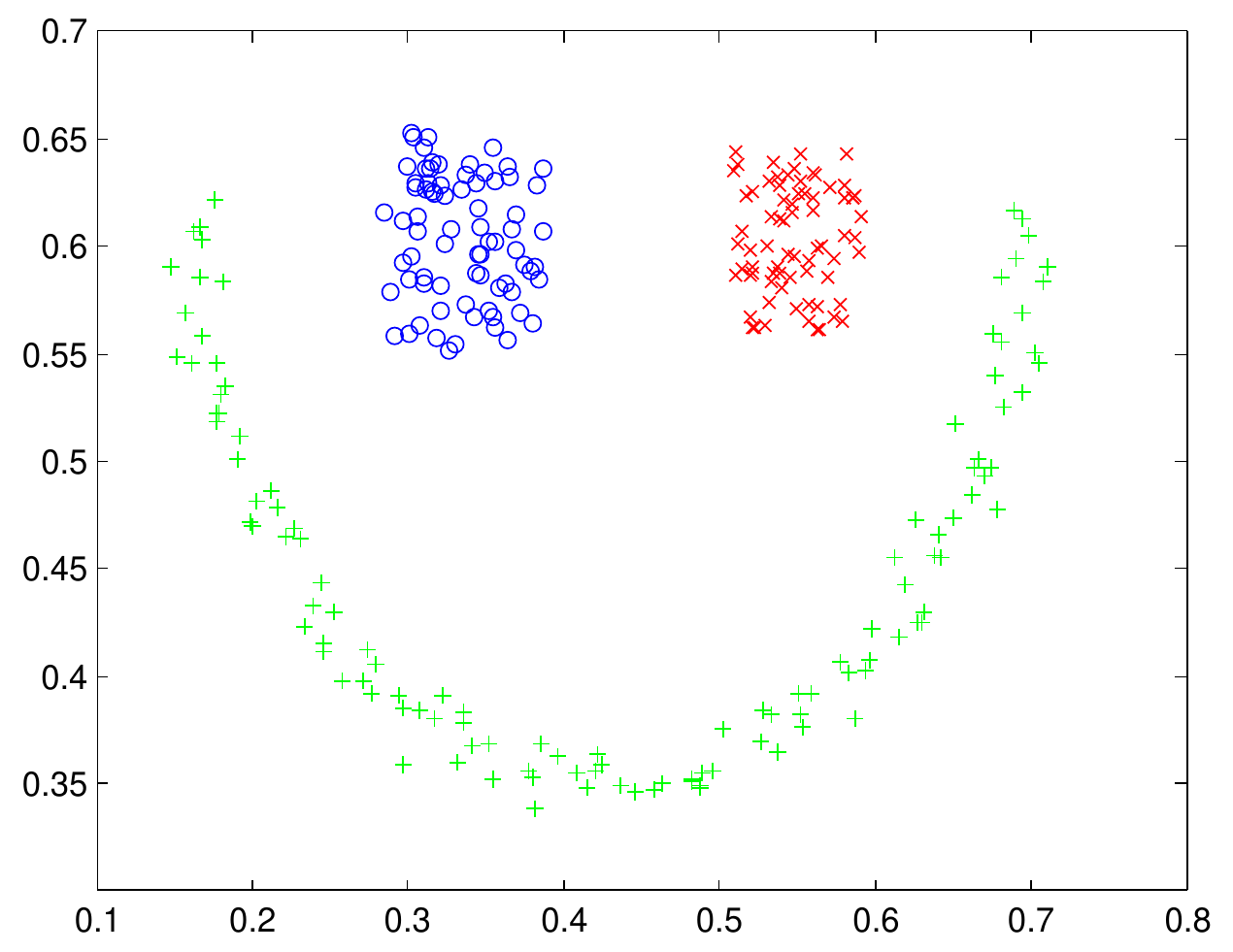}\\
                l. Zelnik
        \end{minipage}
        \caption{CRL on some standard    clustering benchmarks. The obtained clusters are labeled in difference colors.}
        \label{kernel_res}
\end{figure}

We  then  created more clusters and increased    dimensions to make   the problems more challenging. Specifically,  a new dataset Zelnik$^*$ with $q^*=12, m=50$   was created by shifting the data points of Zelnik   four times followed by   (right)multiplying the   data matrix
by  a        random Gaussian matrix. Similarly, we constructed a new dataset Donutcurves$^*$ with $q^*=12$ and $m=100$.
In comparing     spectral clustering and kernel CRL, we considered   two standard types of   the similarity  matrix $\bsbW$, one based on   Gaussian similarities, the other based on    mutual   $k$-nearest neighbor graph with  $k=20$; see   \cite{von2007tutorial} for details.
The
normalized graph Laplacian is then    $\bsbL =\bsbI - \bsbD^{-1/2} \bsbW \bsbD^{-1/2}$.
Table \ref{sim_mbar} and Table \ref{sim_mbar2} summarize the results with  each experiment   repeated $50$ times. In these tough situations,   no algorithm gives perfect clustering, but     CRL    clearly outperforms    spectral clustering and again, performing a further dimension reduction  is still possible and helpful.

\begin{table}
        \caption{\label{sim_mbar}Clustering accuracy  on two enlarged benchmarks   using \textit{Gaussian} similarity matrices.}
        \setlength\tabcolsep{2pt}
        \renewcommand{\arraystretch}{1}
        \centering
        \footnotesize

        \begin{tabular}{l c c  }

                \hline
                Dataset  & Zelnik$^*$  & Donutcurves$^*$ \\
                $(n,m,q^*)$& $(1064,50,12)$  &$(3000,100,12)$ \\
                \hline
                spectral clustering $\qquad$ & 0.70 & 0.71  \\
                CRL ($r = q$)  &0.74& 0.92 \\
                CRL ($r = 0.75q$)  &0.76& 0.94 \\
                \hline
        \end{tabular}

\end{table}

\begin{table}
        \caption{\label{sim_mbar2}Clustering accuracy  on two enlarged benchmarks  using   \textit{mutual $20$-nearest neighbor}   based   similarity matrices.}
        \centering
        \setlength\tabcolsep{2pt}
        \renewcommand{\arraystretch}{1}
        \footnotesize

        \begin{tabular}{l c c  }

                \hline
                Dataset  & Zelnik$^*$  & Donutcurves$^*$ \\
                $(n,m,q^*)$& $(1064,50,12)$  &$(3000,100,12)$ \\
                \hline
                Spectral clustering $\qquad$ & 0.68 & 0.68  \\
                CRL ($r = q$)  &0.72 & 0.89 \\
                CRL ($r = 0.75q$) &0.73& 0.91 \\
                \hline
        \end{tabular}

\end{table}


\subsection{Community detection}
We   performed systematic experiments on some community detection benchmarks.   Although CRL is not specially designed to solve this kind of problem,
it has impressive performance even   compared with some state-of-the-art methods. 

This parts uses two popular network community detection benchmarks,  the GN benchmark \citep{girvan2002community,newman2004finding} and the LFR benchmark \citep{lancichinetti2008benchmark}  to show    the performance of CRL in comparison with some  widely used community detection methods. In the GN benchmark,   $n$ nodes   are divided into $q$ homogeneous clusters of size $s$. The parameters   $z_{in}$ and $z_{out}$   control the internal degree and the external degree of each node, respectively. We set $n=1000$, $q=20$, $s=50$, $z_{in}=15$ and $z_{out} =30$, where    $z_{out}$ doubles the size of  $z_{in}$ to increase the  difficulty in clustering.
In the LFR benchmark, both the degree distribution and  community size distribution follow power laws to generate a network with more heterogeneities.  The  parameters are: $n$, the number of nodes, $d$, the average degree, $d_{\max}$, the maximum node degree, $\mu$, the mixing ratio, $\beta$, the power index of community size, $\gamma$, the power index of node degree and $c_{\min}$ ($c_{\max}$), the minimum (maximum) community size.  In our experiment, we used   $n =1000$, $d=15$, $d_{\max}=50$, $c_{\min} = 20$, $c_{\max} =50$, $\gamma=2$, $\beta=1$, and   varied the crucial parameter  $\mu$ from $0.1$ to $0.9$.
The metrics include,  in addition to  CA and RI, 
  the normalized mutual information (NMI) \citep{danon2005comparing}.  Both RI and NMI have range  $[0, 1]$, and the larger the value is, the more accurate the tested algorithm.
The comparison community detection methods include the Newman-Girvan method \citep{newman2004finding}, Hespanha's algorithm \citep{hespanha2004efficient}, the method by Dannon et al \citep{danon2006effect} and AMOS \citep{chen2016amos}.  In all algorithms, we set $q$ to be $q^*$ to remove the influence of   ad-hoc tuning schemes.
The results   on the GN benchmark are reported   in Table \ref{tab_GN}. For the LFR benchmark, we plot   the performance  of each algorithm
as a function of
the mixing parameter $\mu$ in Figure \ref{lfr}. In summary, although CRL is not particularly designed for network community detection, it does a decent job on the standard benchmarks.

\begin{table}
        \caption{\label{tab_GN}GN benchmark: clustering performance  in terms of CA, RI, NMI }
        \centering

        \begin{tabular}{l c c  c}
                \hline
                & CA  & RI & NMI \\
                \hline
                Newman-Girvan $\quad$   & 0.34 & 0.77 & 0.39\\
                Hespanha  & 0.30 & 0.91& 0.29 \\
                Danon et al.   & 0.36& 0.89 & 0.37 \\
                AMOS & 0.67 & 0.92& 0.65 \\
                CRL 
                & 0.96 & 0.98&  0.94\\
                \hline
        \end{tabular}

\end{table}

\begin{figure}[htbp]
        \centering\includegraphics[width=4in]{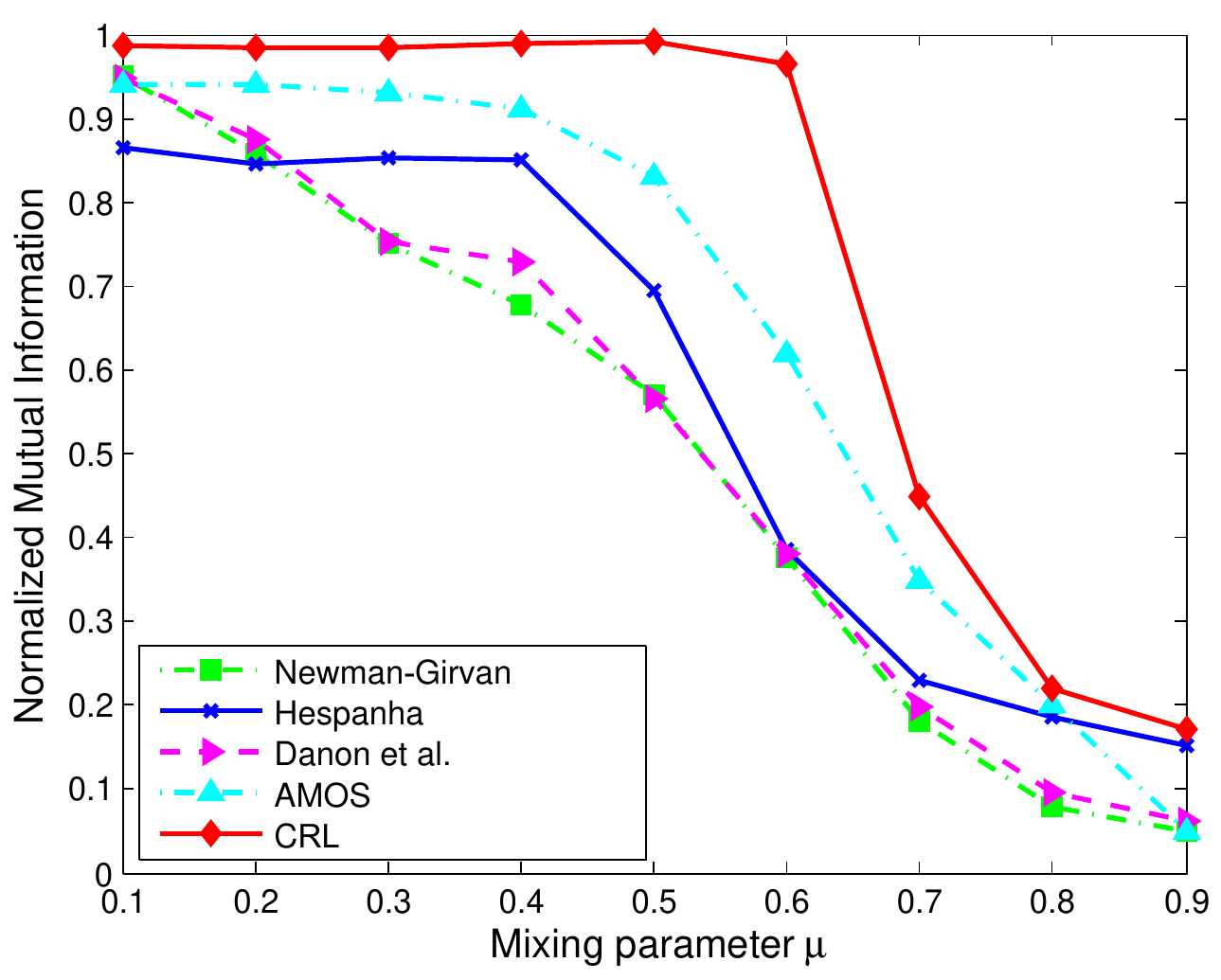}
        \caption{LFR benchmark: clustering performance of   different algorithms evaluated by NMI, as  the mixing parameter $\mu$ varies from 0.1 to 0.9. }
        \label{lfr}
\end{figure}

\subsection{Model misspecification }

This part performs simulations to test the performance of CRL and some other methods when model misspecification occurs. The $n$ observations of the responses and predictors are independently  generated according to $\bsby_i = \bsbB^{*T}   \tilde \bsbx_i + \bsb{e}_i$,  where $\bsby_i\in \mathbb R^m$, $\tilde \bsbx_i\in \mathbb R^p$, and  $\bsb{e}_i$ has standard normal entries ($\mu=0,\sigma=1$).
In the experiment,  $n=100, p = 50, m =25$, and all rows of the design   follow a multivariate normal distribution with a covariance matrix $\bsbSig=[\tau^{|i-j|}]$ where $\tau=0.2$.

To study  the effect of model misspecification, we  generate the true coefficient matrix  $\bsbB^* $ through $\bsbB^\circ$ and $\sigma_B$  as follows. First,     $\bsbB^\circ  = \bsbB_1 \bsbB_2^T$, where $\bsbB_1\in \mathbb R^{p\times r}, \bsbB_2\in \mathbb R^{m\times r}$. The entries of $ \bsbB_2$ are standard Gaussian, and each row $\bsbb_j$ ($1\le j \le p$) of $\bsbB_1$ is drawn from an equally weighted mixture distribution:  $\bsbb_j$ is zero or has an elementwise distribution    $N(k, 1)$,  $1\le k \le q-1$.     We set $q=10, r=5$, and so  $\bsbB^\circ$ has about 10\% rows being zero, and satisfies $\| \bsbB^\circ\|_{2, \mathcal C}=10$, $\mbox{rank}(\bsbB^\circ) = 5$.     Next, we obtain  $\bsbB^*\in \mathbb R^{p\times m}$, consisting of  1,250 entries,  by   adding additional noise $N(0, \sigma_B^2) $ to $\bsbB^\circ $ in an  \textit{elementwise} fashion. Clearly, when $\sigma_B=0$, the true coefficient matrix  has sparsity, row-wise equisparsity, and low rank, but the structural parsimony is easily destroyed by a nonzero       $\sigma_B$.

The methods for comparison include LASSO, group LASSO (G-LASSO), reduced rank regression (RRR), fused LASSO (F-LASSO), and CRL. All the  algorithmic parameters are set to their default values.
Our goal  in this experiment is to make a fair comparison of all methods and understand the true potential of each;  to remove the influence of various schemes for regularization parameter tuning,  we pick the estimate along the solution path that gives  the smallest validation error, evaluated  on a  large independent validation dataset with 10,000 observations.   We varied the misspecification level  $\sigma_B$  from $0$ to $0.12$ and  repeated the experiment in each setup for 20 times. Table \ref{sim_table_misspec}  reports the median  estimation error $\|\hat \bsbB - \bsbB^*\|_F^2$ (denoted by $\mbox{Err}^{(e)}$), as well as the   predictor error  $\| \bsbY_{tst}   -   \bsbX_{tst}  \hat \bsbB \|_F^2     /n_{tst} - m $ over the test dataset $(\bsbY_{tst},\bsbX_{tst} )$ with  $n_{tst}=10\mbox{,}000$   observations (denoted by $\mbox{Err}^{(p)}$).

\begin{table}
    \caption{\label{sim_table_misspec}  Performance comparison between LASSO, group LASSO (G-LASSO), reduced rank regression (RRR), fused LASSO (F-LASSO), and CRL,  based on estimation error $\mbox{(Err}^{(e)}$)  and prediction error ($\mbox{Err}^{(p)} $), when varying the model misspecification level $\sigma_B$ }
    \centering
    \setlength\tabcolsep{2pt}
    \renewcommand{\arraystretch}{1}
    \footnotesize 

    \begin{tabular}{lc  cc cc cc cc cc cc cc cc cc cc cc}
        \hline
        &&\multicolumn{2}{c}{$\sigma_{\scriptscriptstyle B}\!=\!0$}   && \multicolumn{2}{c}{$\sigma_{\scriptscriptstyle B}\!=\!0.04$}   &&
          \multicolumn{2}{c}{$\sigma_{\scriptscriptstyle B}\!=\!0.08$}&&
          \multicolumn{2}{c}{$\sigma_{\scriptscriptstyle B}\!=\!0.12$} && \multicolumn{2}{c}{$\sigma_{\scriptscriptstyle B}\!=\!0.16$} \\
        \cmidrule(lr){3-4}  \cmidrule(lr){6-7} \cmidrule(lr){9-10} \cmidrule(lr){12-13} \cmidrule(lr){15-16}
        &&{\scriptsize $\mbox{Err}^{(e)}$ } & {\scriptsize$\mbox{Err}^{(p)}$} &&{\scriptsize$\mbox{Err}^{(e)}$} & {\scriptsize$\mbox{Err}^{(p)}$}
        &&{\scriptsize$\mbox{Err}^{(e)}$} & {\scriptsize$\mbox{Err}^{(p)}$}&&{\scriptsize$\mbox{Err}^{(e)}$} & {\scriptsize$\mbox{Err}^{(p)}$}
        &&{\scriptsize$\mbox{Err}^{(e)}$} & {\scriptsize$\mbox{Err}^{(p)}$} \\
        \hline\hline
        LASSO       &&22.6&  20.9 &&22.8&  21.0 && 23.0& 21.4 && 23.6& 21.9 && 23.9& 22.1 \\
        G-LASSO     &&22.7&  21.4 &&23.0&  21.7 && 23.8& 22.4 && 25.1& 23.1 && 26.0& 24.3 \\
        \hdashline[1pt/2pt]
        RRR         &&6.58&  6.21 &&8.17&  7.79 && 13.1& 12.6 && 21.0& 20.1 && 25.3& 23.8 \\
        \hdashline[1pt/2pt]
        F-LASSO     &&22.7&  21.6 &&23.0&  21.7 && 23.1& 21.9 && 23.6& 22.3 && 24.1& 22.5 \\
        \hdashline[1pt/2pt]
        CRL         &&2.12&  2.14 &&4.15&  4.14 && 10.2& 10.1 && 20.1& 19.5 && 25.1& 23.7 \\
        \hline

        \hline
    \end{tabular}
\end{table}

First, as $\sigma_B=0$, although a total of  {10\%} of  authentic coefficients  are zero, the sparsity (or group-wise sparsity) is still inadequate  for LASSO or group LASSO to show a clear advantage over the other methods.  Fused LASSO, which imposes sparsity on \textit{successive} coefficient differences $|b_{j+1, k} - b_{j, k}|$ ($1\le  j< p, 1\le k\le m$) in addition to performing variable selection,    failed to  capture all the equi-sparsity, largely  because  it assumes that     the features  are already  grouped and ordered   to have the associated   coefficients   equal in consecutive blocks (which is however impractical in most real life applications).

According to the table, when $\sigma_B$ is large enough (say, $\sigma_B\ge 0.16$) to break the   parsimony assumptions,   all methods (unsurprisingly) yield  very  complex models with large errors,   the  saturated model being the extreme. But the low-rank based RRR, and notably CRL,  outperformed the other methods by a large margin when the large coefficient matrix  only has  \textit{approximate} low rank and equisparsity. This was    evidenced  as long as  $\sigma_B \le 0.1$ by more systematic experiments (not all shown in the table). The excellent performance of CRL  in these cases verifies the oracle-inequality type analysis   mentioned at the end of Section \ref{subsec:upperrCRL}. 

{
       \bibliography{crl}

\begin{thebibliography}{}

\bibitem[Agresti, 2012]{agresti2012categorical}
Agresti, A. (2012).
\newblock {\em Categorical Data Analysis}.
\newblock Wiley Series in Probability and Statistics. Wiley.

\bibitem[Arthur and Vassilvitskii, 2007]{arthur2007k}
Arthur, D. and Vassilvitskii, S. (2007).
\newblock {K-means++: The advantages of careful seeding}.
\newblock In {\em Proceedings of the Eighteenth Annual ACM-SIAM Symposium on
  Discrete Algorithms}, pages 1027--1035. Society for Industrial and Applied
  Mathematics.

\bibitem[Bachem et~al., 2016]{bachem2016fast}
Bachem, O., Lucic, M., Hassani, S.~H., and Krause, A. (2016).
\newblock {Fast and provably good seedings for K-means}.
\newblock In {\em Proceedings of the 30th International Conference on Neural
  Information Processing Systems}, NIPS'16, pages 55--63. Curran Associates
  Inc.

\bibitem[Bickel et~al., 2009]{Bickel09}
Bickel, P.~J., Ritov, Y., and Tsybakov, A.~B. (2009).
\newblock Simultaneous analysis of {L}asso and {D}antzig selector.
\newblock {\em The Annals of Statistics}, 37:1705--1732.

\bibitem[Bregman, 1967]{Bregman1967}
Bregman, L. (1967).
\newblock {The relaxation method of finding the common point of convex sets and
  its application to the solution of problems in convex programming}.
\newblock {\em USSR Computational Mathematics and Mathematical Physics},
  7(3):200 -- 217.

\bibitem[Breiman et~al., 1984]{breiman1984classification}
Breiman, L., Friedman, J., Stone, C., and Olshen, R. (1984).
\newblock {\em Classification and Regression Trees}.
\newblock Taylor \& Francis, Monterey, CA.

\bibitem[Bunea et~al., 2011]{BSW11}
Bunea, F., She, Y., and Wegkamp, M. (2011).
\newblock {Optimal selection of reduced rank estimators of high-dimensional
  matrices}.
\newblock {\em The Annals of Statistics}, 39:1282--1309.

\bibitem[Cai et~al., 2005]{cai2005document}
Cai, D., He, X., and Han, J. (2005).
\newblock {Document clustering using locality preserving indexing}.
\newblock {\em IEEE Transactions on Knowledge and Data Engineering},
  17(12):1624--1637.

\bibitem[Cand\`es and Tao, 2007]{candes2007dantzig}
Cand\`es, E. and Tao, T. (2007).
\newblock {The Dantzig selector: Satistical estimation when p is much larger
  than n}.
\newblock {\em The Annals of Statistics}, pages 2313--2351.

\bibitem[Cand\`es and Plan, 2011]{CandPlan}
Cand\`es, E.~J. and Plan, Y. (2011).
\newblock Tight oracle bounds for low-rank matrix recovery from a minimal
  number of random measurements.
\newblock {\em IEEE Transactions on Information Theory}, 57(4):2342--2359.

\bibitem[Chang and Yeung, 2008]{chang2008robust}
Chang, H. and Yeung, D.-Y. (2008).
\newblock {Robust path-based spectral clustering}.
\newblock {\em Pattern Recognition}, 41(1):191--203.

\bibitem[Chen et~al., 2016]{chen2016amos}
Chen, P.-Y., Gensollen, T., and Hero~III, A.~O. (2016).
\newblock Amos: An automated model order selection algorithm for spectral graph
  clustering.
\newblock {\em arXiv preprint arXiv:1609.06457}.

\bibitem[Chi and Lange, 2015]{chi2015splitting}
Chi, E.~C. and Lange, K. (2015).
\newblock { Splitting methods for convex clustering}.
\newblock {\em Journal of Computational and Graphical Statistics},
  24:4:994--1013.

\bibitem[Chun and Kele{\c{s}}, 2010]{chun2010sparse}
Chun, H. and Kele{\c{s}}, S. (2010).
\newblock {Sparse partial least squares regression for simultaneous dimension
  reduction and variable selection}.
\newblock {\em Journal of the Royal Statistical Society: Series B (Statistical
  Methodology)}, 72(1):3--25.

\bibitem[Danon et~al., 2006]{danon2006effect}
Danon, L., D{\'\i}az-Guilera, A., and Arenas, A. (2006).
\newblock The effect of size heterogeneity on community identification in
  complex networks.
\newblock {\em Journal of Statistical Mechanics: Theory and Experiment},
  2006(11):P11010.

\bibitem[Danon et~al., 2005]{danon2005comparing}
Danon, L., Diaz-Guilera, A., Duch, J., and Arenas, A. (2005).
\newblock {Comparing community structure identification}.
\newblock {\em Journal of Statistical Mechanics: Theory and Experiment},
  2005(09):P09008.

\bibitem[Donoho and Johnstone, 1994]{donoho1994}
Donoho, D.~L. and Johnstone, J.~M. (1994).
\newblock {Ideal spatial adaptation by wavelet shrinkage}.
\newblock {\em Biometrika}, 81(3):425--455.

\bibitem[Girvan and Newman, 2002]{girvan2002community}
Girvan, M. and Newman, M.~E. (2002).
\newblock Community structure in social and biological networks.
\newblock {\em Proceedings of the National Academy of Sciences},
  99(12):7821--7826.

\bibitem[G{\"o}tze et~al., 2021]{gotze2019concentration}
G{\"o}tze, F., Sambale, H., and Sinulis, A. (2021).
\newblock Concentration inequalities for polynomials in
  $\alpha$-sub-exponential random variables.
\newblock {\em Electronic Journal of Probability}, 26:1--22.

\bibitem[Hastie et~al., 2009]{ESL2}
Hastie, T., Tibshirani, R., and Friedman, J. (2009).
\newblock {\em The Elements of Statistical Learning}.
\newblock Springer-Verlag, New York, 2nd edition.

\bibitem[Hespanha, 2004]{hespanha2004efficient}
Hespanha, J.~P. (2004).
\newblock An efficient matlab algorithm for graph partitioning.
\newblock {\em Santa Barbara, CA, USA: University of California}.

\bibitem[Izenman, 1975]{izenman1975reduced}
Izenman, A.~J. (1975).
\newblock {Reduced-rank regression for the multivariate linear model}.
\newblock {\em Journal of Multivariate Analysis}, 5(2):248--264.

\bibitem[Jain and Law, 2005]{jain2005data}
Jain, A. and Law, M. (2005).
\newblock {Data clustering: A user's dilemma}.
\newblock {\em Pattern Recognition and Machine Intelligence}, pages 1--10.

\bibitem[Johnstone and Lu, 2009]{johnstone2009consistency}
Johnstone, I.~M. and Lu, A.~Y. (2009).
\newblock {On consistency and sparsity for principal components analysis in
  high dimensions}.
\newblock {\em Journal of the American Statistical Association},
  104(486):682--693.

\bibitem[J{\o}rgensen, 1987]{Jor87}
J{\o}rgensen, B. (1987).
\newblock {Exponential Dispersion Models}.
\newblock {\em Journal of the Royal Statistical Society. Series B},
  49(2):127--145.

\bibitem[Koltchinskii et~al., 2011]{Kolt11}
Koltchinskii, V., Lounici, K., and Tsybakov, A.~B. (2011).
\newblock {Nuclear-norm penalization and optimal rates for noisy low-rank
  matrix completion}.
\newblock {\em The Annals of Statistics}, 39(5):2302 -- 2329.

\bibitem[Lambert et~al., 2010]{Lambert2010}
Lambert, J.-P., Fillingham, J., Siahbazi, M., Greenblatt, J., Baetz, K., and
  Figeys, D. (2010).
\newblock Defining the budding yeast chromatin-associated interactome.
\newblock {\em Molecular Systems Biology}, 6(1):448.

\bibitem[Lancichinetti et~al., 2008]{lancichinetti2008benchmark}
Lancichinetti, A., Fortunato, S., and Radicchi, F. (2008).
\newblock {Benchmark graphs for testing community detection algorithms}.
\newblock {\em Physical Review E}, 78(4):046110.

\bibitem[Lounici et~al., 2011]{lounici11}
Lounici, K., Pontil, M., Tsybakov, A.~B., and {van de Geer}, S. (2011).
\newblock {Oracle inequalities and optimal inference under group sparsity}.
\newblock {\em The Annals of Statistics}, 39:2164--2204.

\bibitem[Lov{\'a}sz and Plummer, 2009]{lovasz2009matching}
Lov{\'a}sz, L. and Plummer, M.~D. (2009).
\newblock {\em Matching Theory}, volume 367.
\newblock AMS Chelsea Publishing: An Imprint of the American Mathematical
  Society.

\bibitem[Newman and Girvan, 2004]{newman2004finding}
Newman, M.~E. and Girvan, M. (2004).
\newblock {Finding and evaluating community structure in networks}.
\newblock {\em Physical review E}, 69(2):026113.

\bibitem[Pan and Heitman, 2000]{Pan2000}
Pan, X. and Heitman, J. (2000).
\newblock Sok2 regulates yeast pseudohyphal differentiation via a transcription
  factor cascade that regulates cell-cell adhesion.
\newblock {\em Molecular and Cellular Biology}, 20(22):8364--8372.

\bibitem[Pellikaan et~al., 2017]{pellikaan_wu_bulygin_jurrius_2017}
Pellikaan, R., Wu, X.-W., Bulygin, S., and Jurrius, R. (2017).
\newblock {\em Codes, Cryptology and Curves with Computer Algebra}.
\newblock Cambridge University Press, Cambridge, 1st edition.

\bibitem[Rand, 1971]{rand1971objective}
Rand, W.~M. (1971).
\newblock {Objective criteria for the evaluation of clustering methods}.
\newblock {\em Journal of the American Statistical association},
  66(336):846--850.

\bibitem[Rennie and Dobson, 1969]{Rennie_1969}
Rennie, B. and Dobson, A. (1969).
\newblock {On Stirling numbers of the second kind}.
\newblock {\em Journal of Combinatorial Theory}, 7(2):116 -- 121.

\bibitem[Rockafellar, 1970]{Rockafellar1970}
Rockafellar, R.~T. (1970).
\newblock {\em {Convex Analysis}}.
\newblock Princeton University Press, Princeton, NJ.

\bibitem[Rousseeuw and Driessen, 1999]{rousseeuw1999fast}
Rousseeuw, P.~J. and Driessen, K.~V. (1999).
\newblock A fast algorithm for the minimum covariance determinant estimator.
\newblock {\em Technometrics}, 41(3):212--223.

\bibitem[Sambale, 2020]{sambale2020some}
Sambale, H. (2020).
\newblock Some notes on concentration for $\alpha$-subexponential random
  variables.
\newblock {\em arXiv preprint arXiv:2002.10761}.

\bibitem[She, 2010]{she2010sparse}
She, Y. (2010).
\newblock {Sparse regression with exact clustering}.
\newblock {\em Electronic Journal of Statistics}, 4:1055--1096.

\bibitem[She, 2017]{She2017selfact}
She, Y. (2017).
\newblock {Selective factor extraction in high dimensions}.
\newblock {\em Biometrika}, 104(1):97--110.

\bibitem[She et~al., 2020]{SheiGGL2020}
She, Y., Tang, S., and Zhang, Q. (2020).
\newblock {Indirect Gaussian Graph Learning beyond Gaussianity}.
\newblock {\em IEEE Transactions on Network Science and Engineering},
  7:918--929.

\bibitem[She and Tran, 2019]{SCV}
She, Y. and Tran, H. (2019).
\newblock On cross-validation for sparse reduced rank regression.
\newblock {\em Journal of the Royal Statistical Society: Series B},
  81:145--161.

\bibitem[She et~al., 2021]{SheBregman}
She, Y., Wang, Z., and Jin, J. (2021).
\newblock Analysis of generalized {B}regman surrogate algorithms for nonsmooth
  nonconvex statistical learning.
\newblock {\em The Annals of Statistics}, 49(6):3434--3459.

\bibitem[Szarek, 1982]{szarek82}
Szarek, S.~J. (1982).
\newblock {Nets of Grassmann manifold and orthogonal groups}.
\newblock In {\em Proceedings of Banach Spaces Workshop}, volume 169, pages
  169--185.

\bibitem[Tibshirani et~al., 2005]{tibshirani2005sparsity}
Tibshirani, R., Saunders, M., Rosset, S., Zhu, J., and Knight, K. (2005).
\newblock {Sparsity and smoothness via the fused LASSO}.
\newblock {\em Journal of the Royal Statistical Society: Series B (Statistical
  Methodology)}, 67(1):91--108.

\bibitem[Tsybakov, 2008]{tsybakov2009introduction}
Tsybakov, A.~B. (2008).
\newblock {\em Introduction to Nonparametric Estimation}.
\newblock Springer Publishing Company, Incorporated, New York, 1st edition.

\bibitem[van~der Vaart and Wellner, 1996]{van1996weak}
van~der Vaart, A. and Wellner, J. (1996).
\newblock {\em Weak Convergence and Empirical Processes: With Applications to
  Statistics}.
\newblock Springer.

\bibitem[van Lint, 1982]{van2012introduction}
van Lint, J.~H. (1982).
\newblock {\em Introduction to Coding Theory}.
\newblock Springer-Verlag, Berlin, Heidelberg.

\bibitem[{v}on Luxburg, 2007]{von2007tutorial}
{v}on Luxburg, U. (2007).
\newblock A tutorial on spectral clustering.
\newblock {\em Statistics and Computing}, 17(4):395--416.

\bibitem[Wainwright and Jordan, 2008]{WainJordan08}
Wainwright, M.~J. and Jordan, M.~I. (2008).
\newblock Graphical models, exponential families, and variational inference.
\newblock {\em Foundations and Trends in Machine Learning}, 1:1--305.

\bibitem[Yuan and Lin, 2006]{yuan2006model}
Yuan, M. and Lin, Y. (2006).
\newblock {Model selection and estimation in regression with grouped
  variables}.
\newblock {\em Journal of the Royal Statistical Society: Series B (Statistical
  Methodology)}, 68(1):49--67.

\bibitem[Zhang and Xia, 2009]{zhang2009k}
Zhang, C. and Xia, S. (2009).
\newblock {K-means clustering algorithm with improved initial center}.
\newblock In {\em Proceedings of the 2009 Second International Workshop on
  Knowledge Discovery and Data Mining}, pages 790--792. IEEE.

\end{thebibliography}
       \bibliographystyle{apalike}
}

\end{document}